%% file: main.tex
\documentclass[twoside]{article}

\usepackage[accepted]{aistats2021}

\usepackage{amsfonts, amsmath, amssymb, amsthm}
\usepackage{mathtools}
\usepackage{xspace}
\usepackage{nicefrac}
\usepackage{booktabs}
\usepackage{enumerate}
\usepackage{caption}
\usepackage{subcaption}
\usepackage{url}
\usepackage{wrapfig}
\usepackage{thmtools, thm-restate}

\usepackage[colorlinks=true,citecolor=blue]{hyperref} 

\usepackage[round]{natbib}

\usepackage{algorithm}
\usepackage[noend]{algorithmic}

\newcommand{\R}{\mathbb{R}}

\newcommand{\mD}{\mathcal{D}}
\newcommand{\mF}{\mathcal{F}}
\newcommand{\mI}{\mathcal{I}}
\newcommand{\mP}{\mathcal{P}}
\newcommand{\mQ}{\mathcal{Q}}
\newcommand{\mZ}{\mathcal{Z}}

\newcommand{\ServerUpdate}{\textsc{ServerUpdate}\xspace}
\newcommand{\ServerOpt}{\textsc{ServerOpt}\xspace}
\newcommand{\ClientUpdate}{\textsc{ClientUpdate}\xspace}
\DeclareMathOperator*{\E}{\mathbb{E}}

\DeclareMathOperator*{\argmin}{argmin}
\DeclareMathOperator*{\cond}{cond}

\newcommand{\sgd}{\textsc{SGD}\xspace}

\newcommand{\fedsgd}{\textsc{FedSGD}\xspace}
\newcommand{\fedavg}{\textsc{FedAvg}\xspace}
\newcommand{\fedprox}{\textsc{FedProx}\xspace}
\newcommand{\fedavgm}{\textsc{FedAvgM}\xspace}

\newcommand{\fedadam}{\textsc{FedAdam}\xspace}

\newcommand{\maml}{\textsc{MAML}\xspace}
\newcommand{\reptile}{\textsc{Reptile}\xspace}
\newcommand{\fomaml}{\textsc{FOMAML}\xspace}
\newcommand{\lookahead}{\textsc{Lookahead}\xspace}
\newcommand{\mmbp}{\textsc{MetaMinibatchProx}\xspace}

\newcommand{\localupdate}{\textsc{LocalUpdate}\xspace}

\newcommand{\adam}{\textsc{Adam}\xspace}

\newcommand{\norm}[1]{\left\lVert#1\right\rVert}

\declaretheorem{theorem}

\declaretheorem{lemma}
\declaretheorem{corollary}
\newtheorem{assumption}{Assumption}
\newtheorem{definition}{Definition}

\DeclarePairedDelimiterX{\inp}[2]{\langle}{\rangle}{#1, #2}
\DeclarePairedDelimiterX{\abs}[1]{\lvert}{\rvert}{#1}
\DeclarePairedDelimiterX{\cbr}[1]{\{}{\}}{#1} 
\DeclarePairedDelimiterX{\rbr}[1]{(}{)}{#1} 
\DeclarePairedDelimiterX{\sbr}[1]{[}{]}{#1} 

\begin{document}

\include{body}
\include{appendix}

\end{document}

%% file: body.tex
\twocolumn[

\aistatstitle{Convergence and Accuracy Trade-Offs in Federated Learning and Meta-Learning}

\aistatsauthor{ Zachary Charles \And Jakub Kone\v{c}n\'{y} }

\aistatsaddress{ Google Research \And Google Research } ]

\begin{abstract}
We study a family of algorithms, which we refer to as \emph{local update methods}, generalizing many federated and meta-learning algorithms.
We prove that for quadratic models, local update methods are equivalent to first-order optimization on a surrogate loss we exactly characterize.
Moreover, fundamental algorithmic choices (such as learning rates) explicitly govern a trade-off between the condition number of the surrogate loss and its alignment with the true loss.
We derive novel convergence rates showcasing these trade-offs and highlight their importance in communication-limited settings.
Using these insights, we are able to compare local update methods based on their convergence/accuracy trade-off, not just their convergence to critical points of the empirical loss.
Our results shed new light on a broad range of phenomena, including the efficacy of server momentum in federated learning and the impact of proximal client updates.
\end{abstract}

\section{Introduction}\label{sec:intro}

Federated learning \citep{mcmahan17fedavg} is a distributed framework for learning models without directly sharing data. In this framework, clients perform \emph{local updates} (typically using first-order optimization) on their own data. In the popular \fedavg algorithm \citep{mcmahan17fedavg}, the client models are then averaged at a central server. Since the proposal of \fedavg, many new federated optimization algorithms have been developed~\citep{li2018federated, reddi2020adaptive, hsu2019measuring, xie2019local, basu2019qsparse, li2019fair, karimireddy2019scaffold}. These methods typically employ multiple local client epochs in order to improve communication-efficiency. We defer to \citet{kairouz2019advances} and \citet{li2019federated} for more detailed summaries of federated learning.

Local updates have also been used extensively in meta-learning. The celebrated \maml algorithm \citep{finn2017model} employs multiple local model updates on a set of tasks in order to learn a model that quickly adapt to new tasks. \maml has inspired a number of model-agnostic meta-learning methods that also employ first-order local updates~\citep{balcan2019provable, fallah2020convergence, nichol2018first, zhou2019efficient}. There are strong connections between federated learning and meta-learning, despite differences in practical concerns. Formal connections between the two were shown by \citet{khodak2019adaptive} and have since been explored in many other works~\citep{jiang2019improving, fallah2020personalized}.

We refer to methods that utilize multiple local updates across clients (or in the language of meta-learning, tasks) as \emph{local update methods} (see Section \ref{sec:local_update_methods} for a formal characterization). In practice, local update methods frequently outperform ``centralized'' methods such as \sgd~\citep{mcmahan17fedavg, finn2017model, hard2018federated, yang2018applied, hard2020training}. However, the empirical benefits of local update methods are not fully explained by existing theoretical analyses. For example, \citet{woodworth2020local} show that \fedavg often obtains convergence rates comparable to or worse than those of mini-batch \sgd.

We focus on two difficulties that arise when analyzing local update methods. First, analyses must account for \emph{client drift} \citep{karimireddy2019scaffold}. As clients perform local updates on heterogeneous datasets, their local models drift apart. This hinders convergence to globally optimal models, and makes theoretical analyses more challenging. Similar phenomena were examined by \citet{li2018federated, malinovsky2020local, pathak2020fedsplit} and \citet{fallah2020personalized}, who show that various local update methods do not converge to critical points of the empirical loss.

Second, local update methods are difficult to compare. Analyses of different methods may use different hyperparameters regimes, or make different assumptions. Even comparing seemingly similar methods can require significant theoretical insight~\citep{karimireddy2019scaffold, fallah2020personalized}. Moreover, comparisons can be made in fundamentally different ways. One may wish to maximize the final accuracy, or minimize the number of communication rounds needed to attain a given accuracy. Thus, it is not even clear \emph{how} local update methods should be compared.

\paragraph{Contributions} In this work, we invert the conventional narrative that issues such as client drift harm convergence. Instead, we view such phenomena as improving convergence, but to sub-optimal points. 

More generally, we show that local update methods face a fundamental trade-off between convergence and accuracy that is explicitly governed by algorithmic hyperparameters. Perceived failures of methods such as \fedavg actually correspond to operating points prioritizing convergence over accuracy. We use this trade-off to develop a novel framework for comparing local update methods. We compare methods based on their entire convergence-accuracy trade-off, not just their convergence to optimal points. In more detail:

\begin{enumerate}
    \item We show that for quadratic models, local update methods are equivalent to optimizing a single \emph{surrogate} loss function. The condition number of the surrogate is controlled by algorithmic choices. Popular local update methods, including \fedavg and \maml, reduce the surrogate's condition number, but increase the discrepancy between the empirical and surrogate losses. Our results also encompass \emph{proximal} local update methods~\citep{li2018federated, zhou2019efficient}.
    \item We derive novel convergence rates that showcasing this trade-off between convergence and accuracy. Our bounds demonstrate the benefit of local update methods over methods such as mini-batch \sgd in communication-limited settings.
    \item We use this theory to develop a framework for comparing local update methods through a novel \emph{Pareto frontier}, which compares convergence-accuracy trade-offs of classes of algorithms. We use this to derive novel comparisons of many popular local update methods.
    \item We use this technique to shed light on a broad range of phenomena, including the benefit of server momentum, the effect of proximal local updates, and differences between the dynamics of \fedavg and \maml.
    \item While our theoretical results are restricted to quadratic models, we show that such convergence-accuracy trade-offs occur empirically in non-convex settings. We also validate our theoretical observations regarding server momentum and proximal updates on a non-convex task.
\end{enumerate}

We view our work as a step towards holistic understandings of local update methods. Using the aforementioned Pareto frontiers, we highlight a number of new phenomena and open problems. One particularly intriguing observation is that the convergence-accuracy trade-off for \fedavg with heavy-ball server momentum appears to be completely symmetric. For more details, see Section \ref{sec:compare}. Our proof techniques may be of independent interest. We derive a novel analog of the Bhatia-Davis inequality \citep{bhatia2000better} for mean absolute deviations, and use this to understand the accuracy of local update methods.



\paragraph{Notation}

We let $\norm{\cdot}$ denote the $\ell_2$ norm for vectors and the spectral norm for matrices.
For a symmetric positive semi-definite matrix $A$, we let $A^{1/2}$ denote its matrix square root. We let $\preceq$ be the \emph{Loewner order} on positive semi-definite matrices. For a real symmetric matrix $A$, we let $\lambda_{\max}(A)$, $\lambda_{\min}(A)$ denote its largest and smallest eigenvalues, and let $\cond(A)$ denote their ratio. In a slight abuse of notation, if $f$ is a $L$-smooth, $\mu$-strongly convex function, we say $\cond(f) \leq L/\mu$.

\paragraph{Accuracy and Meta-Learning}

We study the accuracy of local update methods on the training population. However, meta-learning algorithms are designed to learn a model that adapts well to new tasks; The empirical loss is not necessarily indicative of the ``post-adaptation'' accuracy of such methods~\citep{finn2017model}. Despite this our focus still yields novel insights into qualitative differences between the training dynamics of federated learning and meta-learning methods. Perhaps surprisingly, we show that in certain hyperparameter regions, these methods exhibit identical trade-offs between convergence and pre-adaptation accuracy (see Figures \ref{fig:compare_fedavg_maml} and \ref{fig:compare_simulated_fedavg_maml}). While we believe our results can be adapted to post-adaptation accuracy via techniques developed by \citet{fallah2020convergence}, we leave the analysis to future work.

\section{Problem Setup}\label{sec:problem_setup}

Let $\mI$ denote some collection of clients, and let $\mathcal{P}$ be a distribution over $\mI$. For each $i \in \mI$, there is an associated distribution $\mD_i$ over the space $\mZ$ of examples. For any $z \in \mZ$, we assume there is symmetric matrix $B_z \in \R^{d \times d}$ and vector $c_z \in \R^d$ such that the loss of a model $x \in \R^d$ at $z$ is given by
\begin{equation}
\label{eq:quadratic_loss}
    f(x; z) := \tfrac{1}{2}\|B_z^{1/2}(x-c_z)\|^2.
\end{equation}
For $i \in \mI$, we define the client loss function $f_i$ and the overall loss function $f$ as follows:
\begin{equation}\label{eq:objective}
f_i(x) := \E_{z \sim \mD_i} [f(x ; z)],~~~f(x) := \E_{i \sim \mP}[f_i(x)].
\end{equation}
The joint distribution $(\mI, \mZ)$ defines a distribution over $\mZ$, recovering standard risk minimization, as well as distributed risk minimization in which $\mP$ and all $\mD_i$ are uniform over finite sets. For $i \in \mI$, define:
\begin{equation}\label{eq:A_i_c_i}
A_i := \E_{z \sim \mD_i}[B_z],~~~c_i := A_i^{-1}\E_{z \sim \mD_i}[B_zc_z].
\end{equation}
We assume these expectations exist and are finite. One can show that up to some additive constant,
\[
f_i(x) = \tfrac{1}{2}\|A_i^{1/2}(x-c_i)\|^2.
\]
We make the following assumptions throughout.

\begin{assumption}\label{assm1}
There are $\mu, L > 0$ such that for all $i$, $\mu I \preceq A_i \preceq L I$.
\end{assumption}

\begin{assumption}\label{assm2}There is some $C > 0$ such that for all $i$, $\|c_i\| \leq C$.
\end{assumption}

Assumption \ref{assm1} bounds the Lipschitz and strong convexity parameters of the $f_i$, and holds if the matrices $B_z$ satisfy bounded eigenvalue conditions. Assumption \ref{assm2} states the $c_i$ are bounded. Intuitively, local update methods provide larger benefit for smaller values of $C$, as the clients progress towards similar optima. While our analysis can be directly generalized to the case where the $c_i$ are contained in a ball of radius $C$ about $p \in \R^d$, we assume $p = 0$ for simplicity. Assumptions \ref{assm1} and \ref{assm2} can be relaxed to only hold in expectation, though this complicates the analysis.

\subsection{Local Update Methods}\label{sec:local_update_methods}

\begin{table*}
    \caption{Special cases of \localupdate when \ServerOpt is gradient descent.}
    \label{table:local_update_special_cases}
    \begin{center}
    \begin{small}
    \begin{tabular}{lll}    
        \toprule
        Algorithm & $\Theta$ & Conditions \\
        \midrule
        Mini-batch \sgd & $\Theta_1$ & $|\mP| = 1, \alpha = 0, \gamma = 0$\\
        \lookahead~\citep{zhang2019lookahead} & $\Theta_{1:K}$ & $|\mP| = 1, \alpha = 0, \gamma > 0$ \\
        \fedsgd~\citep{mcmahan17fedavg} & $\Theta_{1:K}$ & $\alpha = 0, \gamma = 0$\\
        \fedavg~\citep{mcmahan17fedavg}, \reptile~\citep{nichol2018first} & $\Theta_{1:K}$ & $\alpha = 0, \gamma > 0$\\
        \fedprox~\citep{li2018federated}, \mmbp~\citep{zhou2019efficient} & $\Theta_{1:K}$ & $\alpha > 0, \gamma > 0$\\
        \fomaml~\citep{finn2017model} & $\Theta_{K}$ & $\alpha = 0, \gamma > 0$\\
        \maml~\citep{finn2017model} & $\Theta_{2K+1}$ & $\alpha = 0$, quadratics (Theorem \ref{thm:maml})\\
        \bottomrule
    \end{tabular}
    \end{small}
    \end{center}
\end{table*}

We consider a class of algorithms we refer to as \emph{local update} methods. In these methods, at each round $t$ the server samples a set $I_t$ of $M$ clients (in the language of meta-learning, tasks) from $\mP$, and broadcasts its model $x_t$ to all clients in $I_t$. Each client $i \in I_t$ optimizes its loss function $f_i$ (starting at $x_t$) by applying $K$ iterations of mini-batch \sgd with batch size $B$ and client learning rate $\gamma$. As proposed by \citet{li2018federated} and \citet{zhou2019efficient}, clients also add $\ell_2$ regularization with parameter $\alpha \geq 0$ towards the broadcast model $x_t$.

The client sends a linear combination of the gradients it computes to the server. The coefficients of the linear combination are given by $\Theta = (\theta_1, \theta_2, \dots, \theta_k, \dots)$ for $\theta_i \in \R_{\geq 0}$, where $\Theta$ has finite and non-zero support. For such $\Theta$, we define
\begin{equation}\label{eq:size_and_weight}
K(\Theta) := \max\{k~|~\theta_k > 0\},~~w(\Theta) = \sum_{k=1}^{K(\Theta)} \theta_k.
\end{equation}

After receiving all client updates, the server treats their average $q_t$ as an estimate of the gradient of the loss function $f$, and applies $q_t$ to a first-order optimization algorithm \ServerOpt. For example, the server could perform a gradient descent step using the ``pseudo-gradient'' $q_t$. We refer to this process (parameterized by $\alpha, \gamma, \Theta$ and \ServerOpt) as \localupdate and give pseudo-code in Algorithms \ref{alg:outerloop} and \ref{alg:innerloop}.

\newcommand{\T}{\rule{0pt}{2.2ex}}
\newcommand{\SUB}[1]{\ENSURE \hspace{-0.15in} \textbf{#1}}
\newcommand{\algfont}[1]{\texttt{#1}}
\renewcommand{\algorithmicensure}{}

\begin{algorithm}
\caption{\localupdate: \ServerUpdate}
\label{alg:outerloop}
\begin{algorithmic}
\SUB{$\ServerUpdate(x, {\normalfont \ServerOpt}, \alpha, \gamma, \Theta)$:}
\STATE $x_0 = x$
\FOR{each round $t = 0, 1, \dots, T-1$}
    \STATE sample a set $I_t$ of size $M$ from $\mP$
    \FOR{each client $i \in I_t$ \textbf{in parallel}}
        \STATE $q_t^i = \ClientUpdate(i, x_t, \alpha, \gamma, \Theta)$
    \ENDFOR
    \STATE $q_t = (\nicefrac{1}{M})\sum_{i \in I_t} q_{t}^i$
    \STATE $x_{t+1} = \ServerOpt(x_t, q_t)$
\ENDFOR
\STATE return $x_{T+1}$
\end{algorithmic}
\end{algorithm}

\begin{algorithm}
\caption{\localupdate: \ClientUpdate}
\label{alg:innerloop}
\begin{algorithmic}
\renewcommand{\arraystretch}{1.6}
\SUB{$\ClientUpdate(i, x, \alpha, \gamma, \Theta)$:}
\STATE $x_1 = x$
\FOR{$k = 1, 2, \dots, K(\Theta)$}
    \STATE sample a set $S_k$ of size $B$ from $\mD_i$
    \STATE $g_k = (\nicefrac{1}{B}) \sum_{z \in S_k} \nabla_{x_k}\left( f_i(x_k ; z) +\frac{\alpha}{2}\norm{x_k-x}^2\right)$
    \STATE $x_{k+1} = x_k - \gamma g_k$ 
\ENDFOR
\STATE return $\sum_{k = 1}^{K(\Theta)} \theta_k g_k$
\end{algorithmic}
\end{algorithm}

\localupdate recovers many well-known algorithms for various choices $\Theta$. For convenience, define
\begin{equation}\label{eq:theta}
    \Theta_K := (\underbrace{0, \dots, 0}_{\text{K-1 times}}, 1),~~~\Theta_{1:K} := (\underbrace{1, \dots, 1}_{\text{K times}}).
\end{equation}
Special cases of \localupdate when \ServerOpt is gradient descent are given in Table \ref{table:local_update_special_cases}. For details on the relation between \fedavg and \localupdate, see Appendix \ref{appendix:special_cases}. By changing \ServerOpt, we can recover methods such as \fedavgm~\citep{hsu2019measuring} (server gradient descent with momentum), and \fedadam~\citep{reddi2020adaptive} (server \adam~\citep{kingma2014adam}).

\section{Local Update Methods as First-Order Methods}\label{sec:surrogate_loss}

\localupdate can vary drastically from first-order optimization methods on the empirical loss. Despite this, we will show that Algorithm \ref{alg:innerloop} is equivalent in expectation to \ServerOpt applied to a single \textit{surrogate loss}. This surrogate loss is determined by the inputs $\alpha, \gamma$ and $\Theta$ to Algorithm \ref{alg:innerloop}. For each client $i \in \mI$, we define its \textit{distortion matrix} $Q_i(\alpha, \gamma, \Theta)$ as
\begin{equation}\label{eq:Q_matrix}
Q_i(\alpha, \gamma, \Theta) := \sum_{k=1}^{K(\Theta)} \theta_k (I-\gamma (A_i+\alpha I))^{k-1}.
\end{equation}
We define the surrogate loss function of client $i$ as
\begin{equation}\label{eq:surrogate_i}
\tilde{f}_i(x, \alpha, \gamma, \Theta) := \frac{1}{2}\|(Q_i(\alpha, \gamma, \Theta)A_i)^{1/2}(x-c_i)\|^2 
\end{equation}
and the overall surrogate loss function as
\begin{equation}\label{eq:surrogate_loss}
    \tilde{f}(x, \alpha, \gamma, \Theta) := \E_{i \sim \mP} [\tilde{f}_i(x, \alpha, \gamma, \Theta)].
\end{equation}
When $\Theta = \Theta_1$, $Q_i(\alpha, \gamma, \Theta) = I$, in which case there is no distortion. In general, $Q_i(\alpha, \gamma, \Theta)$ can amplify the heterogeneity of the $A_i$. We derive the following theorem linking the surrogate losses to Algorithm~\ref{alg:innerloop}.
\begin{restatable}{theorem}{thmobjective}\label{thm:sgd_objective}For all $i \in \mP$,
\begin{equation*}
\E[\ClientUpdate(i, x, \alpha, \gamma, \Theta)] = \nabla \tilde{f}_i(x, \alpha, \gamma, \Theta).
\end{equation*}
\end{restatable}

For $q_t$ as in Algorithm \ref{alg:outerloop}, Theorem \ref{thm:sgd_objective} implies $\E[q_t] = \nabla \tilde{f}(x, \alpha, \gamma, \Theta)$.
Thus, one round of \localupdate is equivalent in expectation to one step of \ServerOpt on the surrogate loss $\tilde{f}(x, \alpha, \gamma, \Theta)$. As $\gamma \to 0$ or $K \to 1$, $\tilde{f}(x, \alpha, \gamma, \Theta) \to w(\Theta) f(x)$, so as $\gamma$ gets smaller, the ``pseudo-gradients'' $q_t$ more closely resemble stochastic gradients of the empirical loss function.

A version of Theorem \ref{thm:sgd_objective} was shown for $\alpha = 0, \Theta = \Theta_{2}$ by \citet{fallah2020personalized}. We take this a step further and show that in certain settings, \maml is equivalent in expectation to \ServerOpt on a surrogate loss.

\paragraph{\maml} \maml with $K$ local steps can be viewed as a modification of \localupdate. Algorithm \ref{alg:outerloop} remains the same, and in Algorithm \ref{alg:innerloop}, each client executes $K$ mini-batch \sgd steps. However, the client's message to the server is different. Let $X_{K}^i(x)$ be the function that runs $K$ steps of mini-batch \sgd, starting from $x$, for fixed mini-batches $S_1, \dots, S_K$ of size $B$ drawn independently from $\mathcal{D}_i$. Define
\[
m^i_K(x ; z) = f(X_{K}^i(x); z),~m_{K}^i(x) = \E_{z \sim \mD_i} [m^i_K(x ; z)].
\]
Each client $i$ sends a stochastic estimate of $\nabla m_K^i(x)$ to the server. 
The rest is identical to \localupdate; The server averages the client outputs and uses this as a gradient estimate for \ServerOpt.
While \maml is not a special case of \localupdate, we show that if the clients use gradient descent, \maml is equivalent in expectation to \localupdate with $\Theta = \Theta_{2K+1}$.
\begin{restatable}{theorem}{thmmaml}\label{thm:maml}If $X_K^i(x)$ is the function that runs $K$ steps of gradient descent on $\mD_i$ with learning rate $\gamma$ starting at $x$, then
\[
\nabla m^i_K(x) = \nabla_{x} \tilde{f}_i(x, 0, \gamma, \Theta_{2K+1}).
\]
\end{restatable}

An analogous result holds if the clients perform proximal updates ($\alpha > 0$ in \ClientUpdate).
Thus, to understand \localupdate and \maml on quadratic models, it suffices to analyze the optimization dynamics of $\tilde{f}(x, \alpha, \gamma, \Theta)$. We use this viewpoint to study the convergence and accuracy of these methods.

\section{Convergence and Accuracy of Local Update Methods}\label{sec:convergence_and_accuracy}

Comparing \eqref{eq:objective} and \eqref{eq:surrogate_loss}, we see that $\tilde{f}(x, \alpha, \gamma, \Theta)$ and $f(x)$ need not share critical points. Special cases of this fact were noted by \citet{malinovsky2020local}, \citet{fallah2020personalized}, and \citet{pathak2020fedsplit}. We will show that this is not a failure of local update methods. Rather, by altering the loss function being optimized, \localupdate can greatly improve convergence, but to a less accurate point. More generally, the choice of $\alpha, \gamma$ and $\Theta$ dictates a trade-off between convergence and accuracy. Intuitively, the larger $\gamma$ and $K(\Theta)$ are, the faster \localupdate will converge, and the less accurate the resulting model may be.

To show this formally, we restrict to $Q_i(\alpha, \gamma, \Theta) \succ 0$, as then $\tilde{f}(x, \gamma, \Theta)$ is strongly convex with a unique minimizer. This is ensured by the following.
\begin{restatable}{lemma}{lemstrcvx}\label{lem:strongly_convex}
Suppose that $\gamma < (L+\alpha)^{-1}$. Then for all $i$, $Q_i(\alpha, \gamma, \Theta)$ is positive definite and $\tilde{f}(x, \alpha, \gamma, \Theta)$ is strongly convex.
\end{restatable}

\subsection{Condition Numbers}\label{sec:condition_numbers}

Under the conditions of Lemma \ref{lem:strongly_convex}, $\tilde{f}_i(x, \alpha, \gamma, \Theta)$ has a well-defined condition number which we bound.
\begin{restatable}{lemma}{condgeneral}\label{lem:cond_general}
    Suppose $\gamma < (L+\alpha)^{-1}$. Define
    \begin{equation}\label{eq:kappa_QA}
    \kappa(\alpha, \gamma, \Theta) := \dfrac{ \E_i[\lambda_{\max}(Q_i(\alpha, \gamma, \Theta)A_i]}{\E_i[\lambda_{\min}(Q_i(\alpha, \gamma, \Theta)A_i]}.
    \end{equation}
    Then  $\cond(\tilde{f}) \leq \kappa(\alpha, \gamma, \Theta)$.
\end{restatable}

We wish to better understand \eqref{eq:kappa_QA} in cases of interest. We first consider $\Theta = \Theta_{1:K}$, as in \fedavg. Define
\begin{equation}\label{eq:phi}
\phi(\lambda, \alpha, \gamma, K) := \sum_{k=1}^K (1-\gamma(\lambda+\alpha))^{k-1}\lambda.
\end{equation}

We now derive a bound on $\cond(\tilde{f})$ for $\Theta = \Theta_{1:K}$.
\begin{restatable}{lemma}{condfedavg}\label{lem:cond_fedavg}
If $\gamma < (L+\alpha)^{-1}$, $\tilde{f}(x, \alpha, \gamma, \Theta_{1:K})$ is $\phi(L, \alpha, \gamma, K)$-smooth, $\phi(\mu, \alpha, \gamma, K)$-strongly convex, and $\cond(\tilde{f}) \leq \kappa(\alpha, \gamma, \Theta_K)$ where
    \begin{equation}\label{eq:cond_fedavg}
    \kappa(\alpha, \gamma, \Theta_{1:K}) \leq \dfrac{\phi(L, \alpha, \gamma, K)}{\phi(\mu, \alpha, \gamma, K)}.
    \end{equation}
\end{restatable}

When $\gamma = 0, \alpha = 0$, we recover the condition number $L/\mu$ of the empirical loss $f$.
We next consider $\Theta = \Theta_K$, as in \maml-style algorithms.
Define
\begin{equation}\label{eq:psi}
\psi(\lambda, \alpha, \gamma, K) := (1-\gamma(\lambda+\alpha))^{K-1}\lambda.
\end{equation}
We now derive a bound on $\cond(\tilde{f})$ for $\Theta = \Theta_{K}$.
\begin{restatable}{lemma}{condmaml}\label{lem:cond_maml}
If $\gamma < (KL+\alpha)^{-1}$, $\tilde{f}(x, \alpha, \gamma, \Theta_{K})$ is $\psi(L, \alpha, \gamma, K)$-smooth, $\psi(\mu, \alpha, \gamma, K)$-strongly convex, and $\cond(\tilde{f}) \leq \kappa(\alpha, \gamma, \Theta_K)$ where
    \begin{equation}\label{eq:cond_maml}
    \kappa(\alpha, \gamma, \Theta_{K}) \leq \left(\dfrac{1-\gamma (L+\alpha)}{1-\gamma(\mu+\alpha)}\right)^{K-1}\dfrac{L}{\mu}.
    \end{equation}
\end{restatable}

We show in Appendix \ref{appendix:tight_lemmas} that Lemmas \ref{lem:cond_fedavg} and \ref{lem:cond_maml} are tight. The extra condition that $\gamma < (KL+\alpha)^{-1}$ for $\Theta = \Theta_K$ is due to the fact that when $\gamma \geq (KL+\alpha)^{-1}$, $\cond(\tilde{f}(x, \alpha, \gamma, \Theta_K))$ depends on intermediate eigenvalues of the $A_i$, and exhibits more nuanced behavior. We explore this further in Section \ref{sec:compare}.

As $K \to 1$ or $\gamma \to 0$, $\kappa(\alpha, \gamma, \Theta_K) \to L/\mu$, which bounds the condition number of the empirical loss $f$. If $\gamma$ is not close to 0, we get an exponential reduction (in terms of $K$) of the condition number. While the analysis is not as clear for $\Theta_{1:K}$, one can show that $\kappa(\alpha, \gamma, \Theta_{1:K}) \leq L/\mu$, with equality if and only if $\alpha = 0$, and either $\gamma = 0$ or $K = 1$. Moreover, $\kappa(\alpha, \gamma, \Theta_{1:K})$ decreases as $K \to \infty$ or $\gamma \to (L+\alpha)^{-1}$. For both $\Theta_{1:K}$ and $\Theta_K$, increasing $\alpha$ decreases $\kappa$.

Here we see the impact of local update methods on convergence: Popular methods such as \fedavg, \fedprox, \maml, and \reptile reduce the condition number of the surrogate loss function they are actually optimizing. In the next section, we translate this into concrete convergence rates for \localupdate.

\subsection{Convergence Rates}\label{sec:convergence_rates}

We now focus on a \emph{deterministic} version of \localupdate in which all clients participate at each round and perform $K$ steps of gradient descent. The server updates its model using $q_t = \E_{i \sim\mP}[q_t^i]$, where $q_t^i$ is the output of \ClientUpdate for client $i$. In particular, $\mP$ must be known to the server. In this case, Theorem \ref{thm:sgd_objective} implies that $q_t = \nabla \tilde{f}(x_t, \alpha, \gamma, \Theta)$, so \localupdate is equivalent to applying \ServerOpt to the true gradients of $\tilde{f}(x, \alpha, \gamma, \Theta)$.

We specialize to the setting where \ServerOpt is gradient descent, with or without momentum (though our analysis can be directly extended to other optimizers). Thus, \localupdate is equivalent to gradient descent on $\tilde{f}(x, \alpha, \gamma, \Theta)$. Using the bound on $\cond(\tilde{f})$ in Lemma \ref{lem:cond_general}, we can directly apply classical convergence theory gradient descent (\citet[Proposition 1]{lessard2016analysis} give a useful summary) to derive convergence rates for \localupdate. Similar analyses can be done in the stochastic setting.

\begin{theorem}\label{thm:conv_rates}
    Suppose $\gamma < (L+\alpha)^{-1}$ and \ServerOpt is gradient descent with Nesterov, heavy-ball, or no momentum. Then for some hyperparameter setting of \ServerOpt, and $\rho$ as in Table \ref{table:conv_local_update}, the iterates $\{x_t\}_{t \geq 1}$ of \localupdate satisfy
    \begin{equation}\label{eq:conv_rate}
    \norm{x_T - x^*(\alpha, \gamma, \Theta)} \leq \rho^T\norm{x_0-x^*(\alpha, \gamma, \Theta)}.
    \end{equation}
    
\begin{table}[ht]
    \caption{Convergence rates of \localupdate when \ServerOpt is gradient descent (with or without momentum), and $\kappa = \kappa(\alpha, \gamma, \Theta)$ is as in \eqref{eq:kappa_QA}.}
    \label{table:conv_local_update}
    \begin{center}
    \begin{tabular}[t]{ll}    
        \toprule
        Momentum & Rate \\
        \midrule
        None & $\rho = \frac{\kappa -1}{\kappa+1}$ \\[1.0ex]
        Nesterov & $\rho = 1 - \frac{2}{\sqrt{3\kappa + 1}}$ \\[1.0ex]
        Heavy-ball & $\rho = \frac{\sqrt{\kappa}-1}{\sqrt{\kappa}+1}$\\[1.0ex]
        \bottomrule
    \end{tabular}
    \end{center}
\end{table}
\end{theorem}

Thus, (properly tuned) server momentum improves the convergence of \localupdate, giving theoretical grounding\footnote{\citet{yuan2020federated} first showed that momentum can accelerate \fedavg, though they use a different momentum scheme with extra per-round communication.} to the improved convergence of \fedavgm shown by \citet{hsu2019measuring} and \citet{reddi2020adaptive}. Since \ServerOpt does not change the surrogate loss, this improvement in convergence does not degrade the accuracy of the learned model.

Given the bounds on $\cond(\tilde{f})$ in \eqref{eq:cond_fedavg} and \eqref{eq:cond_maml}, we obtain explicit convergence rates for \fedavg- and \maml-style algorithms as well. In particular, one can show that increasing $\gamma$ or $K$ decreases $\kappa$ (and therefore $\rho$). We show in the next section that this comes at the expense of increasing the empirical loss.

\subsection{Distance Between Global Minimizers}\label{sec:distances}

We now turn our attention towards the discrepancy between the surrogate loss $\tilde{f}$ and the empirical loss $f$. We assume $\gamma < (L+\alpha)^{-1}$. By Lemma \ref{lem:strongly_convex}, $\tilde{f}$ and $f$ are strongly convex with global minimizers we denote by
\[
x^*(\alpha, \gamma, \Theta) := \argmin_x \tilde{f}(x, \alpha, \gamma, \Theta),
\]
\[
x^* := \argmin_x f(x).
\]
We are interested in $\|x^*(\alpha, \gamma, \Theta)-x^*\|$. While we focus on the setting where $\mP$ is a discrete distribution over some finite $\mI$, our analysis can be generalized to arbitrary probability spaces $(\mI, \mF, \mP)$. We derive the following bound.

\begin{lemma}\label{lem:opt_dist}
    Let $b = \max_{i \in \mI}\lambda_{\max}(Q_i(\alpha, \gamma, \Theta))$ and $a = \min_{i \in \mI}\lambda_{\min}(Q_i(\alpha, \gamma, \Theta))$. Then
    \begin{equation}
        \|x^*(\alpha, \gamma, \Theta)-x^*\| \leq 8C\dfrac{\sqrt{b} - \sqrt{a}}{\sqrt{b} + \sqrt{a}}.
    \end{equation}
\end{lemma}

When $d = 1$, we can reduce the constant factor to $2C$, which we show is tight (see Appendix \ref{appendix:tight_distance}). While we conjecture that this bound holds with a constant of $2C$ for all $d$, we leave this to future work.

Our proof technique for Lemma \ref{lem:opt_dist} may be of independent interest. We derive this result by first proving an analog of the Bhatia-Davis inequality~\citep{bhatia2000better} for mean absolute deviations of bounded random variables (Theorem \ref{thm:better_bound_mad} in Appendix \ref{appendix:distance_proof}).

Let $\kappa_0 := L/\mu$. Specializing to $\Theta = \Theta_{1:K}$ or $\Theta_K$, we derive a link between $\kappa(\alpha, \gamma, \Theta)$ in \eqref{eq:cond_fedavg} and \eqref{eq:cond_maml} and the distance between optimizers.

\begin{restatable}{lemma}{qcond}\label{lem:dist_fedavg_maml}
    Suppose that either (I) $\gamma < (L+\alpha)^{-1}$ and $\Theta = \Theta_{1:K}$ or (II) $\gamma < (KL+\alpha)^{-1}$ and $\Theta = \Theta_{K}$. Then for all $i \in \mP$, $\cond(Q_i(\alpha, \gamma, \Theta)) \leq \kappa_0\kappa(\alpha, \gamma, \Theta)^{-1}$.
\end{restatable}

Combining this with Lemma \ref{lem:opt_dist}, we get:
\begin{theorem}\label{thm:dist_fedavg_maml}
Under the same settings as Lemma \ref{lem:dist_fedavg_maml},
\begin{equation}\label{eq:dist_fedavg_maml}
    \|x^*(\alpha, \gamma, \Theta)-x^*\| \leq 8C\frac{\sqrt{\kappa_0} - \sqrt{\kappa(\alpha, \gamma, \Theta)}}{\sqrt{\kappa_0} + \sqrt{\kappa(\alpha, \gamma, \Theta)}}.
\end{equation}
\end{theorem}

Applying Theorem \ref{thm:conv_rates}, we bound the convergence of \localupdate to the empirical minimizer $x^*$.

\begin{corollary}\label{cor:conv_rates}
    Under the same settings as Theorem \ref{thm:conv_rates}, for some hyperparameter setting of \ServerOpt, the iterates $\{x_t\}_{t \geq 1}$ of \localupdate satisfy
\begin{align*}
    & \norm{x_T - x^*} \leq \rho^T\norm{x_0-x^*(\alpha, \gamma, \Theta)} + 8C\frac{\sqrt{\kappa_0} - \sqrt{\kappa}}{\sqrt{\kappa_0} + \sqrt{\kappa}}
\end{align*}
    where $\kappa = \kappa(\alpha, \gamma, \Theta)$ is given in \eqref{eq:cond_fedavg} and \eqref{eq:cond_maml}, and $\rho$ is given in Table \ref{table:conv_local_update}.
\end{corollary}

Here we see the benefit of local update methods in \emph{communication-limited} settings. When $T$ is small and $\norm{x_0-x^*}$ is large, we can achieve better convergence by decreasing $\rho$ and leaving the second term fixed. In such settings, \fedavg can arrive at a neighborhood of a critical point in fewer communication rounds than mini-batch \sgd, but may not ever actually reach the critical point. If $\norm{x_0-x^*}$ is small, we may be better served by using mini-batch \sgd instead.

\section{Comparing Local Update Methods}\label{sec:compare}

Comparing optimization algorithms is a fundamental theoretical effort. Many past works compare local update methods based on their convergence to critical points of the empirical loss. By Theorem \ref{thm:sgd_objective}, \localupdate is only guaranteed to converge to critical points of $f$ if $\gamma = 0$ or $K(\Theta) = 1$. Thus, existing analyses ignore many useful cases of \localupdate.

To remedy this, we compare local update algorithms on the basis of both convergence and accuracy. Instead of fixing $\gamma$ and $\Theta$, we analyze \localupdate as $\gamma$ and $K(\Theta)$ vary. To do so, we use our theory from Section \ref{sec:convergence_and_accuracy}. Given $\alpha, \gamma$ and $\Theta$, we define the \textbf{convergence rate} $\rho(\alpha, \gamma, \Theta)$ as the infimum over all $\rho$ such that for all $T \geq 1$, \eqref{eq:conv_rate} holds. Values of $\rho$ when \ServerOpt is gradient descent are given in Table \ref{table:conv_local_update}. For $\Theta = \Theta_{1:K}$ or $\Theta_K$, we define the \textbf{suboptimality} $\Delta(\alpha, \gamma, \Theta)$ by
\begin{equation}\label{eq:Delta}
\Delta(\alpha, \gamma, \Theta) := \dfrac{\sqrt{\kappa_0} - \sqrt{\kappa(\alpha, \gamma, \Theta)}}{\sqrt{\kappa_0} +\sqrt{\kappa(\alpha, \gamma, \Theta}}.
\end{equation}
By Theorem \ref{thm:dist_fedavg_maml}, this captures the asymptotic worst-case suboptimality of \localupdate.

Note that $\rho, \Delta \in [0, 1]$. Therefore, fixing $\mu, L$ and \ServerOpt, we obtain a \textbf{Pareto frontier} in $[0, 1]^2$ by plotting $(\rho, \Delta)$ for various $\gamma$ and $K(\Theta)$. This curve represents the worst-case convergence/accuracy trade-off of a class of local update methods. We generally want the curve to be as close to $(0, 0)$ as possible.

\begin{figure}
\centering
    \begin{subfigure}{.43\linewidth}
    \centering
    \includegraphics[width=\linewidth]{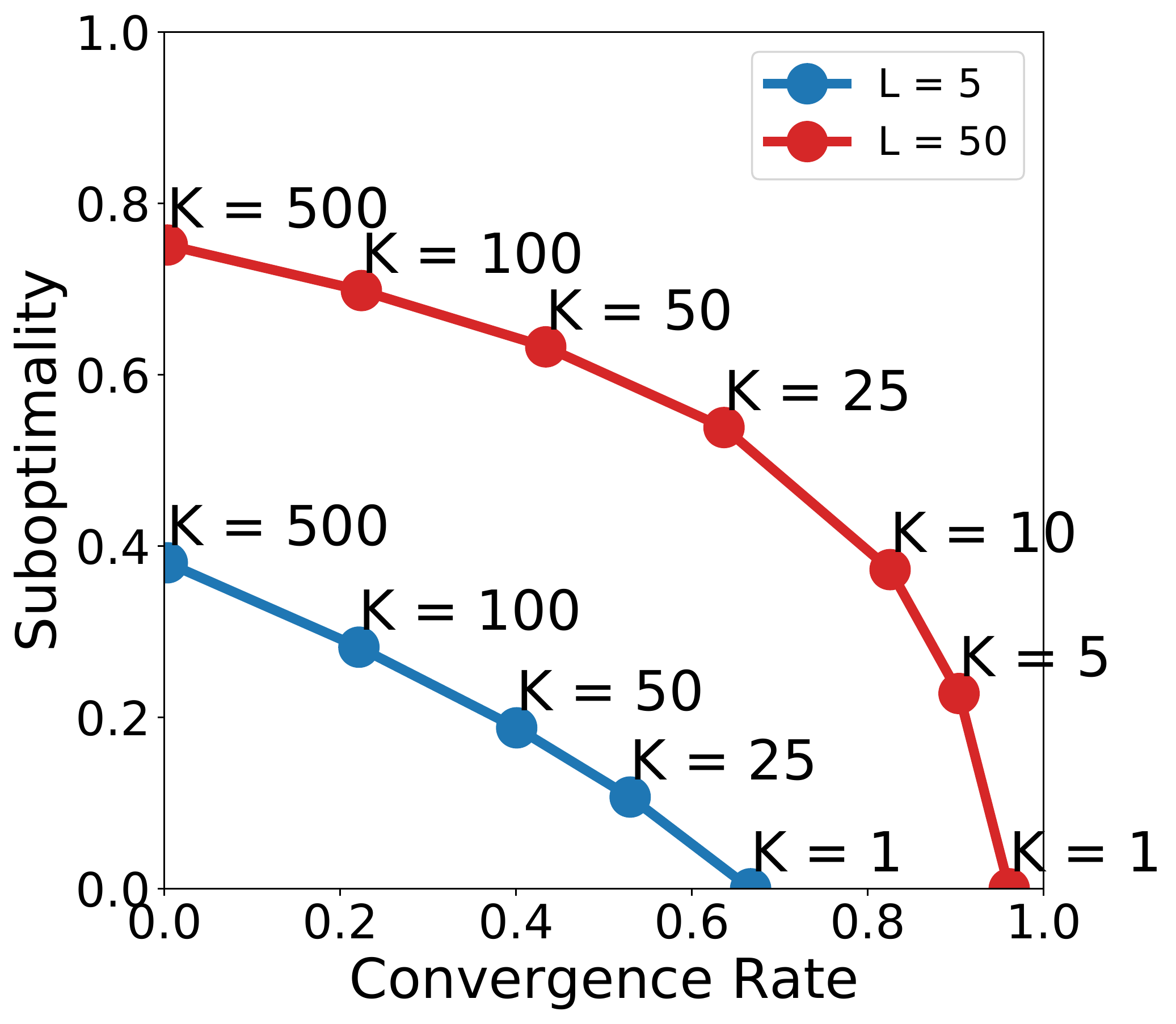}
    \end{subfigure}
    \begin{subfigure}{.47\linewidth}
    \centering
    \includegraphics[width=\linewidth]{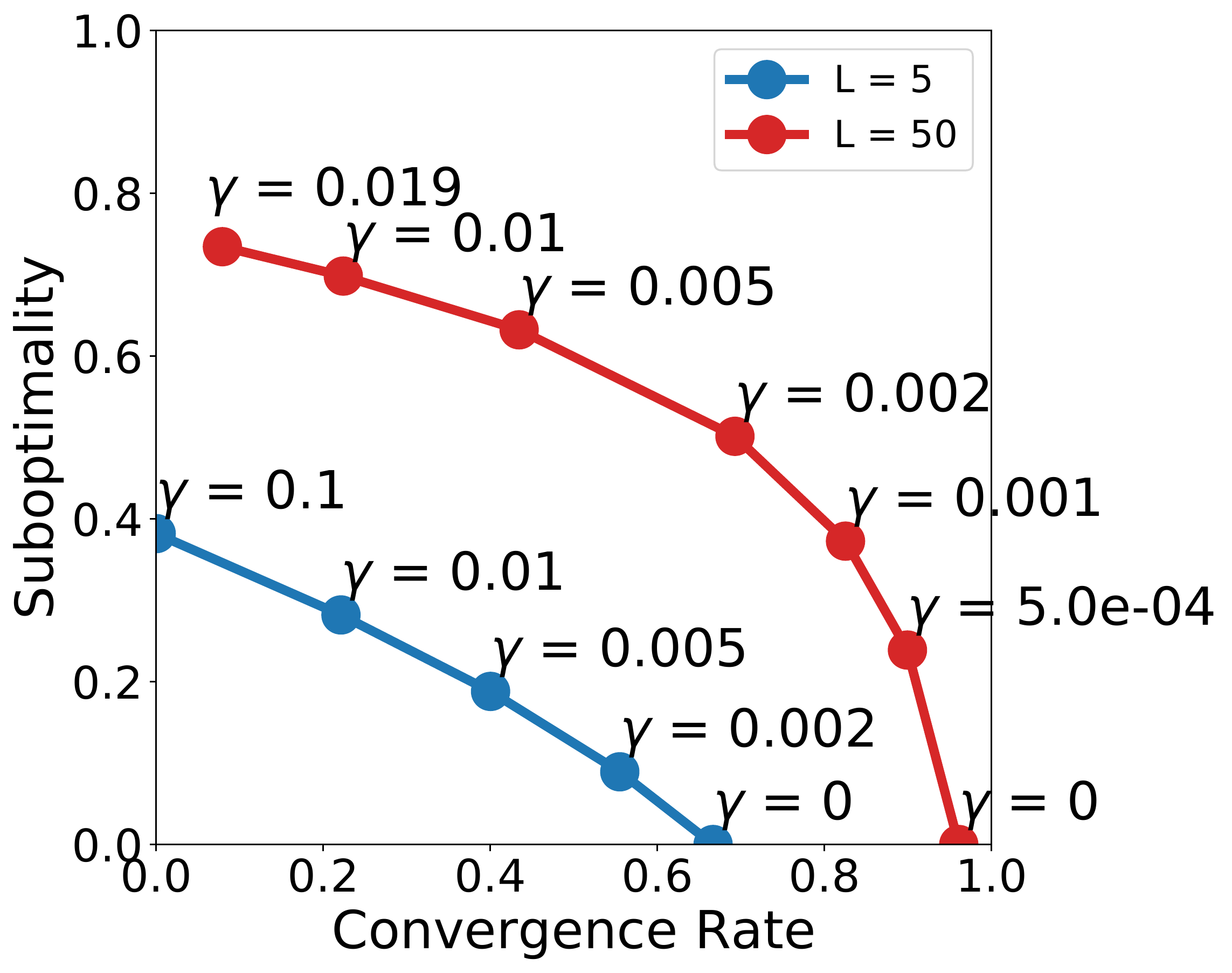}
    \end{subfigure}
\caption{Pareto frontiers for $\mu = 1, \alpha = 0, \Theta = \Theta_{1:K}$ and $L \in \{5, 50\}$. We fix $\gamma = 0.01$ and vary $K$ (left), and fix $K = 100$ and vary $\gamma$ (right).}
\label{fig:basic_fedavg_compare}
\end{figure}

For example, in Figure \ref{fig:basic_fedavg_compare} we let \ServerOpt be gradient descent and set $\alpha = 0, \Theta = \Theta_{1:K}$. We plot $(\rho, \Delta)$ as we vary $K$ and fix $\gamma$, and vice-versa. When $L = 5$, we obtain nearly identical curves. The curves for $L = 50$ are similar, except that when we fix $K$ and vary $\gamma$, we do not reach $\rho \approx 0$. While $\gamma$ and $K$ have similar impacts on convergence-accuracy trade-offs, varying $K$ leads a larger set of attainable $(\rho, \Delta)$. Formally, this is because in \eqref{eq:cond_fedavg}, $\lim_{\gamma \to L^{-1}}(\kappa) \neq 0$. Intuitively, $K \to \infty$ recovers one-shot averaging while $\gamma \to L^{-1}$ does not. Notably, the convergence-accuracy trade-off becomes closer to a linear trade-off as $L/\mu$ decreases.

The Pareto frontiers contain more information than just the convergence rate to a critical point (the curve's intersection with the $x$-axis). This information is useful in communication-limited regimes, where we wish to minimize the number of rounds needed to attain a given accuracy. The curves also help visualize various hyperparameter settings of an algorithms simultaneously. To illustrate this, we use the Pareto frontiers to derive novel findings regarding server momentum, proximal client updates, and qualitative differences between \fedavg and \maml. The results are all given below. For more results, see Appendix \ref{appendix:extra_comparisons}.

\begin{figure}
\centering
    \begin{subfigure}{.95\linewidth}
    \centering
    \includegraphics[width=\linewidth]{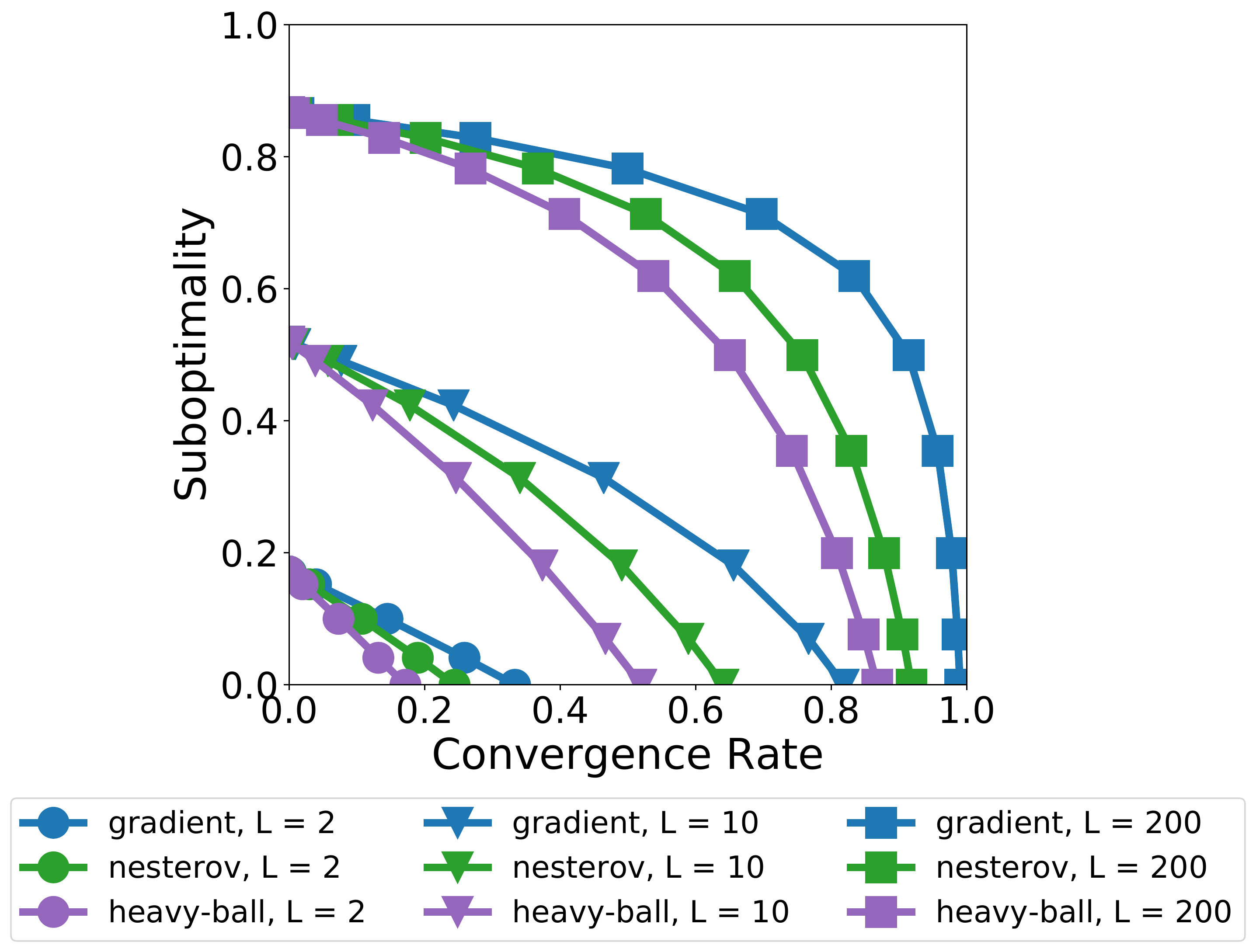}
    \end{subfigure}
\caption{Pareto frontiers for $\mu=1$, varying $L$, and where \ServerOpt is gradient descent with different types of momentum. We let $\alpha = 0, \gamma = (2L)^{-1}$ and $\Theta = \Theta_{1:K}$ for varying $K \in [1, 10^6]$.}
\label{fig:compare_server_opt}
\end{figure}

\paragraph{Impact of Server Momentum} As shown empirically by \citet{hsu2019measuring} and as reflected in Table \ref{table:conv_local_update}, server momentum can improve convergence. To understand this, in Figure \ref{fig:compare_server_opt} we compare Pareto frontiers where $\Theta = \Theta_{1:K}$ and \ServerOpt is gradient descent with various types of momentum (Nesterov, heavy-ball, or no momentum). We see a strict ordering of the server optimization methods. Heavy-ball momentum is better than Nesterov momentum, which is better than no momentum.

One important finding is that the benefit of momentum is more pronounced as $L/\mu$ increases. On the other hand, the benefit of server momentum diminishes for sufficiently large $K$: In Figure \ref{fig:compare_fedavg_maml}, the various types of momentum lead to similar suboptimality when the convergence rate is close to 0. Intuitively, as $K \to \infty$, we recover one-shot averaging, which converges in a single communication round with or without momentum.

Another intriguing observation: The Pareto frontiers for heavy-ball momentum appear to be symmetric about the line $\rho = \Delta$. We conjecture this is true for any $\mu, L$. While we believe that this may be provable by careful algebraic manipulation of our results above, ideally a proof would explain the root causes of this symmetry. Thus, we leave a proof to future work.

\begin{figure}[ht]
\centering
    \begin{subfigure}{.8\linewidth}
    \centering
    \includegraphics[width=\linewidth]{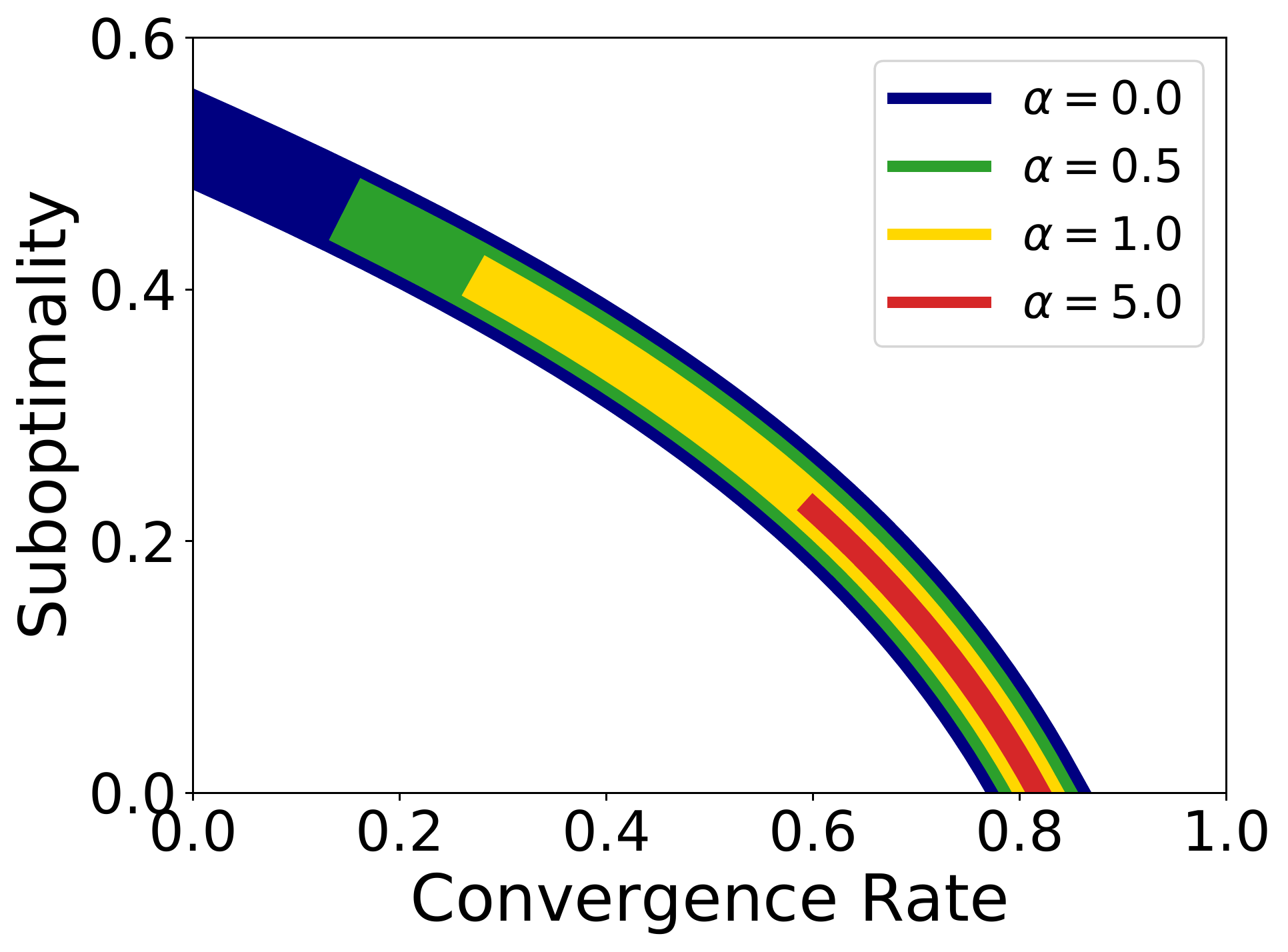}
    \end{subfigure}
\caption{Pareto frontiers for $\mu=1, L = 10, \Theta = \Theta_{1:K}$. We set $\gamma = \nicefrac{1}{2}(L+\alpha)^{-1}$, \ServerOpt as gradient descent, vary $\alpha$ and $K$ over $\{0, 0.5, 1.0, 5.0\}$ and $[1, 10^6]$.}
\label{fig:compare_fedprox}
\end{figure}

\paragraph{Proximal Client Updates}

So far we have only considered $\alpha = 0$. One might posit that as $\alpha$ varies, the Pareto frontier moves closer to the origin. This appears to not be the case. In all settings we examined, changing $\alpha$ did not bring the Pareto frontier closer to $0$. Instead, the frontier for $\alpha > 0$ was simply a subset of the frontier for $\alpha = 0$.

To illustrate this, we plot Pareto frontiers for varying $\alpha$ in Figure \ref{fig:compare_fedprox}. As $\alpha$ increases, the frontier becomes a smaller subset of the frontier for $\alpha = 0$. Thus, proximal client updates may not enable faster convergence. Rather, their benefit may be in guarding against setting $\gamma$ too small or $K$ too large. Figure \ref{fig:compare_fedprox} shows that \fedavg can always attain the same $(\rho, \Delta)$ as \fedprox, but it may require different hyperparameters. The reverse is not true, as \fedprox cannot recover one-shot averaging. Our findings are consistent with work by \citet{wang2020tackling}, who show that \fedprox can reduce the ``objective inconsistency'' of \fedavg, at the expense of increasing convergence time.

\begin{figure}
\centering
    \begin{subfigure}{.8\linewidth}
    \centering
    \includegraphics[width=\linewidth]{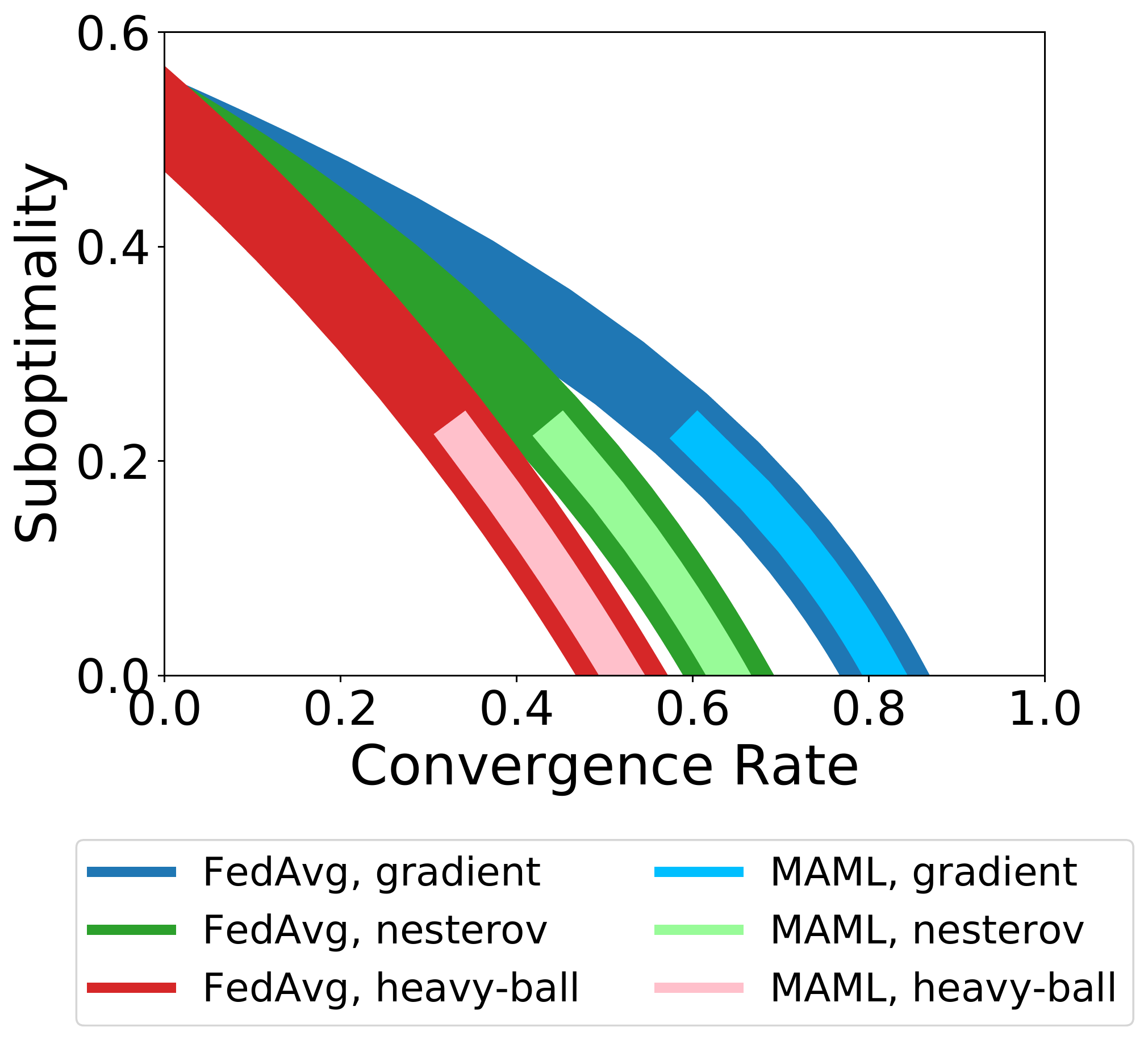}
    \end{subfigure}
\caption{Pareto frontiers for $\mu=1$, $L = 10$, $\alpha = 0$, $\Theta = \Theta_{1:K}$ (\fedavg) and $\Theta_{K}$ (\maml) for varying $K \in [0, 10^6]$. \ServerOpt is gradient descent, with Nesterov, heavy-ball, or no momentum. For $\Theta_{1:K}$, we use $\gamma = 0.001$, and for $\Theta_K$, we use $\gamma = (2LK)^{-1}$.}
\label{fig:compare_fedavg_maml}
\end{figure}

\paragraph{Comparing \maml to \fedavg} We now turn our attention to comparing \fedavg-style algorithms ($\Theta = \Theta_{1:K}$) to \maml-style algorithms ($\Theta = \Theta_K)$. We plot Pareto frontiers for the $\rho, \Delta$ guaranteed by Theorems \ref{thm:conv_rates} and \ref{thm:dist_fedavg_maml}. The results are in Figure \ref{fig:compare_fedavg_maml}.

For each \ServerOpt, the \maml frontier is a subset of the \fedavg frontier.
Recall that in Theorem \ref{thm:conv_rates}, we require $\gamma < (L+\alpha)^{-1}$ for \fedavg, but $\gamma < (KL+\alpha)^{-1}$ for \maml.
In Figure \ref{fig:compare_fedavg_maml} this causes the frontier for $\Theta_{K}$ to be more restrictive than for $\Theta_{1:K}$. However, it is still notable that these two fundamentally different methods, attain the same frontier when $\rho$ is large.

By Lemma \ref{lem:strongly_convex}, $\rho$ and $\Delta$ are still well-defined for $\Theta_K$ when $(KL+\alpha)^{-1} \leq \gamma < (L+\alpha)^{-1}$.
To understand what happens in this regime, we generate random symmetric $A \in \R^{d\times d}$ satisfying $\cond(A) = L/\mu$ and compute $Q(\alpha, \gamma, \Theta_K)$ as in \eqref{eq:Q_matrix}. We then compute $\kappa$ via \eqref{eq:kappa_QA}, and plug this into Table \ref{table:conv_local_update} and \eqref{eq:Delta} to get $\rho, \Delta$. This gives us a simulated Pareto frontier for $\Theta_K$, which we compare to $\Theta_{1:K}$ in Figure \ref{fig:compare_simulated_fedavg_maml}. For details and additional experiments, see Appendix \ref{appendix:simulate_maml}.

\begin{figure}
\centering
    \begin{subfigure}{.7\linewidth}
    \centering
    \includegraphics[width=\linewidth]{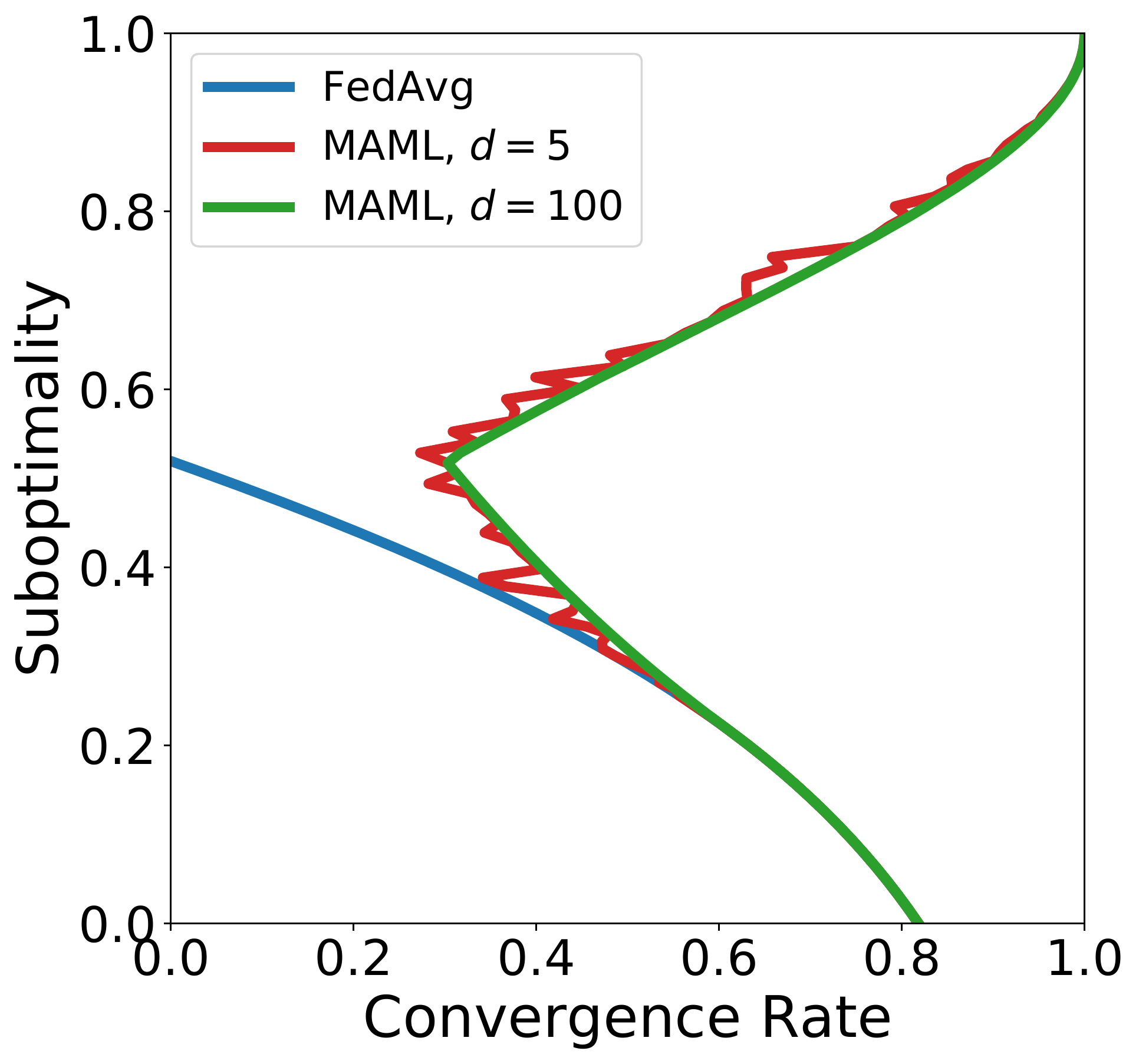}
    \end{subfigure}
\caption{Simulated Pareto frontiers for $\mu=1$, $L = 10$, $\alpha = 0$, $\gamma = 0.001$, $\Theta = \Theta_{K}$ (\maml). \ServerOpt is gradient descent. We randomly sample $A \in \R^{d\times d}$ with $\mu \preceq A \preceq L$, and compute $(\rho, \Delta)$ for various $K \in [1, 10^6]$ and $d \in \{5, 100\}$. We compare to the Pareto frontier for $\Theta_{1:K}$ (\fedavg).}
\label{fig:compare_simulated_fedavg_maml}
\end{figure}

The Pareto frontiers are identical for small $K$ (mirroring Figure \ref{fig:compare_fedavg_maml}), but diverge when $\gamma \geq (KL+\alpha)^{-1}$. The frontier for \maml then moves further from 0. Intuitively, \fedavg tries to learn a global model, while \maml tries to learn a model that adapts quickly to new tasks~\citep{finn2017model}; \maml need not minimize $(\rho, \Delta)$.
The \maml frontier is noisy for $d = 5$ (as $\rho$, $\Delta$ depend on random eigenvalues of $A$), but stabilizes for $d = 100$. While we posit that this reflects a semi-circle law for eigenvalues of random matrices \citep{alon2002concentration}, we leave an analysis to future work.

One final observation that highlights the similarities and differences of \fedavg- and \maml-style methods: In Figure \ref{fig:compare_fedavg_maml}, the curve for \maml when $d = 100$ has a clear cusp. This seems to occur at the same suboptimality (ie. $y$-value) as the intersection of the \fedavg curve with the $y$-axis. In other words, the behavior of \maml diverges substantially from \fedavg, but only after it reaches the same suboptimality as \fedavg for $K \to \infty$ (which corresponds to one-shot averaging). We are unsure why the suboptimality of one-shot averaging corresponds to a cuspidal operating point of \maml, but this observation highlights significant nuance in the behavior of these methods.

\section{Limitations and Discussion}

Our convergence-accuracy framework and the resulting Pareto frontiers can be useful tools in understanding how algorithmic choices impact local update methods. The obvious limitation is that they only apply to quadratic models. While this is restrictive, we show empirically in Appendix \ref{appendix:additional_experiments} that even for non-convex functions, the client learning rate governs a convergence-accuracy trade-off for \fedavg.

Our framework may also be useful in identifying important phenomena underlying \localupdate, even in non-quadratic settings. To demonstrate this, we show that many of the observations in Section \ref{sec:compare} hold in non-convex settings. We train a CNN on the FEMNIST dataset~\citep{caldas2018leaf} using \localupdate where $\Theta = \Theta_{1:50}$. We tune client and server learning rates. See Appendix \ref{appendix:experiment_setup} for full details. In Figure \ref{fig:emnist_results}, we illustrate how server momentum and $\alpha$ change convergence. Our results match the Pareto frontiers in Figures \ref{fig:compare_server_opt} and \ref{fig:compare_fedprox}: Server momentum improves convergence, while $\alpha$ has little to no effect, provided we tune learning rates.

\begin{figure}[ht]
\centering
    \begin{subfigure}{.49\linewidth}
    \centering
    \includegraphics[width=\linewidth]{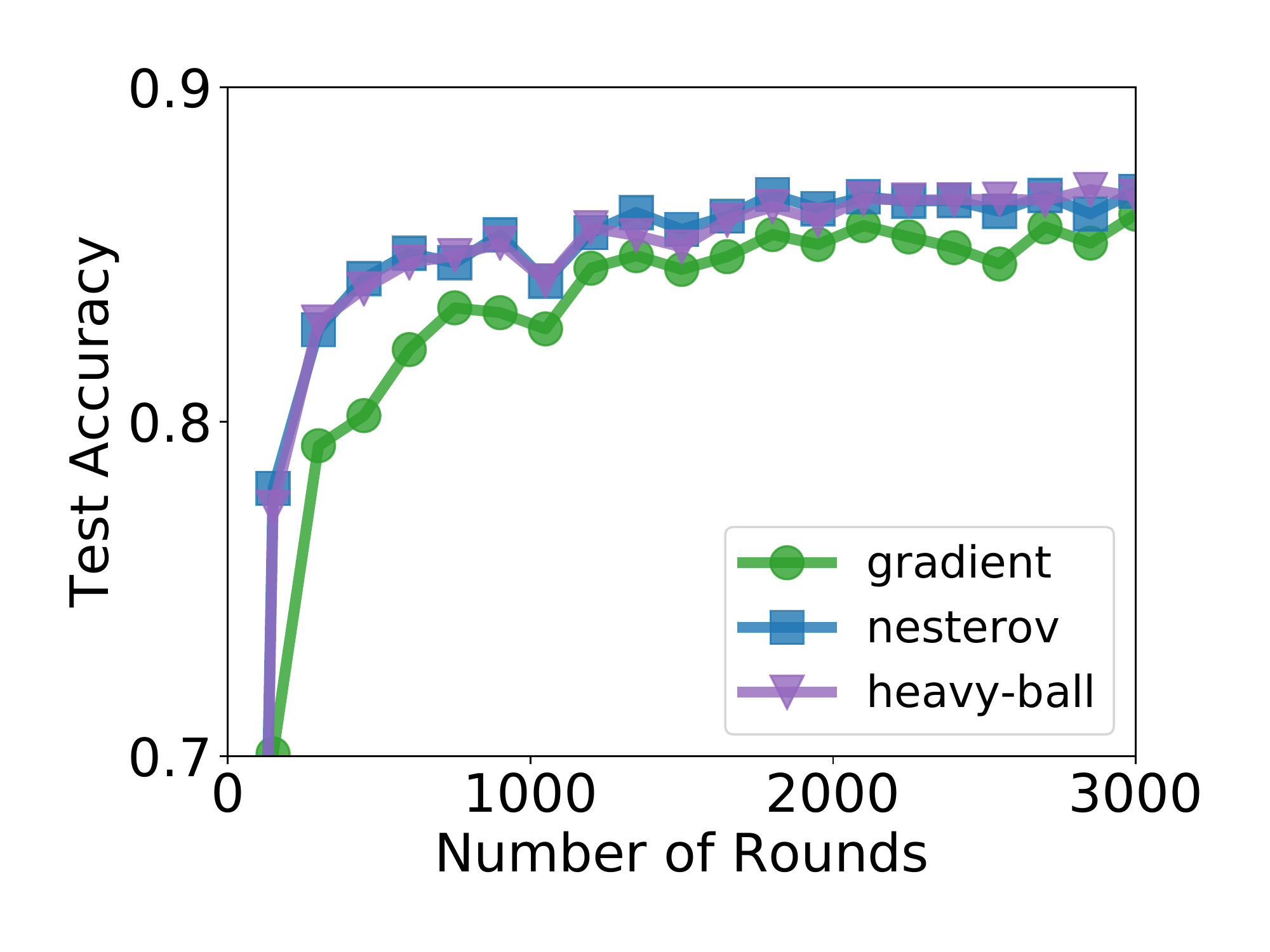}
    \end{subfigure}
    \begin{subfigure}{.49\linewidth}
    \centering
    \includegraphics[width=\linewidth]{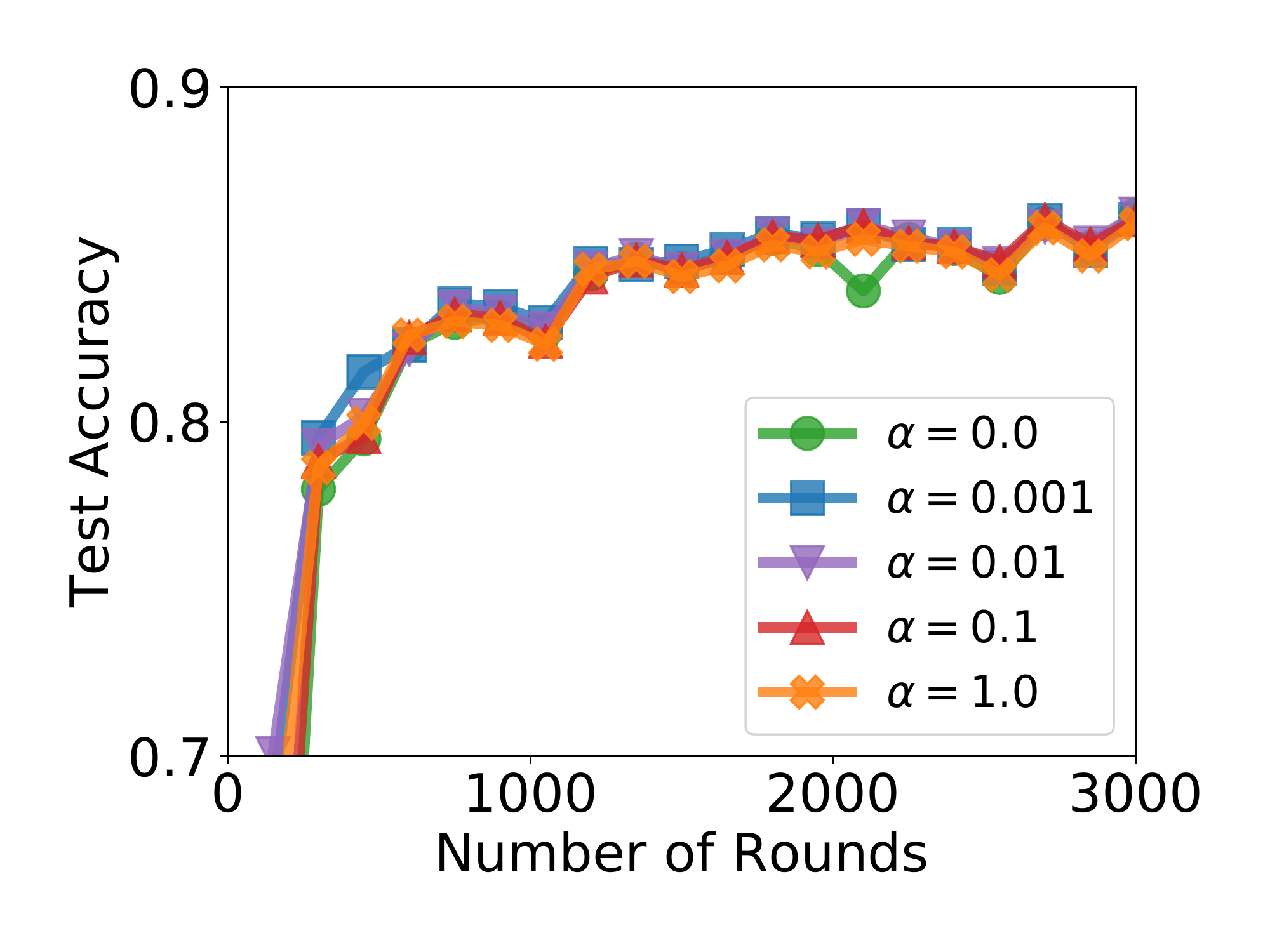}
    \end{subfigure}
\caption{Test accuracy of \localupdate with $\Theta = \Theta_{1:50}$ on FEMNIST with tuned learning rates. (Left) Varying types of server momentum, $\alpha = 0$. (Right) No momentum and varying $\alpha$.}
\label{fig:emnist_results}
\end{figure}

This brief example illustrates that our framework can identify crucial facets of local update methods. While our framework may not capture all relevant details of such methods, we believe it greatly simplifies their analysis, comparison, and design. In the future, we hope to extend this framework to more general loss functions. Other important extensions include stochastic settings with partial client participation, as well as trade-offs between convergence and post-adaptation accuracy of local update methods.

\bibliographystyle{plainnat}
\bibliography{main}

%% file: appendix.tex
\appendix

\onecolumn

\section{Relations between \fedavg, \fedprox, and \localupdate}\label{appendix:special_cases}

We focus on the following (simplified) version of \fedavg, otherwise known as Local \sgd~\citep{zinkevich2010parallelized, stich2018local}: At each iteration $t$, we sample some set of clients $I_t$ of size $M$ from the client population $\mI$. Each client $i \in I_t$ receives the server's model $x_t$, and applies $K$ steps of mini-batch \sgd to its local model, resulting in an updated local model $x_t^i$. The server receives these models from the sampled clients, and updates its model via
\[
x_{t+1} = \dfrac{1}{M}\sum_{i \in I_t} x_t^i.
\]
Fix $t$, and let $g_k^i$ denote the $k$-th mini-batch gradient of client $i$. Suppose we use a learning rate of $\gamma$ on each client when performing mini-batch \sgd. Then we have
\begin{align*}
    x_{t+1} &= \dfrac{1}{M}\sum_{i \in I_t} x_t^i\\
    &= x_t - \dfrac{1}{M}\sum_{i \in I_t}\left(x_t - x_t^i\right)\\
    &= x_t - \dfrac{1}{M}\sum_{i \in I_t}x_t - \left(x_t - \gamma \sum_{k = 1}^K g_k^i\right)\\
    &= x_t - \gamma\dfrac{1}{M}\sum_{i \in I_t} \sum_{k=1}^K g_k^i.
\end{align*}
A similar analysis holds for \fedprox, but with the usage of a proximal term with parameter $\alpha > 0$. In both cases, this is exactly \localupdate (see Algorithms \ref{alg:outerloop} and \ref{alg:innerloop}) with $\Theta = \Theta_{1:K}$, $\eta = \gamma$, and where \ServerOpt is gradient descent. However, by instead using a server learning rate of $\eta$ that is allowed to vary independently of $\gamma$ in \localupdate, we can obtain markedly different convergence behavior. We note that a form of this decoupling has previously been explored by \citet{karimireddy2019scaffold} and \citet{reddi2020adaptive}. However, these versions instead perform averaging on the so-called ``model delta'', in which the server model is updated via
\begin{align*}
x_{t+1} &= x_t - \dfrac{\eta}{M}\sum_{i \in I_t} (x_t - x_t^i)\\
&= x_t - \dfrac{\eta}{M}\sum_{i \in I_t}\left(x_t - \left(x_t - \gamma \sum_{k=1}^K g_k^i\right)\right)\\
&= x_t - \dfrac{\eta\gamma}{M}\sum_{i \in I_t} \sum_{k=1}^K g_k^i.
\end{align*}
Thus, while this does decouple $\eta$ and $\gamma$ to some degree, it does not fully do so. In particular, if we set $\gamma = 0$, then $x_{t+1} = x_t$, in which case we can make no progress overall. This is particularly important because, as implied by Theorem \ref{thm:sgd_objective}, for $K(\Theta) > 1$ we can only guarantee that the surrogate loss has the same critical points as the true loss by setting $\gamma = 0$. More generally, we see that the effective learning rate used in the model-delta approach is the product $\eta\gamma$. This can result in conflations between the effect of changing the server learning rate $\gamma$ and changing  the client learning rate $\eta$. By disentangling these, we can better understand differences in the impact of these parameters on the underlying optimization dynamics.

\section{Omitted Proofs}

\paragraph{Notation} For a matrix $A \in \R^{m \times n}$ we let $A^{\intercal} \in \R^{n \times m}$ denote its transpose. Similarly, given a vector $v \in \R^{d \times 1}$, we let $v^{\intercal} \in \R^{1 \times d}$ denote its transpose. We let $\norm{\cdot}$ denote the $\ell_2$ norm for vectors, and the spectral norm for matrices. For a real, symmetric matrix $A$, we let $\lambda_{\max}(A), \lambda_{\min}(A)$ denote the maximum and minimum eigenvalue of a matrix. We let $\preceq$ denote the Loewner order on symmetric positive semi-definite matrices.

\subsection{Proofs of Theorems \ref{thm:sgd_objective} and \ref{thm:maml}}

We first prove a general result about stochastic local updates on quadratic functions. Let $A \in \R^{d \times d}$ be symmetric and positive definite, and let $c \in \R^d$. Let
\[
h(x) = \frac{1}{2}\norm{A^{1/2}(x-c)}^2.
\]
Suppose we perform $K$ iterations of \sgd on $h$ with learning rate $\gamma$. That is, starting at $x_1$ we generate a sequence of independent random vectors $\{g_k\}_{k=1}^K$ and corresponding \sgd iterates $\{x_k\}_{k=1}^{K+1}$ satisfying, for $1 \leq k \leq K$,
\begin{equation}\label{eq:g_k}
\E[g_k] = \nabla h(x_k),
\end{equation}
\begin{equation}\label{eq:x_k}
    x_{k+1} = x_k - \gamma g_k.
\end{equation}

We then have the following lemma regarding the $g_k$.

\begin{lemma}\label{lem:local_update_quadratic}
    For all $k \geq 1$,
    \begin{equation}
        \E[g_{k+1}] = (I-\gamma A)\E[g_k].
    \end{equation}
    In particular, this implies
    \begin{equation}
        \E[g_{k}] = (I-\gamma A)^{k-1}A(x_1-c).
    \end{equation}
\end{lemma}
\begin{proof}
    By \eqref{eq:g_k}, the law of total expectation, and the independence of the $g_k$,
    \begin{align*}
        \E[g_k] = \E\left[ \E[ \nabla h(x_k)~|~g_1, \dots, g_{k-1}] \right] = \E[A(x_k-c)].
    \end{align*}
    By linearity of expectation,
    \begin{equation}\label{eq:expect_g_x}
        \E[g_k] = A(\E[x_k] - c).
    \end{equation}
    Combining \eqref{eq:x_k} and \eqref{eq:expect_g_x}, we have
    \begin{align*}
    \E[g_{k+1}] &= A(\E[x_{k+1}] - c)\\
    &= A(\E[x_k] - \gamma \E[g_k] - c)\\
    &= A(A^{-1}\E[g_k] - \gamma\E[g_k])\\
    &= (I-\gamma A)\E[g_k]
    \end{align*}
    This completes the first part of the proof. The second follows from noting that $\E[g_1] = A(x_1-c)$.
\end{proof}

Recall that for $i \in \mI$, $\alpha \geq 0, \gamma \geq 0$ and $\Theta = (\theta_1, \dots, \theta_{K(\Theta)})$ we define the matrix $Q_i(\alpha, \gamma, \Theta)$ by
\begin{equation}\label{eq:Q_restated}
Q_i(\alpha, \gamma, \Theta) := \sum_{k=1}^{K(\Theta)} \theta_k(I-\gamma(A_i + \alpha I))^{k-1}.
\end{equation}
We can now prove Theorem \ref{thm:sgd_objective}. For convenience, we restate the theorem here.

\thmobjective*
\begin{proof}
    Fix $i \in \mI$, $x_1$, and $\alpha \geq 0$. For $z \in \mZ$, define
    \[
    h_\alpha(x ; z) := f(x ; z) + \frac{\alpha}{2}\norm{x-x_1}^2,~~~h_\alpha(x) := \E_{z \sim \mD_i}[h_\alpha(x ;z)].
    \]
    
    Further define:
    \[
    B_{z, \alpha} := B_z + \alpha I,~~~b_{z, \alpha} := B_zc_z + \alpha x_1,~~~\tau_{z, \alpha} := \frac{1}{2}c_z^{\intercal}B_zc_z + \frac{\alpha}{2}\norm{x_1}^2.
    \]
    
    By \eqref{eq:quadratic_loss}, we have
    \begin{align*}
        h_\alpha(x ; z) &= \frac{1}{2}\norm{B_z^{1/2}(x-c_z)}^2 + \frac{\alpha}{2}\norm{x-x_1}^2\\
        &= \frac{1}{2}x^{\intercal}(B_z + \alpha I)x - x^{\intercal}(B_zc_z +\alpha x_1) + \frac{1}{2}c_z^{\intercal}B_zc_z + \frac{\alpha}{2}\norm{x_1}^2\\
        &= \frac{1}{2}x^{\intercal}B_{z, \alpha}x -x^{\intercal}b_{z, \alpha} + \tau_{z, \alpha}.
    \end{align*}
    
    Therefore,
    \begin{equation}\label{eq:h_alpha}
        h_\alpha(x) = \frac{1}{2}x^{\intercal}\E_{z \sim \mD_i}[B_{z, \alpha}]x - x^{\intercal}\E_{z \sim \mD_i}[b_{z, \alpha}] + \E_{z \sim\mD_i}[\tau_{z, \alpha}].
    \end{equation}
    
    Define
    \[
    A_\alpha := \E_{z \sim \mD_i}[B_{z, \alpha}],~~~c_\alpha := A_\alpha^{-1}\E_{z \sim\mD_i} [b_{z, \alpha}],~~~ \tau_\alpha := \E_{z \sim\mD_i} [\tau_{z, \alpha}].
    \]
    
    Note that $A_\alpha = A_i + \alpha I$, where $A_i$ is as in \eqref{eq:A_i_c_i}. Straightforward manipulation of \eqref{eq:h_alpha} implies
    \begin{equation}
    h_\alpha(x) = \dfrac{1}{2}\norm{A_\alpha^{1/2}(x-c_\alpha)}^2 + \tau_\alpha.
    \end{equation}
    
    Note that the stochastic gradients $g_1, \ldots, g_K$ computed in Algorithm \ref{alg:innerloop} are therefore independent stochastic gradients of $h_\alpha$. By applying Lemma \ref{lem:local_update_quadratic} and noting that the constant term $\tau_\alpha$ does not impact these stochastic gradients, we have that for $k \geq 1$,
    \begin{equation}\label{eq:h_alpha_grads}
    \E[g_k] = (I-\gamma A_\alpha)^{k-1}A_\alpha(x-c_\alpha).
    \end{equation}
    
    Expanding and using the fact that in Algorithm \ref{alg:innerloop}, $x_1 = x$, we have
    \begin{align*}
        \E[g_k] &= (I-\gamma A_\alpha)^{k-1}A_\alpha\left(x - A_\alpha^{-1}\E_{z \sim \mD_i}[B_zc_z + \alpha x]\right)\\
        &= (I-\gamma A_\alpha)^{k-1}\left(A_\alpha x - \E_{z \sim\mD_i}[B_zc_z] - \alpha x\right)\\
        &= (I-\gamma (A_i + \alpha I))^{k-1}A_i(x-c_i).
    \end{align*}
    This last step follows from \eqref{eq:A_i_c_i}. Taking a sum and using the linearity of expectation,
    \begin{align*}
        \E[\ClientUpdate(i, x, \alpha, \gamma, \Theta)] &= \sum_{k=1}^{K(\Theta)}\E[g_k]\\
        &= \sum_{k=1}^{K(\Theta)}\theta_k(I-\gamma (A_i + \alpha I))^{k-1}A_i(x-c_i)\\
        &= Q_i(\alpha, \gamma, \Theta) A_i(x-c_i)\\
        &= \nabla \tilde{f}_i(x, \alpha, \gamma, \Theta).
    \end{align*}
\end{proof}

A similar analysis using Lemma \ref{lem:local_update_quadratic} can be used to derive Theorem \ref{thm:maml}, which we also restate.

\thmmaml*
\begin{proof}[Proof of Theorem \ref{thm:maml}]
    Fix $i \in \mI$, and for convenience of notation, let $X_k := X_k^i(x), X_1 = x$. Thus, for $k \geq 1$ we have
    \[
    X_{k+1} = X_k - \gamma \nabla_{X_k} f_i(X_k).
    \]
    Since $\nabla^2 f_i(y) = A_i$ for all $y$, we have
    \begin{equation}\label{eq:maml_grad1}
    \nabla_{X_k} X_{k+1} = I - \gamma A_i.
    \end{equation}
    By Lemma \ref{lem:local_update_quadratic}, we also have
    \begin{equation}\label{eq:maml_grad2}
    \nabla_{X_k} f_i(X_k) = (I-\gamma A_i)^{K-1}A_i(x-c_i).
    \end{equation}
    For a function $\psi: \R^a \to \R^b$, let its Jacobian at a point $x \in \R^a$ be denoted by $J_x(\psi)$. Applying \eqref{eq:maml_grad1}, \eqref{eq:maml_grad2} and the chain rule, we have
    \begin{align*}
        (\nabla m_K^i(x))^{\intercal} &= J_{X_1}(f_i(X_{K+1}))\\
        &= J_{X_{K+1}}(f_i(X_{K+1}))J_{X_1}(X_{K+1})\\
        &= J_{X_{K+1}}(f_i(X_{K+1}))\prod_{k=1}^K J_{X_k}(X_{k+1})\\
        &= [(I-\gamma A_i)^KA_i(x-c_i)]^{\intercal}\prod_{k=1}^K (I-\gamma A_i)^{\intercal}\\
        &= \left((I-\gamma A_i)^{2K}A_i(x-c_i)\right)^{\intercal}.
    \end{align*}
    The result follows from applying \eqref{eq:Q_restated} for $\Theta = \Theta_{2K+1}$.
\end{proof}

\subsection{Proofs of Lemmas \ref{lem:strongly_convex}, \ref{lem:cond_general}, \ref{lem:cond_fedavg}, \ref{lem:cond_maml}, and \ref{lem:dist_fedavg_maml}}

These lemmas will follow from a spectral analysis of $Q_i(\alpha, \gamma, \Theta)$. We defer the proof of Lemma \ref{lem:opt_dist} to Appendix \ref{appendix:distance_proof} due to its more elaborate nature. We first state a general result about eigenvalues of expected values of matrices.

\begin{lemma}\label{lem:expect_matrix}
    Suppose we have a probability space $(\Omega, \mF, \mP)$ where $\Omega$ is the set of symmetric matrices in $\R^{d\times d}$, and let $B$ be a random matrix drawn from $\mP$. Suppose that $\E[B]$ exists and is finite. Define
    \[
    \tau_1 := \E[\lambda_{\min}(B)],~~\tau_2 := \E[\lambda_{\max}(B)].
    \]
    Then $\tau_1I \preceq \E[B] \preceq \tau_2I$.
\end{lemma}
\begin{proof}
    By standard properties of expectations, $\E[B]$ is a symmetric matrix.
    Fix $v \in \R^d$ such that $\norm{v} = 1$. By the linearity of expectation,
    \[
    v^{\intercal}\E[B]v = \E[v^{\intercal}Bv] \leq \E[\lambda_{\max}(B)].
    \]
    An analogous argument shows $v^{\intercal}\E[B]v \geq \E[\lambda_{\min}(B)]$. The result follows.
\end{proof}

We also compute the spectrum of $Q_i(\alpha, \gamma, \Theta)$ and $Q_I(\alpha, \gamma, \Theta)A_i$.

\begin{lemma}\label{lem:eig_Q}
    For each eigenvalue $\lambda$ of $A_i$, $Q_i(\alpha, \gamma, \Theta)$ has an eigenvalue
    \begin{equation}\label{eq:eig_Q}
    \sum_{k=1}^{K(\Theta)} \theta_k(1-\gamma(\lambda+\alpha))^{k-1}
    \end{equation}
    and $Q_i(\alpha, \gamma, \Theta)A_i$ has an eigenvalue
    \begin{equation}\label{eq:eig_QA}
    \sum_{k=1}^{K(\Theta)} \theta_k(1-\gamma(\lambda+\alpha))^{k-1}\lambda,
    \end{equation}
    both with the same multiplicity as $\lambda$.
\end{lemma}
\begin{proof}
    Let $v$ be an eigenvector of $A_i$ with eigenvalue $\lambda$. Then $v$ is an eigenvector of $(I-\gamma (A_i+\alpha I))^{k-1}$ with eigenvalue $(1-\gamma(\lambda+\alpha))^{k-1}$, implying \eqref{eq:eig_Q} is an eigenvalue of $Q_i(\alpha, \gamma, \Theta)$ with eigenvector $v$. Similarly, we note that $v$ is an eigenvector of $(I-\gamma(A_i + \alpha I))^{k-1}A_i$ with eigenvalue $(1-\gamma(\lambda+\alpha))^{k-1}\lambda$, implying \eqref{eq:eig_QA} is an eigenvalue of $Q_i(\alpha, \gamma, \Theta)A_i$ with eigenvector $v$. The statement about multiplicities follows directly.
\end{proof}

Lemma \ref{lem:eig_Q} can be used in a straightforward manner to prove Lemma \ref{lem:strongly_convex}, which we restate and prove below.

\lemstrcvx*
\begin{proof}
    By Assumption \ref{assm1}, we have that for every eigenvalue $\lambda$ of $A_i$, $\lambda \leq L$. Therefore, for any such $\lambda$,
    \[
    1-\gamma(\lambda + \alpha) \geq 1-\gamma(L+\alpha) > 0.
    \]
    Since the $\theta_k$ are nonnegative and not all zero by assumption, we see that \eqref{eq:eig_Q} is a sum of nonnegative terms, at least one of which must be positive. Therefore, all eigenvalues of $Q_i(\alpha, \gamma, \Theta)$ are positive, and $Q_i(\alpha, \gamma, \Theta)$ is therefore symmetric and positive definite.
    
    By Assumption \ref{assm1}, $A_i$ is also symmetric and positive definite, hence the product $Q_i(\alpha, \gamma, \Theta)A_i$ is symmetric positive definite. However, since
    \[
    \nabla^2_x \tilde{f}_i(x, \alpha, \gamma, \Theta) = Q_i(\alpha, \gamma, \Theta)A_i \succ 0
    \]
    we see that $\tilde{f}_i$ is a strongly convex quadratic function.
\end{proof}

We can also use \ref{lem:eig_Q} to derive Lemma \ref{lem:cond_general}. In fact, we will show a slightly stronger version, where we also bound the Lipschitz and strong convexity parameters of $\tilde{f}$.

\begin{lemma}\label{lem:cond_general_2}
    Suppose $\gamma < (L+\alpha)^{-1}$. Define
    \begin{equation}\label{eq:QA_L}
    \tilde{L}(\alpha, \gamma, \Theta) := \E_i[\lambda_{\max}(Q_i(\alpha, \gamma, \Theta)A_i)],
    \end{equation}
    \begin{equation}\label{eq:QA_mu}
    \tilde{\mu}(\alpha, \gamma, \Theta) := \E[\lambda_{\min}(Q_i(\alpha, \gamma, \Theta)A_i)].
    \end{equation}
    Then $\tilde{f}(\alpha, \gamma, \Theta)$ is $\tilde{L}(\alpha, \gamma, \Theta)$-Lipschitz and $\tilde{\mu}(\alpha, \gamma, \Theta)$-strongly convex.
\end{lemma}
\begin{proof}
    Recall that the smoothness and strong convexity parameters of are the largest and smallest eigenvalue of its Hessian. Therefore, it suffices to bound the eigenvalues of $\tilde{f}(\alpha, \gamma, \Theta)$. By \eqref{eq:surrogate_i}, we have
    \[
    \nabla^2\tilde{f}_i(x, \alpha, \gamma, \Theta) = Q_i(\alpha, \gamma, \Theta)A_i.
    \]
    By \eqref{eq:surrogate_loss}, we also have
    \[
    \nabla^2\tilde{f}(\alpha, \gamma, \Theta) = \E_i[\nabla^2\tilde{f}_i(\alpha, \gamma, \Theta)] = \E_i[Q_i(\alpha, \gamma, \Theta)A_i].
    \]
    Applying Lemma \ref{lem:expect_matrix}, we conclude the proof.
\end{proof}

Note that $\kappa(\alpha, \gamma, \Theta)$ in \eqref{eq:kappa_QA} is simply the ratio of $\tilde{L}(\alpha, \gamma, \Theta)$ to $\tilde{\mu}(\alpha, \gamma, \Theta)$, so we immediately derive Lemma \ref{lem:cond_general} as a corollary to Lemma \ref{lem:cond_general_2}. Lemmas \ref{lem:cond_fedavg} and \ref{lem:cond_maml} are also straightforward consequences of Lemma \ref{lem:cond_general_2}, as we show below. For posterity, we state and prove each one separately.

Recall that as in \eqref{eq:phi}, we define
\[
\phi(\lambda, \alpha, \gamma, K) := \sum_{k=1}^K(1-\gamma(\lambda+\alpha))^{k-1}\lambda. 
\]
We can now restate and prove Lemma \ref{lem:cond_fedavg}.

\condfedavg*
\begin{proof}
    By Lemma \ref{lem:cond_general_2}, it suffices to derive upper and lower bounds on the eigenvalues of $Q_i(\alpha, \gamma, \Theta_{1:K})A_i$. Fix $i \in \mI$. By Lemma \ref{lem:eig_Q}, we see that the eigenvalues of $Q_i(\alpha, \gamma, \Theta_{1:K})A_i$ are exactly of the form $\phi(\lambda, \alpha, \gamma, K)$, where $\lambda$ is an eigenvalue of $A_i$. By Assumption \ref{assm1}, each such $\lambda$ satisfies $\lambda \in [\mu, L]$ where $\mu > 0$.
    
    Fix $\alpha, \gamma, K$, and define $g(\lambda) := \phi(\lambda, \alpha, \gamma, K)$. Since $\gamma(L+\alpha) < 1$, basic properties of geometric sums imply that for $\lambda \in [\mu, L]$,
    \begin{equation}
        g(\lambda) = \phi(\lambda, \alpha, \gamma, K) = \dfrac{1-(1-\gamma(\lambda+\alpha))^K}{\gamma}\dfrac{\lambda}{\lambda+\alpha}.
    \end{equation}
    For a given $\lambda$, let $\xi = 1-\gamma(\lambda+\alpha)$.
    Simple but tedious computations show that if we take a derivative with respect to $\lambda$, we have
    \begin{align*}
        \gamma(\lambda+\alpha)g'(\lambda) &= K\lambda\gamma\xi^{K-1} + (1-\xi^K)\dfrac{\alpha}{\lambda + \alpha}.
    \end{align*}
    Note that since $\gamma < (L+\alpha)^{-1}$ by assumption, $0 \leq \xi \leq 1$ for $\lambda \in [\mu, L]$. Therefore, $g'(\lambda) \geq 0$ for $\lambda \in [\mu, L]$, so any eigenvalue $\chi$ of $Q_i(\alpha, \gamma, \Theta_{1:K})A_i$ must satisfy
    \[
    \phi(\mu, \alpha, \gamma, K) = g(\mu) \leq \chi \leq g(L) = \phi(L, \alpha, \gamma, K).
    \]
    The result then follows by Lemma \ref{lem:cond_general_2}.
\end{proof}

Recall that in \eqref{eq:psi}, we defined
\[
\psi(\lambda, \alpha, \gamma, K) := (1-\gamma(\lambda + \alpha))^{K-1}\lambda.
\]
We use a similar proof as that of Lemma \ref{lem:cond_fedavg} to prove Lemma \ref{lem:cond_maml}, which we restate and prove below.

\condmaml*
\begin{proof}
    By Lemma \ref{lem:cond_general_2}, it suffices to bound the eigenvalues of $Q_i(\alpha, \gamma, \Theta_K)A_i$. Fix $i \in \mI$. By Lemma \ref{lem:eig_Q}, we see that the eigenvalues of $Q_i(\alpha, \gamma, \Theta_K)$ are of the form $\psi(\lambda, \alpha, \gamma, K)$ where $\lambda$ is an eigenvalue of $A_i$. Note that by Assumption \ref{assm1}, any such $\lambda$ satisfies $\lambda \in [\mu, L]$.
    
    Fix $\alpha, \gamma, K$, and define $h(\lambda) := \psi(\lambda, \alpha, \gamma, K)$. Let $\zeta = 1-\gamma(\lambda+\alpha)$. Straightforward computations show
    \[
    h'(\lambda) = \zeta^{K-2}(1-\gamma(K\lambda + \alpha)).
    \]
    Since $\gamma < (KL+\alpha)^{-1}$, we in particular have $\gamma < (L+\alpha)^{-1}$ so $0 \leq \zeta \leq 1$ for $\lambda \in [\mu, L]$. Since $\gamma < (KL+\alpha)^{-1}$, we also have the term $1-\gamma(K\lambda + \alpha) \geq 0$ for $\lambda \in [\mu, L]$. Thus, $h'(\lambda) \geq 0$ for $\lambda \in [\mu, L]$. Thus, any eigenvalue $\chi$ of $Q_i(\alpha, \gamma, \Theta_K)A_i$ must satisfy
    \[
    \psi(\mu, \alpha, \gamma, K)  = h(\mu) \leq \chi \leq h(L) = \psi(L, \alpha, \gamma, K) .
    \]
    The result then follows by Lemma \ref{lem:cond_general_2}.
\end{proof}

Finally, we are now equipped to prove Lemma \ref{lem:dist_fedavg_maml}, which we restate here for posterity.

\qcond*
\begin{proof}
    This will follow almost immediately from Lemmas \ref{lem:cond_fedavg}, \ref{lem:cond_maml}, and \ref{lem:eig_Q}. First, consider the case $\Theta = \Theta_{1:K}$. Then by Lemma \ref{lem:eig_Q} and Assumption \ref{assm1}, we see that
    \[
    \lambda_{\max}(Q_i(\alpha, \gamma, \Theta_{1:K})) \leq \sum_{i=1}^K (1-\gamma(\mu + \alpha))^{K-1} = \dfrac{\phi(\mu, \alpha, \gamma, K)}{\mu}
    \]
    and
    \[
    \lambda_{\min}(Q_i(\alpha, \gamma, \Theta_{1:K})) \geq \sum_{i=1}^K (1-\gamma(L + \alpha))^{K-1} = \dfrac{\phi(L, \alpha,\gamma, K)}{L}.
    \]
    Therefore,
    \[
    \cond(Q_i(\alpha, \gamma, \Theta_{1:K})) \leq \dfrac{\phi(\mu, \alpha, \gamma, K)}{\phi(L, \alpha, \gamma, K)}\dfrac{L}{\mu} = \kappa_0\kappa(\alpha, \gamma, \Theta_{1:K})^{-1}.
    \]
    Here we used the fact that $\kappa_0 := L/\mu$ and Lemma \ref{lem:cond_fedavg}. An almost identical proof gives the analogous result for $\Theta = \Theta_K$.
\end{proof}

\subsection{Tightness of Lemmas \ref{lem:cond_fedavg} and \ref{lem:cond_maml}}\label{appendix:tight_lemmas}

In fact, Lemmas \ref{lem:cond_fedavg} and \ref{lem:cond_maml} are tight. Fix any $\alpha \geq 0$ and $\gamma < (\alpha+L)^{-1}$. Let $\mP$ be supported on a single client $i$, and let this client's dataset be supported on a single example $z$ where
\[
B_z = \begin{pmatrix} L &0\\0 & \mu\end{pmatrix}, c_z = \begin{pmatrix}0\\0\end{pmatrix}.
\]
Note that by \eqref{eq:A_i_c_i}, we then have $A_i = B_i, c_i = 0$. In fact, we will show that in this case, the bounds on the condition numbers given in Lemma \ref{lem:cond_fedavg} and \ref{lem:cond_maml} are tight. By direct computation,
\begin{align*}
Q_i(\alpha, \gamma, \Theta_{1:K})A_i = \begin{pmatrix} \phi(L, \alpha, \gamma, K) & 0\\ 0 & \phi(\mu, \alpha, \gamma, K)\end{pmatrix}.
\end{align*}

If $\gamma < (L+\alpha)^{-1}$ then similar reasoning to the proof of Lemma \ref{lem:cond_fedavg} implies that the condition number satisfies
\begin{align*}
    \cond(\tilde{f}_i(\alpha, \gamma, \Theta_{1:K})) &= \dfrac{\lambda_{\max}(Q_i(\alpha, \gamma, \Theta_{1:K})A_i)}{\lambda_{\min}(Q_i(\alpha, \gamma, \Theta_{1:K})A_i)} = \dfrac{\phi(L, \alpha, \gamma, K)}{\phi(L, \alpha, \gamma, K)}.
\end{align*}

Similarly, we have that
\begin{align*}
Q_i(\alpha, \gamma, \Theta_{K})A_i = \begin{pmatrix} \psi(L, \alpha, \gamma, K) & 0\\ 0 & \psi(\mu, \alpha, \gamma, K)\end{pmatrix}.
\end{align*}

By analogous reasoning to the proof of Lemma \ref{lem:cond_maml}, if $\gamma < (KL+\alpha)^{-1}$, we have
\begin{align*}
    \cond(\tilde{f}_i(\alpha, \gamma, \Theta_{1:K})) &= \dfrac{\lambda_{\max}(Q_i(\alpha, \gamma, \Theta_{K})A_i)}{\lambda_{\min}(Q_i(\alpha, \gamma, \Theta_{K})A_i)} = \dfrac{\psi(L, \alpha, \gamma, K)}{\psi(L, \alpha, \gamma, K)}.
\end{align*}

\section{Proof of Lemma \ref{lem:opt_dist}}\label{appendix:distance_proof}

In order to prove the results in this section, we will use the following straightforward lemma regarding the structure of $x^*(\alpha, \gamma, \Theta)$.

\begin{lemma}\label{lem:opt_point}
Suppose $\gamma < (L+\alpha)^{-1}$. Then
\[
    x^*(\alpha, \gamma, \Theta) = \E\left[Q_i(\alpha, \gamma, \Theta)A_i\right]^{-1}\E\left[Q_i(\alpha, \gamma, \Theta)A_ic_i\right].
\]
\end{lemma}

\begin{proof}
By definition,
\[
\tilde{f}(x, \alpha, \gamma, \Theta) = \E\sbr*{\frac{1}{2}\norm{(Q_i(\alpha, \gamma, \Theta)A_i)^{1/2}(x-c_i)}^2 }.
\]
Therefore,
\begin{align*}
    \nabla \tilde{f}(x, \alpha, \gamma, \Theta) = \E\sbr{Q_i(\alpha, \gamma, \Theta)A_i(x-c_i)} = \E\sbr{Q_i(\alpha, \gamma, \Theta)A_i}x - \E\sbr{Q_i(\alpha, \gamma, \Theta)A_ic_i}.
\end{align*}
Since $\gamma < (L+\alpha)^{-1}$, we know by Lemma \ref{lem:strongly_convex} that $\tilde{f}$ is strongly convex in $x$. It then follows that
\[
x^*(\alpha, \gamma, \Theta) = \E\sbr{Q_i(\alpha, \gamma, \Theta)A_i}^{-1}\E\sbr{Q_i(\alpha, \gamma, \Theta)A_ic_i}.
\]
\end{proof}

To prove Lemma \ref{lem:opt_dist}, we will reduce it to a statement about mean absolute deviations of bounded random variables. We define the mean absolute deviation of a random variable below.

\begin{definition}\label{def:mad}
Let $X$ be a random variable in some Banach space over $\R$. The mean absolute deviation of $X$ is
\[
D(X) := \E[\norm{X - \E[X]}].
\]
\end{definition}

\subsection{One-dimensional Case}

We first proceed for $d = 1$, so that $Q_i(\alpha, \gamma, \Theta), A_i, c_i \in \R$. We do this both for expository purposes, as when $d = 1$ we can rely on classical versions of the mean absolute deviation, and because we actually derive tight bounds for $d = 1$.

To derive our results, we bound the mean absolute deviation of bounded random variables. Our result is inspired by the bound by \citet{bhatia2000better} on the variance of bounded random variables.

\begin{theorem}\label{thm:better_bound_mad}
Suppose $X$ is a discrete random variable taking values in $[a, b] \subseteq \R$ for $a < b$. Then
\[
D(x) \leq \dfrac{2(b-\E[X])(\E[X] - a)}{b-a}.
\]
Moreover, this holds with equality iff $X$ is supported on $\{a, b\}$.
\end{theorem}

\begin{proof}
    Suppose $X$ takes on values $x_1, \dots, x_n$ with probabilities $p_1, \dots p_n$. We will first show that there is a random variable $Y$ supported on $\{a, b\}$ such that $D(X) \leq D(Y)$.
    
    Without loss of generality, suppose $x_1 \in (a, b)$. Define
    \[
    s := p_1 \dfrac{b-x_1}{b-a},~~~t := p_1\dfrac{x_1-a}{b-a}.
    \]
    First note that $s, t \in [0, 1]$. Simple analysis also shows
    \begin{equation}\label{eq:p_cond1}
    s + t = p_1\end{equation}
    \begin{equation}\label{eq:p_cond2}
    sa + tb = p_1x_1.
    \end{equation}
    Let $X'$ be the random variable taking on values $a, x_2, \dots, x_n, b$ with probabilities $s, p_2, \dots p_n, t$. Note that these probabilities are nonnegative and sum to 1 by \eqref{eq:p_cond1}. By \eqref{eq:p_cond2}, we also have $\E[X'] = \E[X]$.
    
    It is straightforward to show that $D(X') \geq D(X)$. Suppose $x_1 \geq \E[X]$. Then we have
    \begin{align*}
    D(X') - D(X) &= s(\E[X]-a) + t(b-\E[X]) - p_1(x_1-\E[X])\\
    &= p_1\left(\dfrac{(b-x_1)(\E[X]-a)}{b-a} + \dfrac{(x_1-a)(b-\E[X])}{b-a} - (x_1-\E[X])  \right)\\
    &= \dfrac{p_1}{b-a}\left((b-x_1)(\E[X]-a) + (x_1-a)(b-\E[X]) - (b-a)(x_1-\E[X]) \right)\\
    &= \dfrac{p_1}{b-a}\left((b-x_1)(2\E[X]-a-x_1) + (x_1-a)(b-x_1)\right)\\
    &= \dfrac{2p_1(b-x_1)(\E[X]-a)}{b-a}\\
    &\geq 0.
    \end{align*}
    
    An analogous argument shows that $D(X') \geq D(X)$ when $x_1 \leq \E[X]$. A thorough examination of the derivation above shows that this inequality is strict if and only if $x_1 \in (a, b)$ and $p_1 > 0$. Thus, $D(X') \geq D(X)$ and $X'$ has one fewer possible outcome in the range $(a, b)$.
    
    By iterating this procedure, (which is guaranteed to terminate after at most $n$ iterations), we obtain some random variable $Y$ supported on $\{a, b\}$ such that $D(X) \leq D(Y)$, with equality if and only if $X$ is already supported on $\{a, b\}$. Suppose $Y$ takes on $a, b$ with probabilities $(1-p), p$. Straightforward calculation shows
    \[
    D(Y)  = 2(1-p)p(b-a) = \dfrac{2(b-\E[X])(\E[X] - a)}{b-a}.
    \]
\end{proof}

Using the fact that for any real $x$, $(b-a)^2 \geq 4(b-x)(x-a)$ (with equality iff $x = (b+a)/2$) we arrive at the following corollary, analogous to Popoviciu's inequality on variances~\citep{popoviciu1935equations}.
\begin{corollary}
If $X$ is a discrete random variable on $[a, b]$, then
\[
D(x) \leq \frac{1}{2}(b-a),
\]
with equality iff $X$ takes on the values $a$ and $b$, each with probability $1/2$.
\end{corollary}

We can now prove a stronger version of Lemma \ref{lem:opt_dist} when $d = 1$. For simplicity, we assume $\mP$ is a discrete distribution on some finite $\mI$, though the analysis can be generalized to arbitrary distributions.

\begin{lemma}\label{lem:opt_dist_d1}
Let $b = \max_{i \in \mI}\lambda_{\max}(Q_i(\alpha, \gamma, \Theta))$, $a = \min_{i \in \mI}\lambda_{\min}(Q_i(\alpha, \gamma, \Theta))$. Then for $d = 1$,
\[
\norm{x^*(\alpha, \gamma, \Theta)-x^*} \leq 2C\dfrac{\sqrt{b}-\sqrt{a}}{\sqrt{b}+\sqrt{a}}.
\]
\end{lemma}

\begin{proof}
    Suppose $|\mI| = [n]$ and $\mP$ is the discrete distribution on $\mI$ with associated probabilities $p_1, \ldots, p_n$. Without loss of generality, assume $p_1, \dots, p_n > 0$. For brevity, we will let $Q_i$ denote $Q_i(\alpha, \gamma, \Theta)$. Note that since $d = 1$, $Q_i, A_i, c_i$ are all elements of $\R$. By Assmptions \ref{assm1} and \ref{assm2}, and by definition of $a, b$, we have that for all $i$, $0 < \mu \leq A_i \leq L$, $Q_i \in [a, b]$, and $|c_i| \leq C$. Moreover, by Lemma \ref{lem:strongly_convex}, we must have $a > 0$.
    
    Let $q_i = p_iA_i$. Note that $q_i \geq p_i\mu > 0$. Thus, the $q_i$ define a unique distribution $\mQ$ over $\mI$ where $i$ is sampled with probability proportional to $q_i$. Let $Z$ denote the random variable $Q_i$ where $i \sim \mQ$ and let $v = \E[Z]$. Thus,
    \[
    v = \E[Z] = \dfrac{\sum_{i=1}^n q_iQ_i}{\sum_{i=1}^n q_i} = \dfrac{\sum_{i=1}^n p_iQ_iA_i}{\sum_{i=1}^n p_iA_i}.
    \]
    Since each $Q_i \in [a, b]$, we also have $v \in [a, b]$.By Lemma \ref{lem:opt_point},
    \begin{align*}
        \norm{x^*(\alpha, \gamma, \Theta) - x^*} &= \norm{\dfrac{\sum_{i=1}^n q_iQ_ic_i}{\sum_{i=1}^nq_iQ_i} - \dfrac{\sum_{i=1}^n q_ic_i}{\sum_{i=1}^nq_i}}\\
        &= \norm{\dfrac{\sum_{i=1}^n q_iQ_ic_i - v\sum_{i=1}^nq_ic_i}{\sum_{i=1}^nq_iQ_i}}\\
        &= \norm{\dfrac{\sum_{i=1}^n(Q_i-v)q_ic_i}{\sum_{i=1}^n q_iQ_i}}\\
        &= \norm{\dfrac{\sum_{i=1}^n(Q_i-v)q_ic_i}{\sum_{i=1}^n q_iQ_i}\dfrac{\sum_{i=1}^n q_i}{\sum_{i=1}^nq_i}}\\
        &= \dfrac{\norm{\E_{i \sim \mQ}(Q_i-v)c_i}}{v}\\
        &\leq \dfrac{\E_{i \sim \mQ}[\norm{(Q_i-v)c_i}]}{v}\\
        &\leq \dfrac{C\E[\norm{Z-v}]}{v}.
    \end{align*}
    
    Since $v = \E[Z]$, we can apply Theorem \ref{thm:better_bound_mad}, finding
    \begin{align*}
    \norm{x^*(\alpha, \gamma, \Theta) - x^*} &\leq \dfrac{2C(b-v)(v-a)}{v(b-a)}.
    \end{align*}
    Maximizing the right-hand side for $v \in [a, b]$, we get
    \begin{align*}
        \norm{x^*(\alpha, \gamma, \Theta) - x^*} &\leq 2C\dfrac{\sqrt{b}-\sqrt{a}}{\sqrt{b}+\sqrt{a}}.
    \end{align*}
\end{proof}

\subsection{Tightness of Lemma \ref{lem:opt_dist_d1}}\label{appendix:tight_distance}

In this section, we will show that in some sense, Lemma \ref{lem:opt_dist_d1} is tight. Specifically, we will show that for all $\epsilon > 0$, there are $A_1, A_2 \in \R$, $c_1, c_2 \in \R$ such that $\norm{c_i} \leq 1$, a distribution $\mP$ over $\{1, 2\}$, and $\alpha, \gamma, \Theta$ such that for
\[
a = \min\{\lambda_{\min}(Q_1(\alpha, \gamma, \Theta) | i = 1, 2\},~~~b = \max\{\lambda_{\max}(Q_i(\alpha, \gamma, \Theta) | i=1, 2\},
\]
we have
\[
\norm{x^*(\alpha, \gamma, \Theta) - x^*} \geq 2\dfrac{\sqrt{b}-\sqrt{a}}{\sqrt{b} + \sqrt{a}} - \epsilon.
\]

Let $A_1 = 4, A_2 = 1, c_1 = 1, c_2 = -1$. Let $\alpha = 0, \gamma = 1/8$. Then by \eqref{eq:Q_matrix}, we have
\[
Q_1(\alpha, \gamma, \Theta_K) = \left(\frac{1}{2}\right)^K,~~Q_2(\alpha, \gamma, \Theta_K) = \left(\frac{7}{8}\right)^K.
\]
Therefore, in this setting, $b = (7/8)^K, a = 2^{-K}$. By applying L'Hopital's rule, we find
\begin{equation}\label{eq:tight1}
\lim_{K \to \infty} 2\dfrac{\sqrt{b}-\sqrt{a}}{\sqrt{b} + \sqrt{a}} = 2.
\end{equation}
Let $\mP$ be the distribution that selects $i = 1$ with probability $p$ and $i = 2$ with probability $1-p$. Then
\begin{align*}
\norm{x^*(\alpha, \gamma, \Theta_K) - x^*} &= \left|\dfrac{4pQ_1 - (1-p)Q_2}{4pQ_1 + (1-p)Q_2} - \dfrac{5p-1}{3p+1}\right|\\
&= \left|\dfrac{4p\left(\frac{1}{2}\right)^K -(1-p)\left(\frac{7}{8}\right)^K}{4p\left(\frac{1}{2}\right)^K +(1-p)\left(\frac{7}{8}\right)^K} - \dfrac{5p-1}{3p+1}\right|.
\end{align*}

For any fixed $p \in (0, 1)$, straightforward but tedious applications of L'Hopital's rule yields the fact that
\begin{equation}\label{eq:tight2}
\lim_{K \to \infty}\norm{x^*(\alpha, \gamma, \Theta_K) - x^*} = \dfrac{8p}{3p+1}
\end{equation}
By \eqref{eq:tight1} and \eqref{eq:tight2}, we see that by selecting $K$ sufficiently large and $p$ sufficiently close to 1, we can ensure that
\[
\norm{x^*(\alpha, \gamma, \Theta_K) - x^*} \geq 2-\epsilon
\]
and that
\[
2\dfrac{\sqrt{b}-\sqrt{a}}{\sqrt{b} + \sqrt{a}} \geq 2-\epsilon
\]
thus implying that
\[
\norm{x^*(\alpha, \gamma, \Theta_K) - x^*} \geq 2\dfrac{\sqrt{b}-\sqrt{a}}{\sqrt{b} + \sqrt{a}}-\epsilon.
\]

\subsection{General Case}

In this case, we will use a similar proof strategy. However, we will use a matrix-weighted version of the mean absolute deviation. Suppose we have positive-definite symmetric matrices $X_1, \dots, X_n, Y_1, \dots Y_n \in \R^{d\times d}$ such that for all $i$, $X_i$ and $Y_i$ commute. Let $Y = \sum_{i=1}^n Y_i$.

\begin{definition}\label{def:matrix_mean}
The matrix-weighted mean of $\{X_1, \dots X_n\}$ with respect to $\{Y_1,\dots Y_n\}$ is given by
\[
f(X_1, \dots , X_n | Y_1,\dots, Y_n) := \left(\sum_{i=1}^n X_iY_i\right)Y^{-1}.
\]
\end{definition}
When $d = 1$, and the $Y_i > 0$, this gives the standard mean of a discrete random variable $X$ taking values $X_1, \dots X_n$ with probabilities $Y_1,\dots Y_n$. When the context is clear, we will simply denote this by $f(X | Y)$. We first prove a simple lemma regarding the Loewner ordering and matrix-weighted means.

\begin{lemma}\label{lem:loewner}
Suppose that for all $i$, $aI \preceq X_i \preceq bI$. Then $aI \preceq f(X|Y) \preceq bI$.
\end{lemma}

This generalizes the fact that if $X$ is a random variable taking values in $[a, b]$, then $a \leq \E[X] \leq b$ to symmetric positive-definite matrices.

\begin{proof}
Since the $X_i, Y_i$ are commuting positive definite matrices and $Y$ is positive definite, $f(X|Y)$ is similar to the matrix
\[
P = Y^{-1/2}\left(\sum_{i=1}^n Y_i^{1/2}X_iY_i^{1/2}\right)Y^{-1/2}.
\]
By assumption, $X_i \preceq bI$. Therefore,
\[
Y_i^{1/2}X_iY_i^{1/2} \preceq Y_i^{1/2}bIY_i^{1/2} = bY_i.
\]
Hence, we have
\[
\sum_{i=1}^n Y_i^{1/2}X_iY_i^{1/2} \preceq bY.
\]
Thus,
\[
P \preceq bY^{-1/2}YY^{-1/2} = bI.
\]
An analogous argument shows that $P \succeq aI$. By basic properties of matrix similarity, we therefore find $aI \preceq f(X|Y) \preceq bI$.
\end{proof}

We can use this matrix-weighted mean to define a normalized, matrix-weighted version of the mean absolute deviation.

\begin{definition}\label{def:matrix_mad}
The normalized matrix-weighted discrepancy of $\{X_1, \dots X_n\}$ with respect to $\{Y_1,\dots Y_n\}$ is given by
\[
M(X_1, \dots X_n | Y_1, \dots, Y_n) = \sum_{i=1}^n\norm{Y^{-1}f(X | Y)^{-1}(X_i-f(X | Y))Y_i}
\]
where $\norm{\cdot}$ is the operator norm.
\end{definition}

Note that when $d = 1$, $M(X | Y) = D(X)/|\E[X]|$, where $D(X)$ is as in Definition \ref{def:mad} and $X$ takes on values $X_1, \dots X_n > 0$ with probabilities $Y_1, \dots Y_n > 0$.

We will prove an analog of Theorem \ref{thm:better_bound_mad} for this normalized matrix-weighted discrepancy.

\begin{theorem}\label{thm:matrix_mad}
Let $X_1, \dots X_n$ be symmetric matrices in $\R^{d\times d}$ satisfying $aI \preceq X_i \preceq bI$ for all $i$, and suppose $Y_1, \dots Y_n$ are symmetric, positive-definite matrices in $\R^{d\times d}$. Then
\[
M(X_1, \dots X_n | Y_1,\dots Y_n) \leq \dfrac{2(b-a)}{b}.
\]
\end{theorem}

To prove this, we will require a straightforward lemma regarding eigenvalues of symmetric positive definite matrices.
\begin{lemma}\label{lem:pd_inv}
Let $P_1, P_2$ be symmetric positive definite matrices and let $P = P_1 + P_2$. Then
\[
\norm{P_1P^{-1}} \leq 1.
\]
\end{lemma}
\begin{proof}
By basic properties of the Loewner ordering,
\[
P_1 \preceq P \implies P^{-1/2}P_1P^{-1/2} \preceq P^{-1/2}PP^{-1/2} = I.
\]
Since $P_1P^{-1}$ is similar to $P^{-1/2}P_1P^{-1/2}$, we have
\[
\norm{P_1P^{-1}} = \norm{P^{-1/2}P_1P^{-1/2}} \leq 1.
\]
\end{proof}

We can now prove Theorem \ref{thm:matrix_mad}.

\begin{proof}[Proof of Theorem \ref{thm:matrix_mad}]
We will proceed in a similar manner to the proof of Theorem \ref{thm:better_bound_mad}. We will first show that we can always find a set of symmetric positive definite matrices $X_1', \dots, X_m', Y_1', \dots, Y_m'$ such that:
\begin{enumerate}
    \item For all $i$, $aI \preceq X_i' \preceq bI$.
    \item For all $i$, $X_i', Y_i'$ commute.
    \item $\sum_i Y_i' = \sum_i Y_i$.
    \item $f(X' |Y') = f(X | Y)$.
    \item $M(X' |Y') \geq M(X | Y)$.
    \item For $t = |\{X_1, \dots , X_n\}\backslash\{aI, bI\}|$, we have $|\{X_1', \dots , X_m'\}\backslash\{aI, bI\}| \leq \max\{0, t-1\}$.
\end{enumerate}

By iterating this procedure, we can replace $X_1, \dots, X_n, Y_1, \dots, Y_n$ with matrices $X''_1,\dots, X''_l, Y''_1,\dots, Y''_l$ where the $X''_i$ are all in the set $\{aI, bI\}$. It will then suffice to show that $M(X''|Y'')$ satisfies the desired bound, which we do by a somewhat direct computation, though one that is made much easier due to the fact that $aI, bI$ are diagonal.

We now proceed in detail. Define the matrices
\[
S := \dfrac{bI-X_1}{b-a}Y_1,~~~T:= \dfrac{X_1-aI}{b-a}Y_1.
\]

Note that since $X_1, Y_1$ commute, $S$ and $T$ are products of symmetric, positive definite, commuting matrices. They are therefore symmetric, positive definite, commuting matrices as well. One can easily verify that
\begin{equation}\label{eq:Y_cond1}
S + T = Y_1
\end{equation}
and
\begin{equation}\label{eq:Y_cond2}
aS + bT = X_1Y_1.
\end{equation}

By \eqref{eq:Y_cond1} we have
\begin{equation}\label{eq:Y_cond3}
S + Y_2 + \dots + Y_n + T = Y.
\end{equation}
and combining \eqref{eq:Y_cond2} and \eqref{eq:Y_cond3}, we have
\begin{equation}\label{eq:Y_cond4}
f(aI, X_2, \dots, X_n, bI | S, Y_2, \dots Y_n, T) = f(X_1, \dots, X_n | Y_1, \dots, Y_n).
\end{equation}
We will use $Z$ to denote the matrices in \eqref{eq:Y_cond4}. Note that by Lemma \ref{lem:loewner}, we know that $aI \preceq Z \preceq bI$. We will show that this replacement of $(X_1, Y_1)$ by $(aI, S)$ and $(bI, T)$ does not decrease the mean absolute deviation. By \eqref{eq:Y_cond3} and \eqref{eq:Y_cond4},
\begin{align*}
    & M(aI, X_2, \dots, X_n, bI | S, Y_2, \dots, Y_n, T) - M(X_1,\dots, X_n | Y_1, \dots, Y_n)\\
    &= \norm{Y^{-1}Z^{-1}(Z-aI)S} + \norm{Y^{-1}Z^{-1}(bI-Z)T} - \norm{Y^{-1}Z^{-1}(Z-X_1)Y_1}\\
    &= \dfrac{\norm{Y^{-1}Z^{-1}(Z-aI)(bI-X_1)Y_1} + \norm{Y^{-1}Z^{-1}(bI-Z)(X_1-aI)Y_1} - \norm{Y^{-1}Z^{-1}(bI-aI)(Z-X_1)Y_1}}{b-a}\\
\end{align*}

Define
\begin{align*}
    T_1 &= Y^{-1}Z^{-1}(Z-aI)(bI-X_1)Y_1,\\
    T_2 &= Y^{-1}Z^{-1}(bI-Z)(X_1-aI)Y_1,\\
    T_3 &= Y^{-1}Z^{-1}(bI-aI)(Z-X_1)Y_1.
\end{align*}

Note that since $aI \preceq Z \preceq bI, aI\preceq X_1 \preceq bI$ and the $Y_i$ are positive definite, $T_1$ and $T_2$ are positive semi-definite matrices. Simple algebraic manipulation implies that $T_3 = T_1 - T_2$. Since $T_1, T_2$ are positive definite matrices, we have
\[
\norm{T_1-T_2} \leq \max\left\{ \lambda_{\max}(T_1) - \lambda_{\min}(T_2), \lambda_{\max}(T_2) - \lambda_{\min}(T_1)\right\} \leq \lambda_{\max}(T_1) + \lambda_{\max}(T_2) = \norm{T_1} + \norm{T_2}.
\]
Thus, 
\[
M(aI, X_2, \dots, X_n, bI | S, Y_2, \dots, Y_n, T) \geq M(X_1,\dots, X_n | Y_1, \dots, Y_n).
\]

We therefore exhibit exactly the matrices satisfying properties (1)-(6) described above. By iterating this procedure, we obtain positive definite, symmetric matrices $(X_1'',\dots X_m''), (Y_1'',\dots, Y_m'')$ such that $X_i'', Y_i''$ commute, each $X_i''$ is equal to $aI$ or $bI$, and such that
\[
M(X_1'', \dots, X_m'' | Y_1'' \dots Y_m'') \geq M(X_1, \dots X_n | Y_1 \dots Y_n).
\]

By consolidating $X_i''$ that are equal, we can assume without loss of generality that we have matrices $(aI, bI)$ with associated symmetric positive definite matrices $(C_1, C_2)$. Let $C = C_1 + C_2$, and let $R = aC_1 + bC_2$. We then have
\begin{align*}
f(aI, bI | C_1, C_2) = RC^{-1}.
\end{align*}
Let $Z = RC^{-1}$. Then by direct computation,
\begin{align*}
    M(aI, bI | C_1, C_2) &= \norm{R^{-1}(aI-Z)C_1} + \norm{R^{-1}(bI-Z)C_2}\\
    &= \norm{(aR^{-1} - C^{-1})C_1} + \norm{(C^{-1}-bR^{-1})C_2}.
\end{align*}
After some straightforward but tedious algebraic manipulation, we find
\begin{align*}
    M(aI, bI | C_1, C_2) &= (b-a)\norm{R^{-1}C_2C^{-1}C_1} + (b-a)\norm{C^{-1}C_1R^{-1}C_2}\\
    &\leq (b-a)\norm{R^{-1}C_2}\norm{C^{-1}C_1} + (b-a)\norm{C^{-1}C_1}\norm{R^{-1}C_2}.
\end{align*}

Since $C_1, C_2$ are symmetric positive definite matrices, we have
\[
\norm{C^{-1}C_1} = \lambda_{\min}(CC_1^{-1})^{-1} = \dfrac{1}{\lambda_{\min}(I + C_2C_1^{-1})} = \dfrac{1}{1+ \lambda_{\min}(C_2C_1^{-1})}.
\]
An analogous computation shows that
\[
\norm{R^{-1}C_2} = \dfrac{1}{a\lambda_{\min}(C_1C_2^{-1}) + b}.
\]

Letting $p, q$ denote $\lambda_{\min}(C_2C_1^{-1}), \lambda_{\min}(C_1C_2^{-1})$ respectively, and noting that we therefore have $p, q > 0$, we have
\begin{align*}
    M(aI, bI | C_1, C_2) &\leq \dfrac{2(b-a)}{(1+q)(ap+b)} \leq \dfrac{2(b-a)}{b}.
\end{align*}

\end{proof}

We can now prove Lemma \ref{lem:opt_dist}.

\begin{proof}
Suppose $|\mI| = n$ and $\mP$ is the discrete distribution on $\mI$ with associated probabilities $p_i$. For brevity, we will let $Q_i$ denote $Q_i(\alpha, \gamma, \Theta)$. We will let $Y_i = p_iA_i$ and $Y = \sum_{i=1}^n Y_i$. By Lemma \ref{lem:opt_point}, we have
\begin{align*}
    \norm{x^*(\alpha, \gamma, \Theta) - x^*} &= \norm{\left(\sum_{i=1}^n Q_iY_i\right)^{-1}\left(\sum_{i=1}^n Q_iY_ic_i\right) + \left(\sum_{i=1}^n Y_i\right)^{-1}\left(\sum_{i=1}^n Y_ic_i\right)}\\
    &= \norm{\left(\sum_{i=1}^n Q_iY_i\right)^{-1}\left(\left( \sum_{i=1}^n Q_iY_ic_i\right) + \left(\sum_{i=1}^n Q_iY_i\right)Y^{-1}\left(\sum_{i=1}^n Y_ic_i\right)\right)}\\
    &= \norm{\left(\sum_{i=1}^n Q_iY_i\right)^{-1}\left(\sum_{i=1}^n (Q_i-f(Q | Y))Y_ic_i \right)}\\
    &\leq \sum_{i=1}^n \norm{Y^{-1}f(Q|Y)^{-1}(Q_i-f(Q|Y))Y_i}\norm{c_i}\\
    &\leq CM(Q|Y).
\end{align*}

Here we used Lemma \ref{lem:eig_Q}, which in particular shows that since $\gamma < (L+\alpha)^{-1}$, all the $Q_i$ are positive definite. Moreover, Lemma \ref{lem:eig_Q} shows that $Q_i, A_i$ share the same eigenvectors, and therefore commute with one another. Hence, the $Q_i$ commute with the $Y_i$. Moreover, by Assumption \ref{assm1}, the $Y_i$ are positive definite symmetric matrices, and by Assumption \ref{assm2}, $\norm{c_i} \leq C$. Applying Theorem \ref{thm:matrix_mad}, we have
\begin{align*}
    \norm{x^*(\alpha, \gamma, \Theta) - x^*} &\leq 2C\dfrac{b-a}{b} \leq 8C\dfrac{\sqrt{b}-\sqrt{a}}{\sqrt{b}+\sqrt{a}}
\end{align*}

The last inequality holds from simple algebraic manipulation.

\end{proof}

\section{Additional Pareto Frontiers}\label{appendix:extra_comparisons}

\subsection{Simulated \maml-style Pareto Frontiers}\label{appendix:simulate_maml}

In order to plot Pareto frontiers for \maml-style methods ($\Theta = \Theta_K$) when $\gamma > (KL+\alpha)^{-1}$, we generate random symmetric matrices $A \in \R^{d\times d}$ satisfying $\mu I \preceq A \preceq L I$. We then compute the associated matrix $Q(\alpha, \gamma, \Theta_K)$. Once we have these matrices, we can compute the condition number $\kappa$ of $Q(\alpha, \gamma, \Theta_K)A$. We can then compute $\rho, \Delta$ by substituting $\kappa$ into Table \ref{table:conv_local_update} and Lemma \ref{lem:opt_dist}.

To generate $A$, we generate $B$ by sampling its entries independently from $\mathcal{N}(0, 1)$. We then set $A = \beta_1B^{\intercal}B + \beta_2I$, where $\beta_1, \beta_2$ are the unique scalars such that
\[
\mu = \lambda_{\min}(A) \leq \lambda_{\max}(A) = L.
\]

We plot the resulting Pareto frontiers for varying $L$ and fixed $\mu$ in Figure \ref{fig:extended_sim_maml}. We see that as $L$ increases with respect to $\mu$, the discrepancy between the \maml curves and the \fedavg curves grows. In particular, for small $L/\mu$, we see that the \maml curve recovers most of the \fedavg curve before diverging, while for large $L/\mu$, the two diverge almost immediately. Again, we see that when $d = 5$, there is some noise in $(\rho, \Delta)$, which seems to approach some limiting behavior for $d = 100$.

\begin{figure}[ht]
\centering
    \begin{subfigure}{.4\linewidth}
    \centering
    \includegraphics[width=\linewidth]{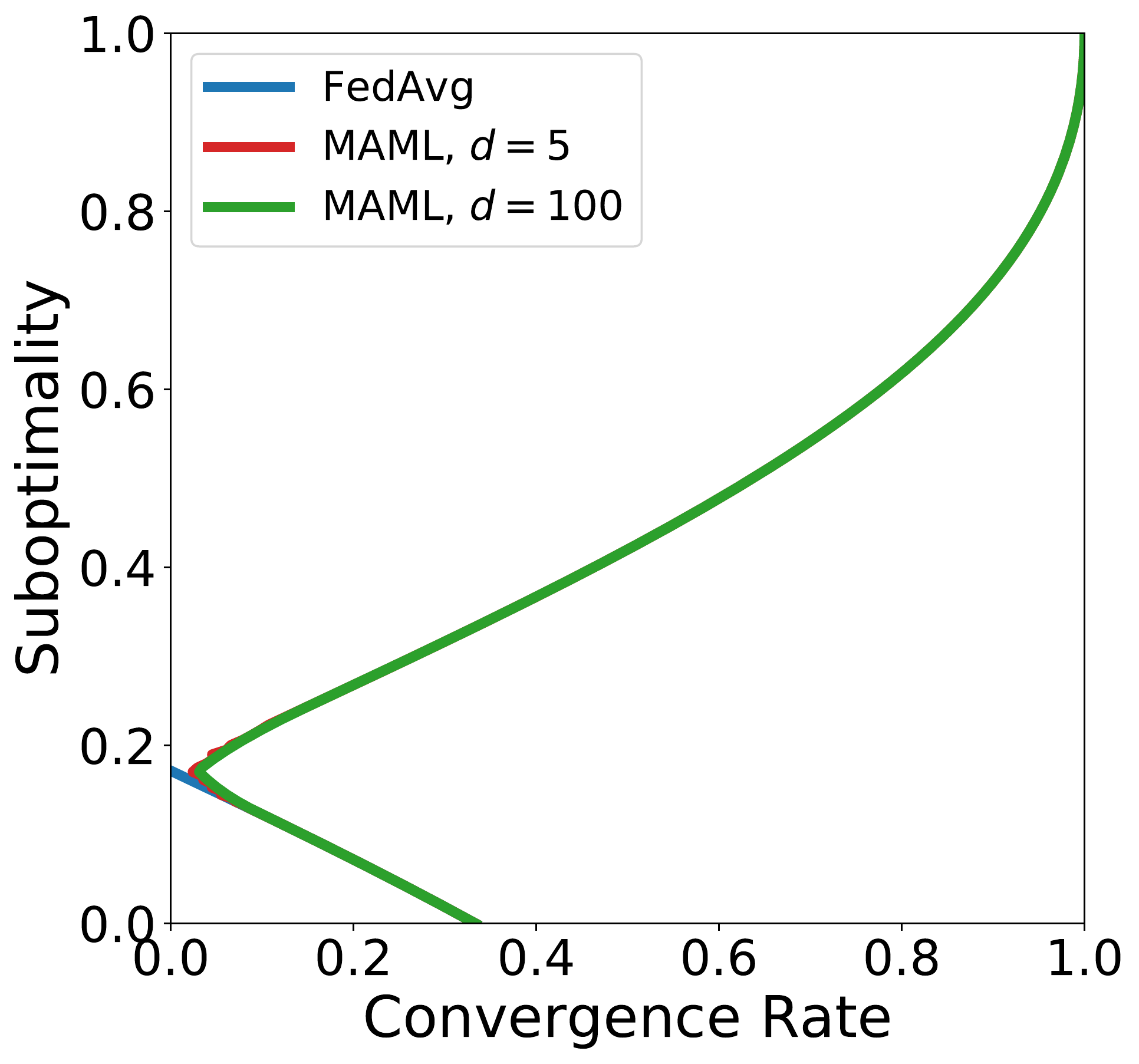}
    \caption{$L = 2$}
    \end{subfigure}
    \begin{subfigure}{.4\linewidth}
    \centering
    \includegraphics[width=\linewidth]{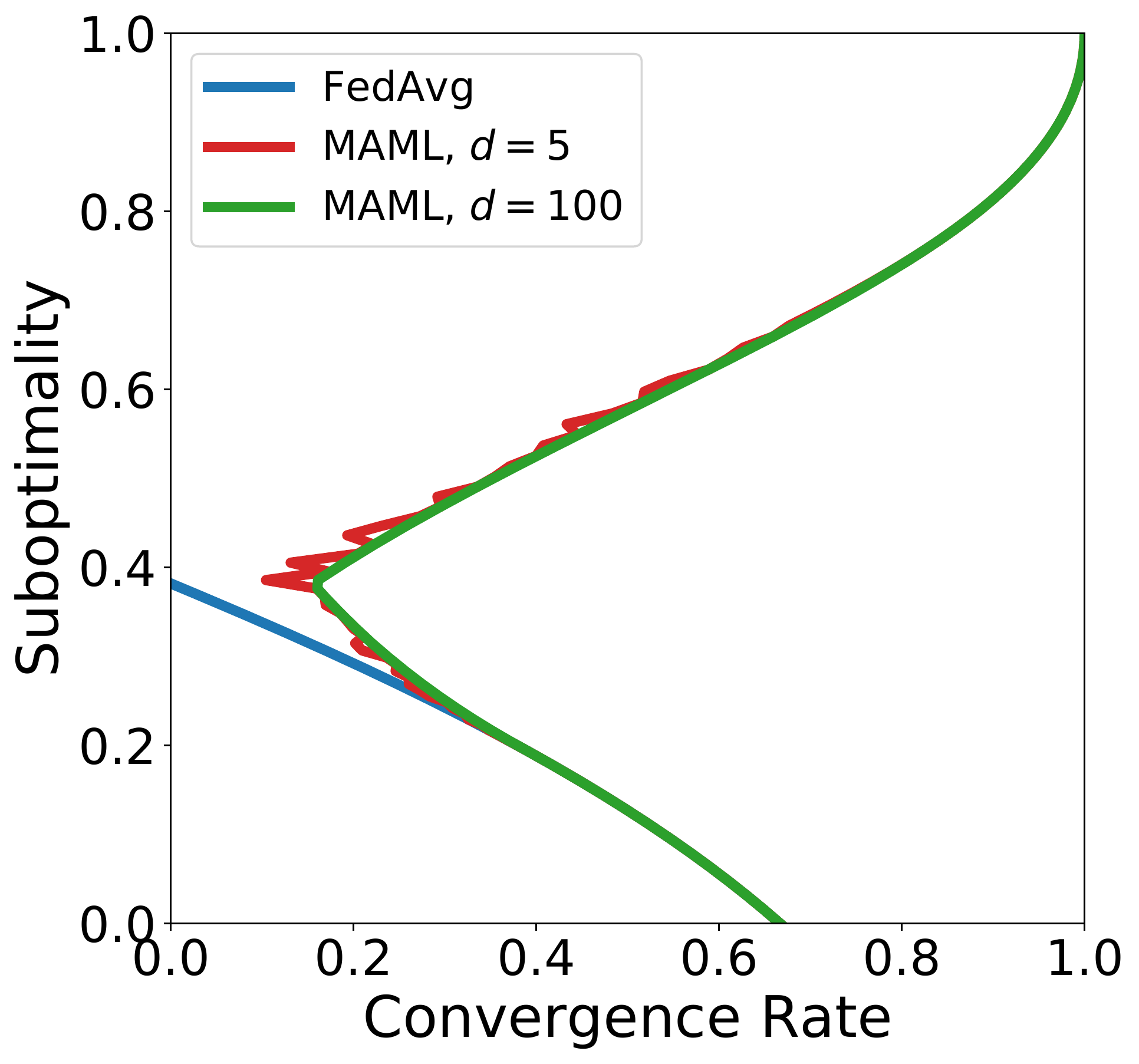}
    \caption{$L = 5$}
    \end{subfigure}
    \begin{subfigure}{.4\linewidth}
    \centering
    \includegraphics[width=\linewidth]{figures/simulated_maml_gradient_k10.pdf}
    \caption{$L = 10$}
    \end{subfigure}
    \begin{subfigure}{.4\linewidth}
    \centering
    \includegraphics[width=\linewidth]{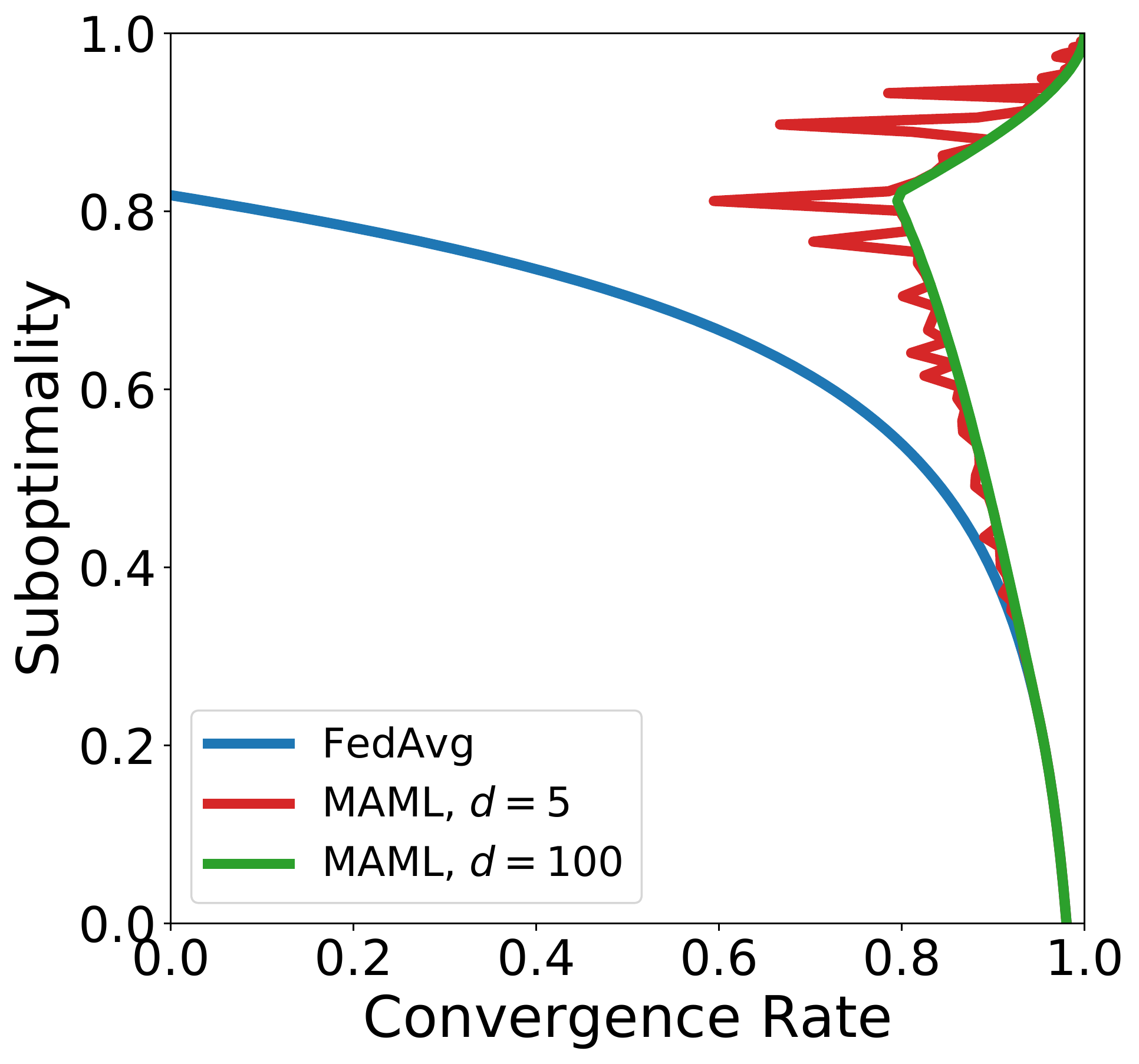}
    \caption{$L = 100$}
    \end{subfigure}
\caption{Simulated \maml Pareto frontiers for $\mu = 1, \gamma = 0.001, \Theta = \Theta_K$, varying $L$, and where \ServerOpt is gradient descent. We vary $L \in \{2, 5, 10, 100\}$ and $K \in [1, 10^6]$. We randomly generate $A \in \R^{d\times d}$ with $\mu I \preceq A \preceq LI$ and compute the associated $(\rho, \Delta)$. We also plot the Pareto frontier for $\Theta_{1:K}$ and the same $L$.}
\label{fig:extended_sim_maml}
\end{figure}

We perform a similar experiment, but where we fix $L = 10$ and vary \ServerOpt over gradient descent with no momentum, with Nesterov momentum, and with heavy-ball momentum. The results are given in Figure \ref{fig:sim_maml_momentum}. While the differences are not huge, we see that momentum helps convergence in all cases, \fedavg or \maml. Moreover, we see an interesting phenomenon where the type of momentum changes the concavity of the \maml Pareto frontier for $d = 100$. As we add momentum, the region to the right of the \maml curve becomes more convex, becoming more rounded for heavy-ball momentum than for Nesterov momentum.

\begin{figure}[ht]
\centering
    \begin{subfigure}{.31\linewidth}
    \centering
    \includegraphics[width=\linewidth]{figures/simulated_maml_gradient_k10.pdf}
    \caption{\ServerOpt: Gradient descent with no momentum}
    \end{subfigure}
    \begin{subfigure}{.31\linewidth}
    \centering
    \includegraphics[width=\linewidth]{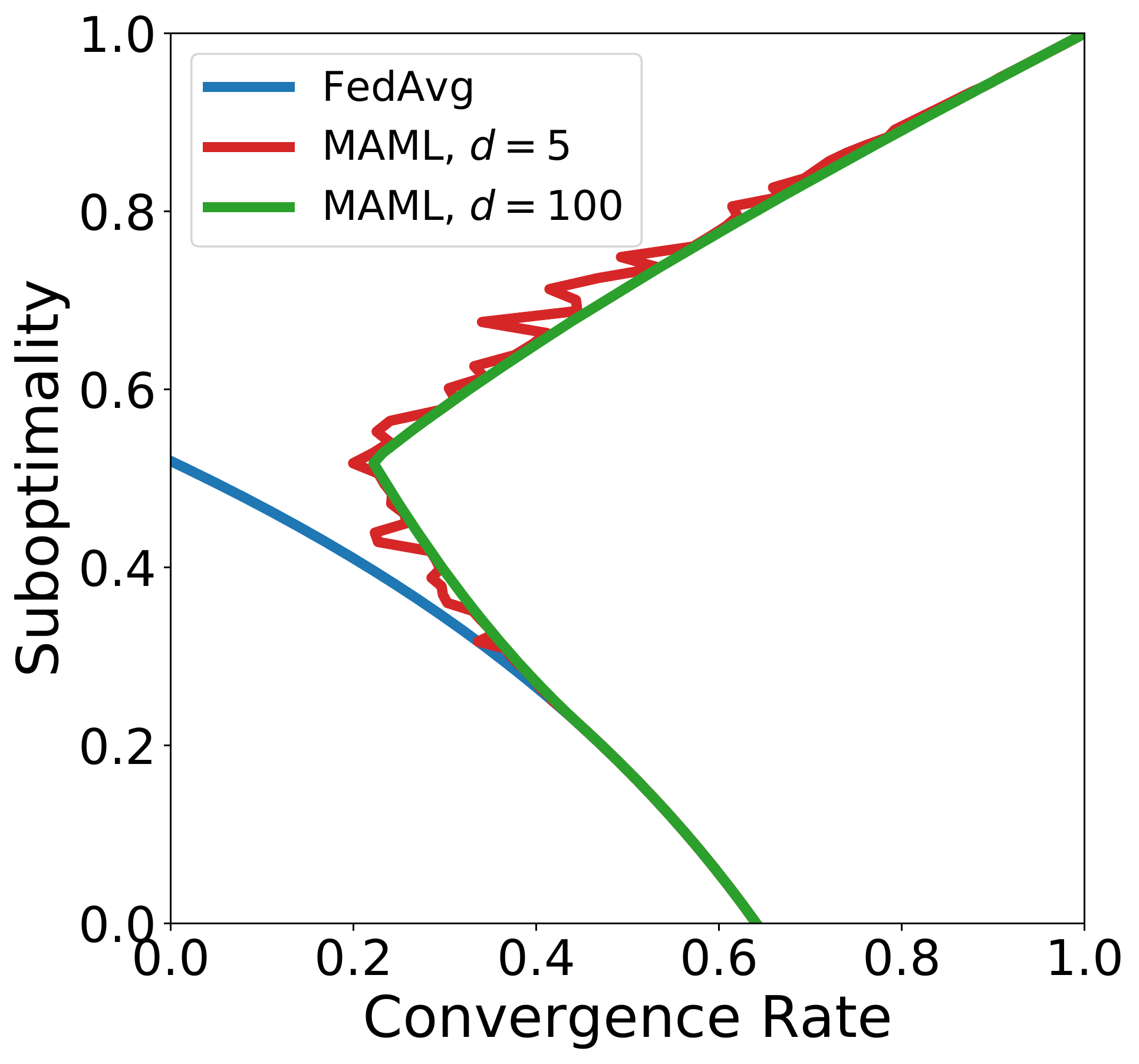}
    \caption{\ServerOpt: Gradient descent with Nesterov momentum}
    \end{subfigure}
    \begin{subfigure}{.31\linewidth}
    \centering
    \includegraphics[width=\linewidth]{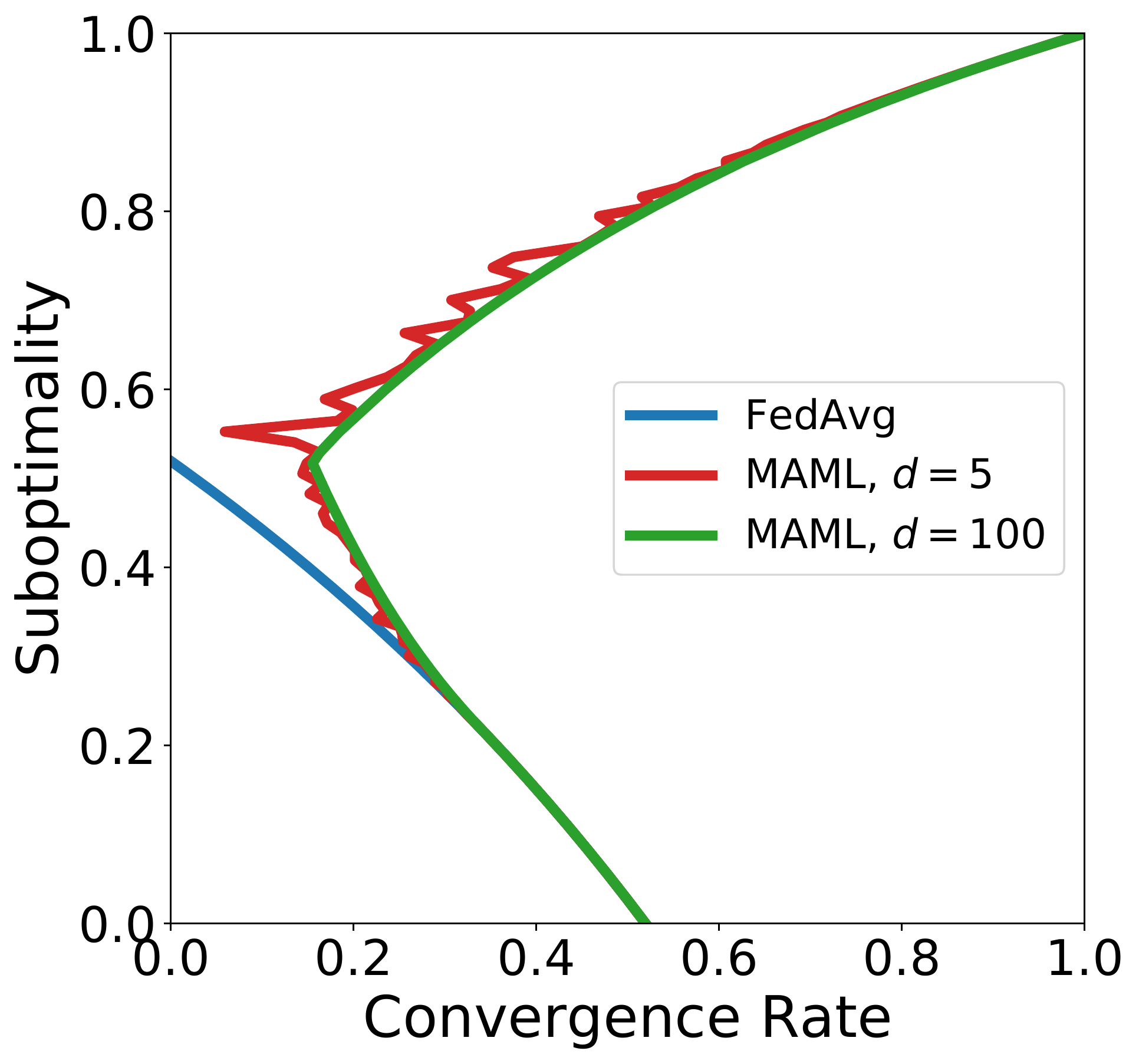}
    \caption{\ServerOpt: Gradient descent with heavy-ball momentum}
    \end{subfigure}
\caption{Simulated \maml Pareto frontiers for $\mu = 1, L = 10, \gamma = 0.001, \Theta = \Theta_K$, and where \ServerOpt is gradient descent with various types of momentum. We vary $K \in [1, 10^6]$. We randomly generate $A \in \R^{d\times d}$ with $\mu I \preceq A \preceq LI$ and compute the associated $(\rho, \Delta)$. We also compare to the Pareto frontier for $\Theta_{1:K}$ with the same \ServerOpt.}
\label{fig:sim_maml_momentum}
\end{figure}

\subsection{Proximal \maml-style Pareto Frontiers}

In Figure \ref{fig:maml_prox} we plot the analog of the Pareto frontiers in Figure \ref{fig:compare_fedprox}, but for \maml-style algorithms where $\Theta = \Theta_{K}$. We see a similar, though more subdued, version of the behavior in Figure \ref{fig:compare_fedprox}. That is, adding a proximal term simply alters how much of the Pareto frontier is traversed; it does not change the fundamental shape. Note that here we only used $\gamma$ satisfying $\gamma < (KL+\alpha)^{-1}$, as required by Theorem \ref{thm:conv_rates}. In particular, the only restriction on the shape of the curve seems to be coming from the fact that larger $\alpha$ reduces the set of $\gamma$ satisfying $\gamma < (KL+\alpha)^{-1}$.

\begin{figure}[ht]
\centering
    \begin{subfigure}{.35\linewidth}
    \centering
    \includegraphics[width=\linewidth]{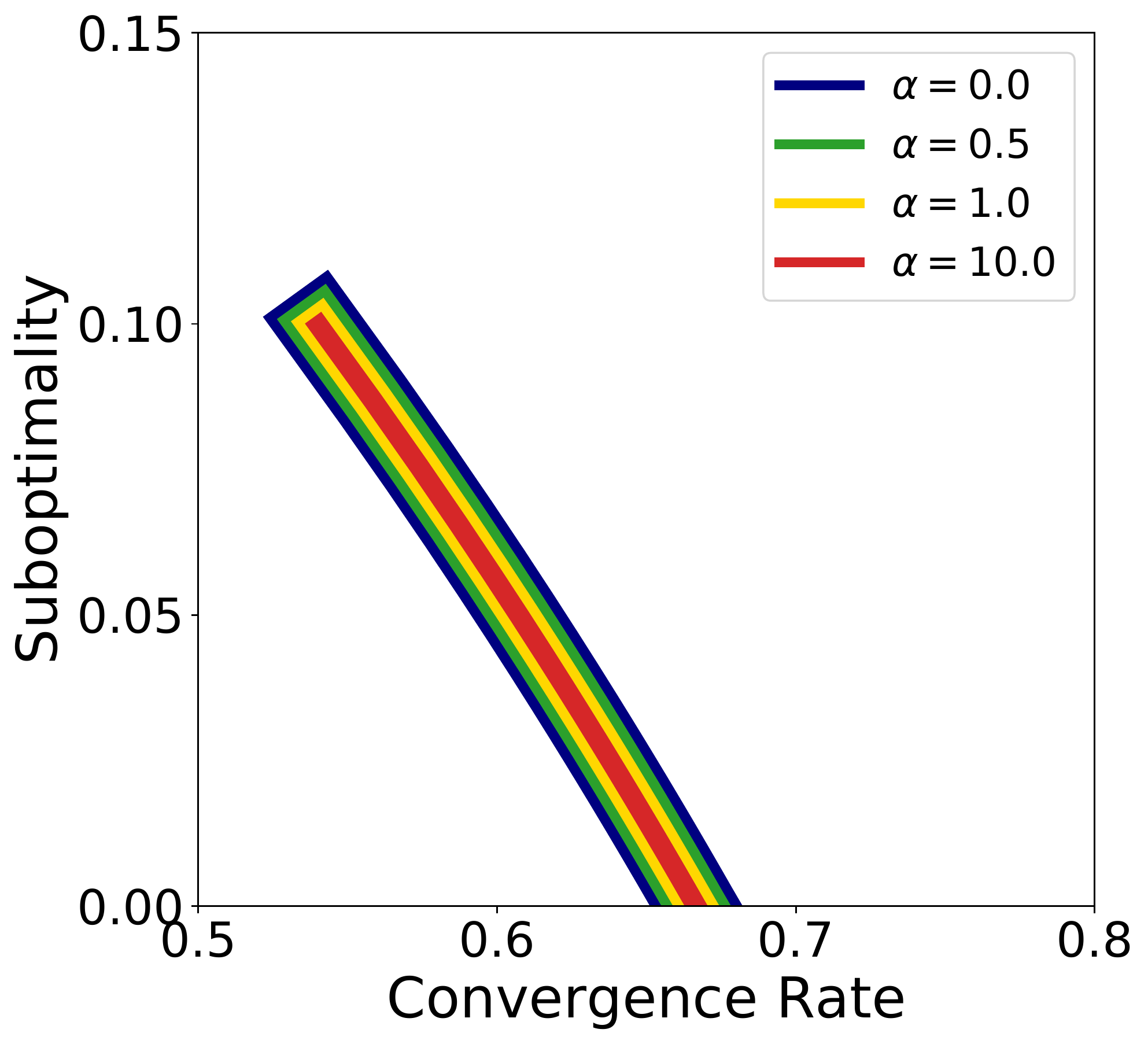}
    \end{subfigure}
\caption{Pareto frontiers for $\mu=1$, $L = 10, \gamma = 10^{-7}, \Theta = \Theta_{K}$ and varying $\alpha$. We generate the frontiers by varying $K \in [1, 10^6]$. We let \ServerOpt be gradient descent. For clarity, we used plots of varying width.}
\label{fig:maml_prox}
\end{figure}

To see the effects of $\alpha$ when $\gamma \geq (KL+\alpha)^{-1}$, we use the same simulated approach as in Figure \ref{fig:compare_simulated_fedavg_maml}. We do this for varying $\alpha$ in Figure \ref{fig:sim_maml_prox}. We see that while increasing $\alpha$ shrinks the space of the Pareto curve of \fedavg, it does not seem to change the \maml curves by a meaningful amount.

\begin{figure}[ht]
\centering
    \begin{subfigure}{.4\linewidth}
    \centering
    \includegraphics[width=\linewidth]{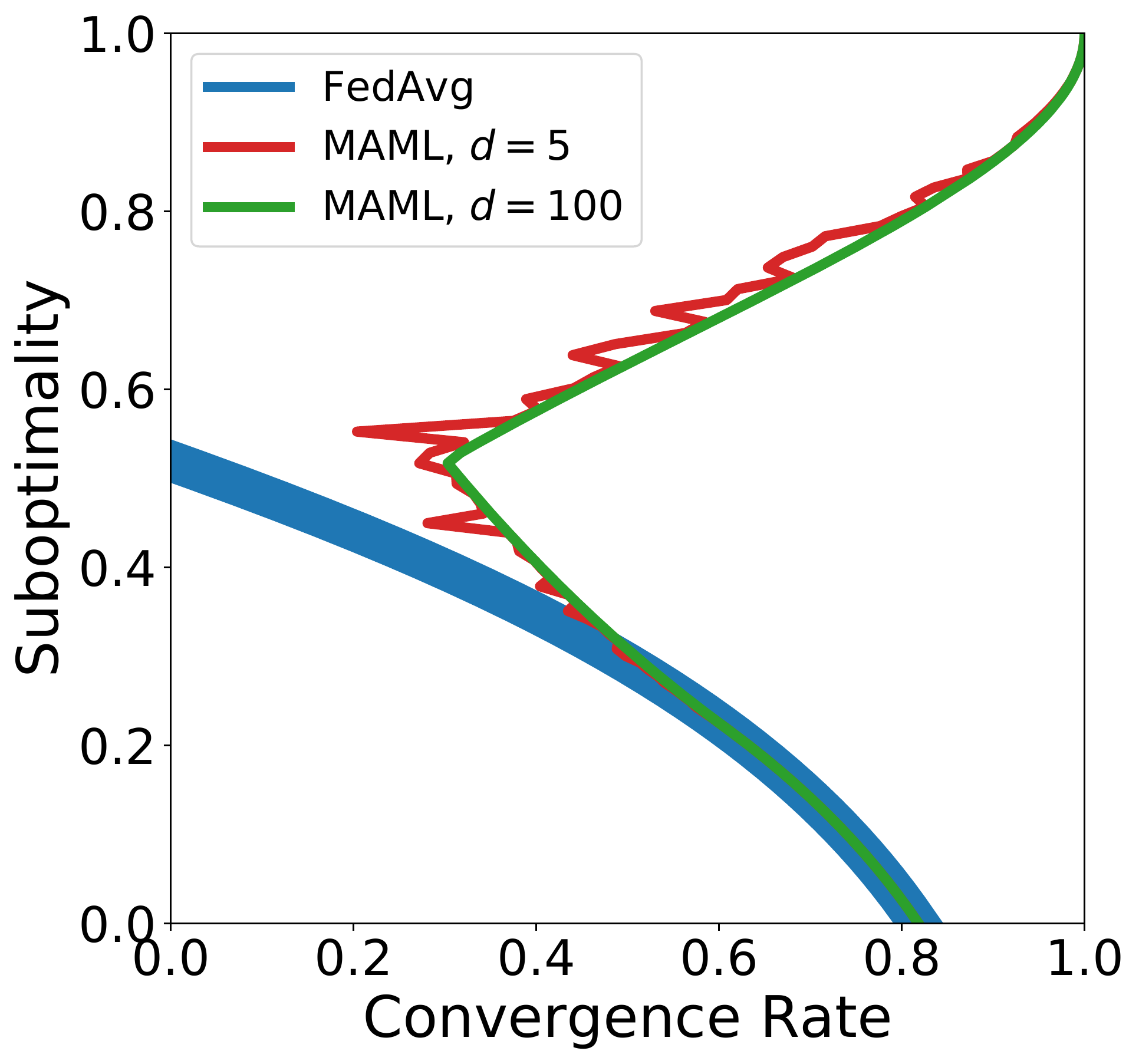}
    \caption{$\alpha = 0$}
    \end{subfigure}
    \begin{subfigure}{.4\linewidth}
    \centering
    \includegraphics[width=\linewidth]{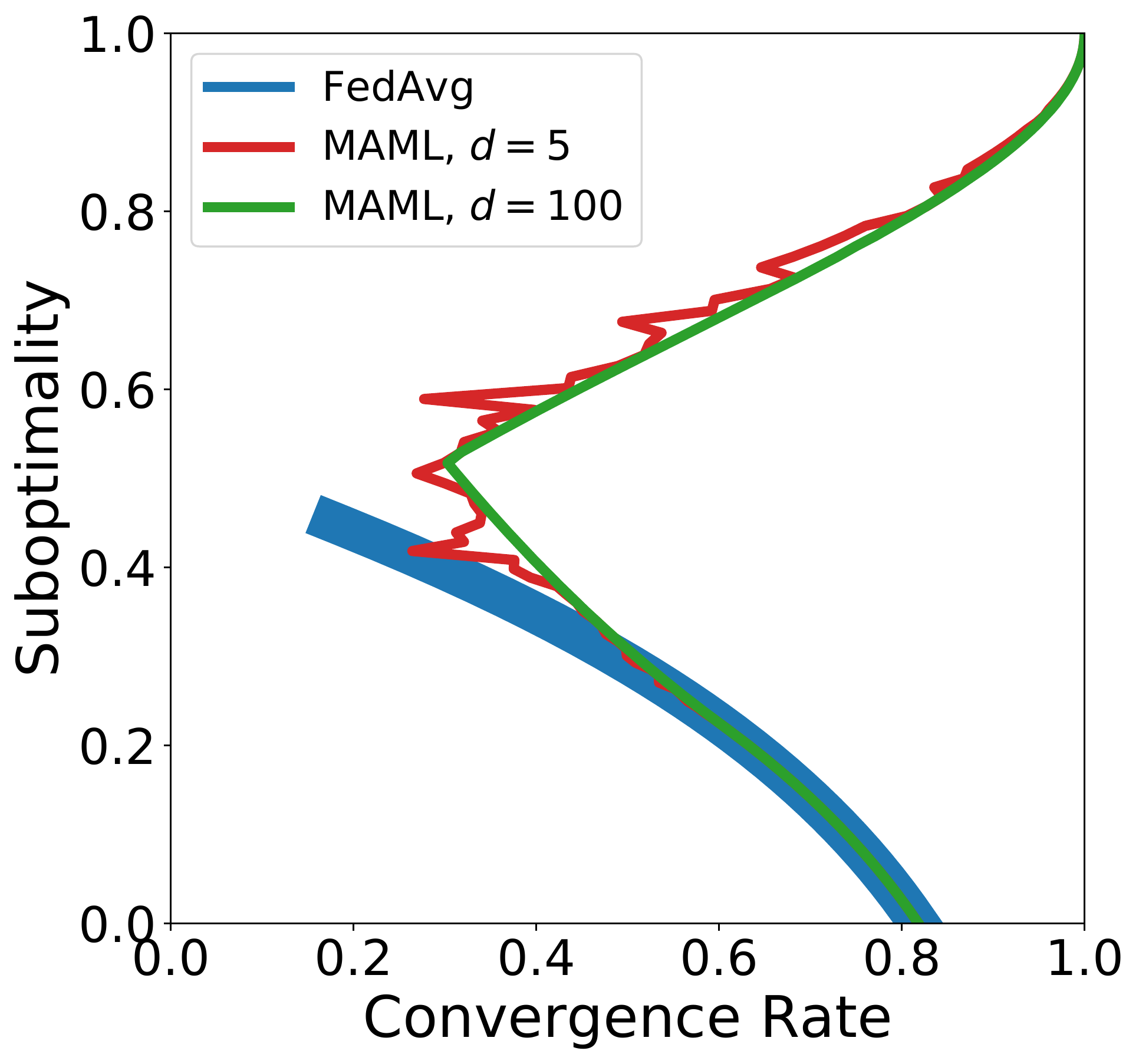}
    \caption{$\alpha = 0.5$}
    \end{subfigure}
    \begin{subfigure}{.4\linewidth}
    \centering
    \includegraphics[width=\linewidth]{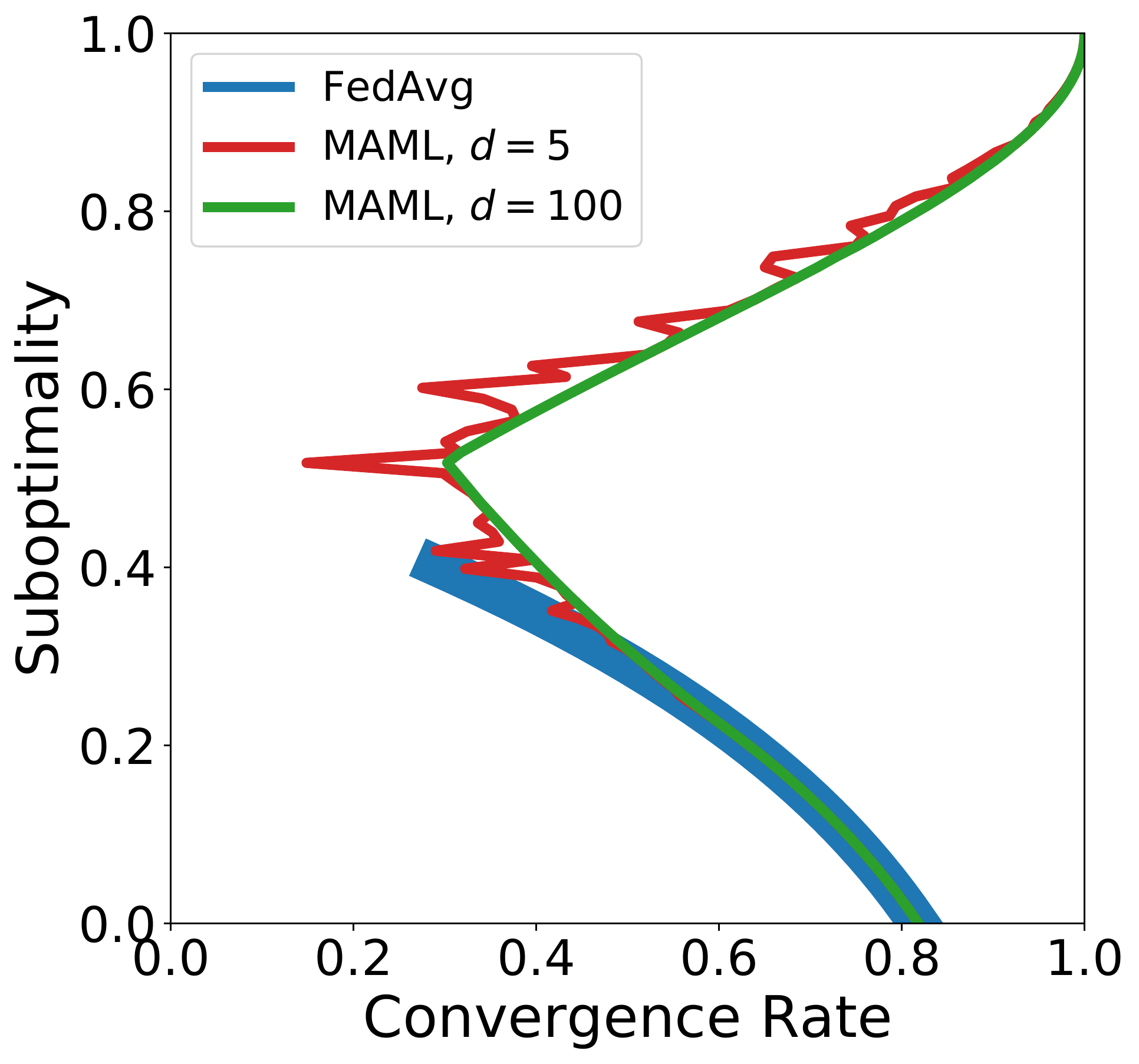}
    \caption{$\alpha = 1.0$}
    \end{subfigure}
    \begin{subfigure}{.4\linewidth}
    \centering
    \includegraphics[width=\linewidth]{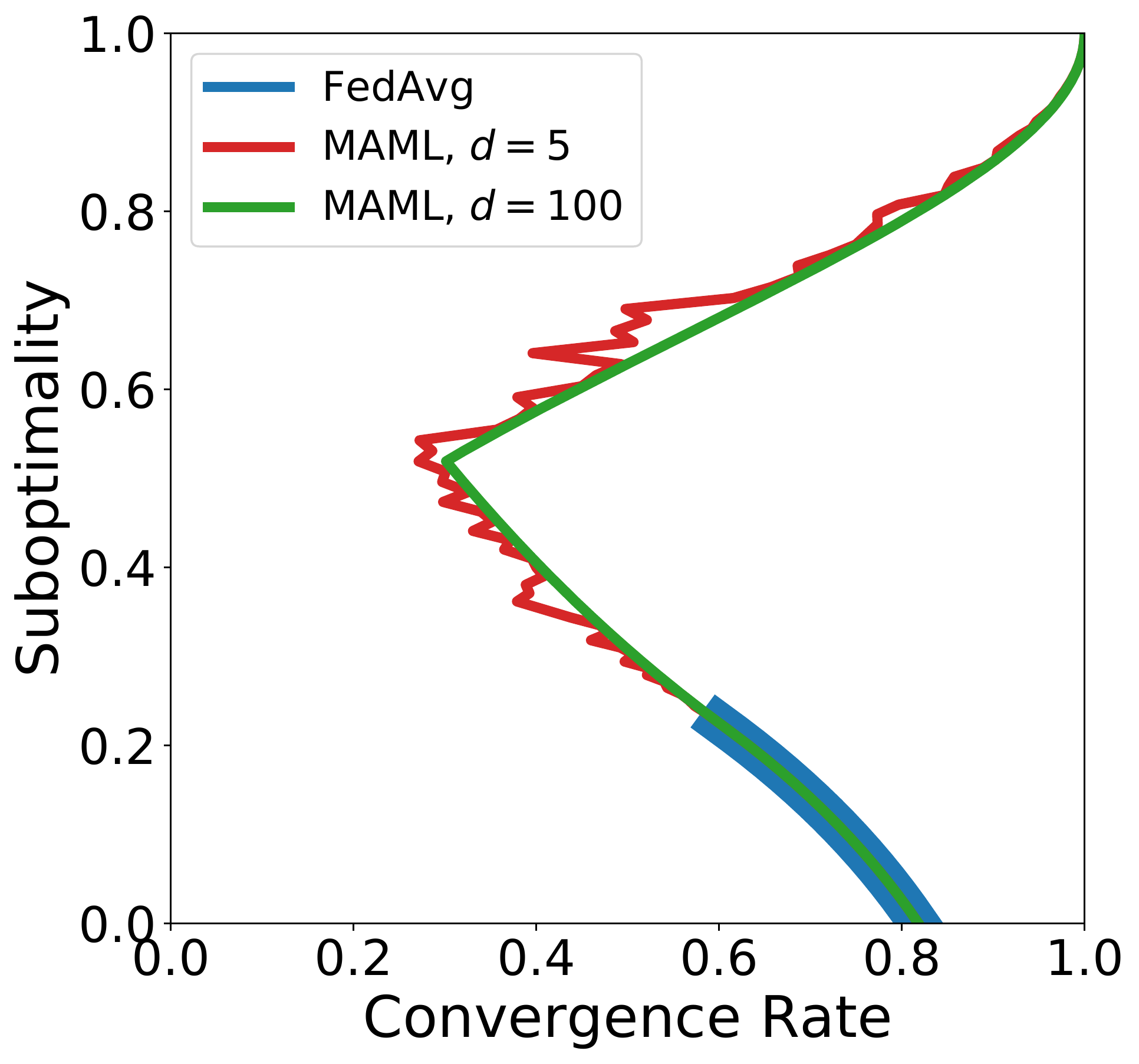}
    \caption{$\alpha = 5.0$}
    \end{subfigure}
\caption{Simulated Pareto frontiers for $\mu = 1, L = 10, \gamma = 0.001, \Theta = \Theta_K$ (\maml), where \ServerOpt is gradient descent. We use $\alpha \in \{0, 0.5, 1, 10\}$, $K \in [1, 10^6]$. We randomly generate $A \in \R^{d\times d}$ with $\mu I \preceq A \preceq LI$ and compute the associated $(\rho, \Delta)$. We also compare to the Pareto frontier for $\Theta_{1:K}$ (\fedavg) and the same $\alpha$.}
\label{fig:sim_maml_prox}
\end{figure}

\clearpage

\section{Experimental Setup}\label{appendix:experiment_setup}

\subsection{Datasets and Models}

We use three datasets: the federated extended MNIST dataset (FEMNIST)~\citep{caldas2018leaf}, CIFAR-100~\citep{krizhevsky2009learning}, and Shakespeare~\citep{caldas2018leaf}. The first two are image datasets, and the third is a language dataset. All datasets are publicly available. We specifically use the versions available in TensorFlow Federated~\citep{ingerman2019tff}, which gives a federated structure to all three. We keep the client partitioning when training, and create a test dataset by taking a union over all test client datasets. Statistics on the number of clients and examples in each dataset are given in Table \ref{table:datasets}.
    
\begin{table}[ht]
    \caption{Dataset statistics.}
    \label{table:datasets}
    \begin{center}
    \begin{sc}
    \begin{tabular}[t]{@{}lrrrr@{}}    
        \toprule
        Dataset & Train Clients & Train Examples & Test Clients & Test Examples \\
        \midrule
        FEMNIST & 3,400 & 671,585 & 3,400 & 77,483\\
        CIFAR-100 & 500 & 50,000 & 100 & 10,000 \\
        Shakespeare & 715 & 16,068 & 715 & 2,356\\
        \bottomrule
    \end{tabular}
    \end{sc}
    \end{center}
\end{table}

\paragraph{FEMNIST} The FEMNIST dataset consists of images hand-written alphanumeric characters. There are 62 total alphanumeric characters represented in the dataset. The images are partitioned among clients according to their author. The dataset has natural heterogeneity stemming from the writing style of each person. We train a convolutional network on the dataset (the same one used by \citet{reddi2020adaptive}). The network has two convolutional layers. Each convolutional layer uses $3\times 3$ kernels, max pooling, and then dropout with probability $p = 0.25$. The model has a final dense softmax output layer.

\paragraph{CIFAR-100} The CIFAR-100 dataset is a computer vision dataset consisting of $32 \times 32 \times 3$ images with 100 possible labels. While this dataset does not have a natural partition among clients, a federated version was created by \citet{reddi2020adaptive} using hierarchical latent Dirichlet allocation to enforce moderate amounts of heterogeneity among clients. We train a ResNet-18 on this dataset, where we replace all batch normalization layers with group normalization layers~\citep{wu2018group}. The use of group norm over batch norm in federated learning was first advocated by \citet{hsieh2019non}.

We perform small amounts of data augmentation and preprocessing, as is standard with CIFAR-100. We first perform a random crop to shape $(24, 24, 3)$, followed by a random horizontal flip. We then normalize the pixel values according to their mean and standard deviation. Thus, given an image $x$, we compute $(x - \mu)/\sigma$ where $\mu$ is the average of the pixel values in $x$, and $\sigma$ is the standard deviation.

\paragraph{Shakespeare} The Shakespeare dataset is derived from the benchmark designed by \citet{caldas2018leaf}. The dataset corpus is the collected works of William Shakespeare, and the clients correspond to roles in Shakespeare's plays with at least two lines of dialogue. To eliminate confusion, \emph{character} here will refer to alphanumeric and other such symbols, while we will use \emph{client} to denote the various roles in plays. We split each client's lines into sequences of 80 characters, padding if necessary. We use a vocabulary size of 90: 86 characters contained in Shakespeare's work, beginning and end of line tokens, padding tokens, and out-of-vocabulary tokens. We perform next-character prediction on the clients' dialogue using an RNN. The RNN takes as input a sequence of 80 characters, embeds it into a learned 8-dimensional space, and passes the embedding through 2 LSTM layers, each with 256 units. Finally, we use a softmax output layer with 80 units, where we try to predict a sequence of 80 characters formed by shifting the input sequence over by one. Therefore, our output dimension is $80\times 90$. We compute loss using cross-entropy loss.

\subsection{Implementation and Hyperparameters}

We implement \localupdate in TensorFlow Federated~\citep{ingerman2019tff}. We use \localupdate with $\Theta = \Theta_{1:K}$ and client learning rate $\gamma$. In all experiments, \ServerOpt is gradient descent with server learning rate $\eta$, with either no momentum, Nesterov momentum, or heavy-ball momentum. We sample $M = 10$ clients per round. We sample without replacement within a given round, and with replacement across rounds. In order to derive fair comparisons between different hyperparameter settings, we use a random seed to fix which clients are sampled at each round. We use a batch size of $B = 20$ for FEMNIST and CIFAR-100, and $B = 4$ for Shakespeare.

\subsection{Details of Figure \ref{fig:emnist_results}}For posterity's sake, we re-plot Figure \ref{fig:emnist_results} in Figure \ref{fig:emnist_results_2}. To generate these plots, we perform two distinct experiments. In the first experiment (Figures \ref{fig:emnist_results} and \ref{fig:emnist_results_2}, left), we fix $\alpha = 0$ and vary \ServerOpt. Specifically, we let \ServerOpt be gradient descent with no momentum (gradient), gradient descent with Nesterov momentum (nesterov), and gradient descent with heavy-ball momentum (momentum). When \ServerOpt uses Nesterov or heavy-ball momentum, we use a momentum parameter of $\beta = 0.9$. In the second experiment (Figures \ref{fig:emnist_results} and \ref{fig:emnist_results_2}, right), we fix \ServerOpt to be gradient descent with no momentum, and vary the proximal strength $\alpha$. In both cases, we fix $\Theta = \Theta_{1:50}$, and tune $\gamma, \eta$ over the range
\[
\gamma, \eta \in \{10^{-3}, 10^{-2.5}, \dots, 10^{0.5}, 10\}.
\]
We select the values of $\gamma, \eta$ attaining the best average test accuracy over the last 100 rounds.

\begin{figure}[ht]
\centering
    \begin{subfigure}{.4\linewidth}
    \centering
    \includegraphics[width=\linewidth]{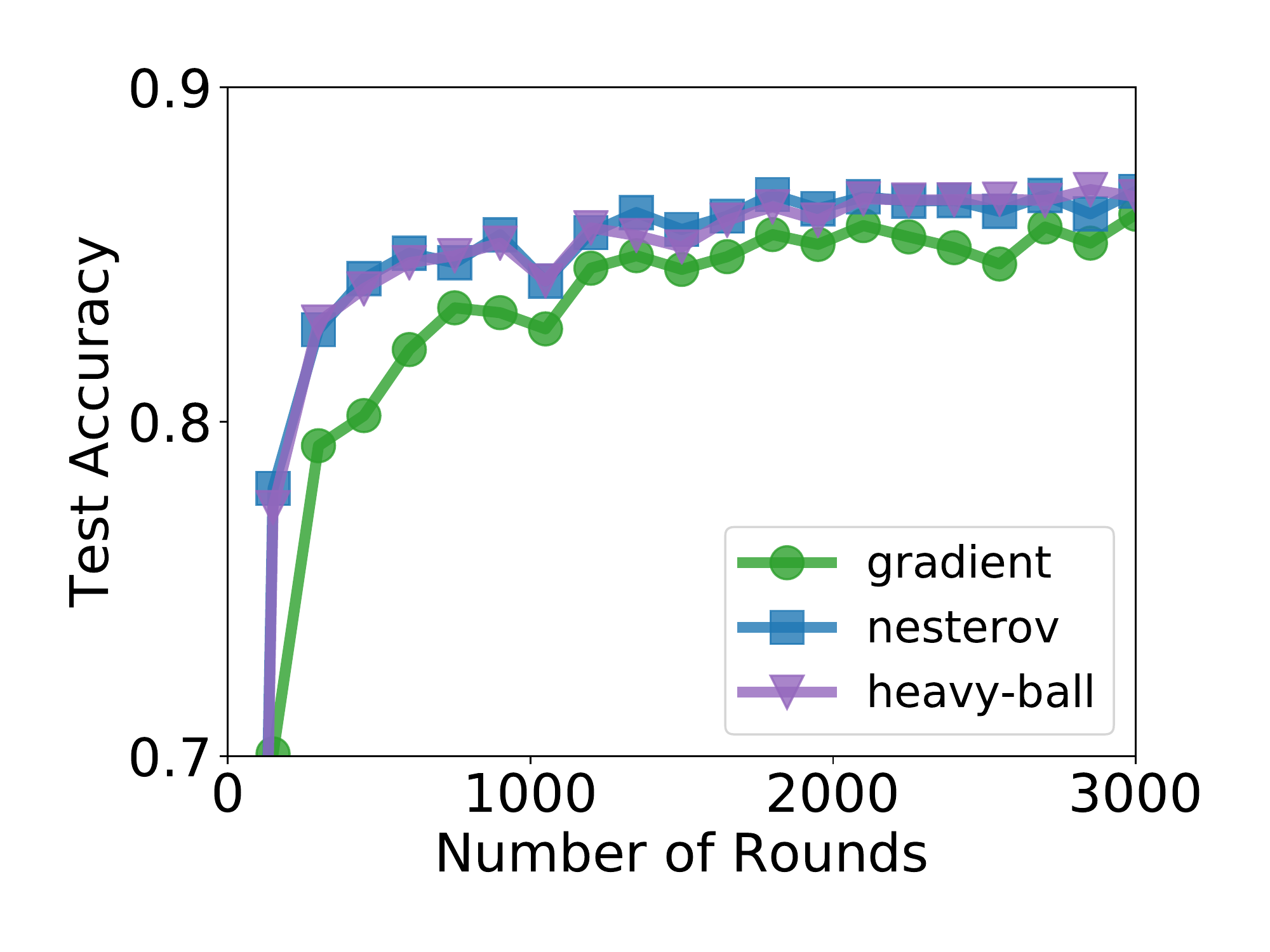}
    \end{subfigure}
    \begin{subfigure}{.4\linewidth}
    \centering
    \includegraphics[width=\linewidth]{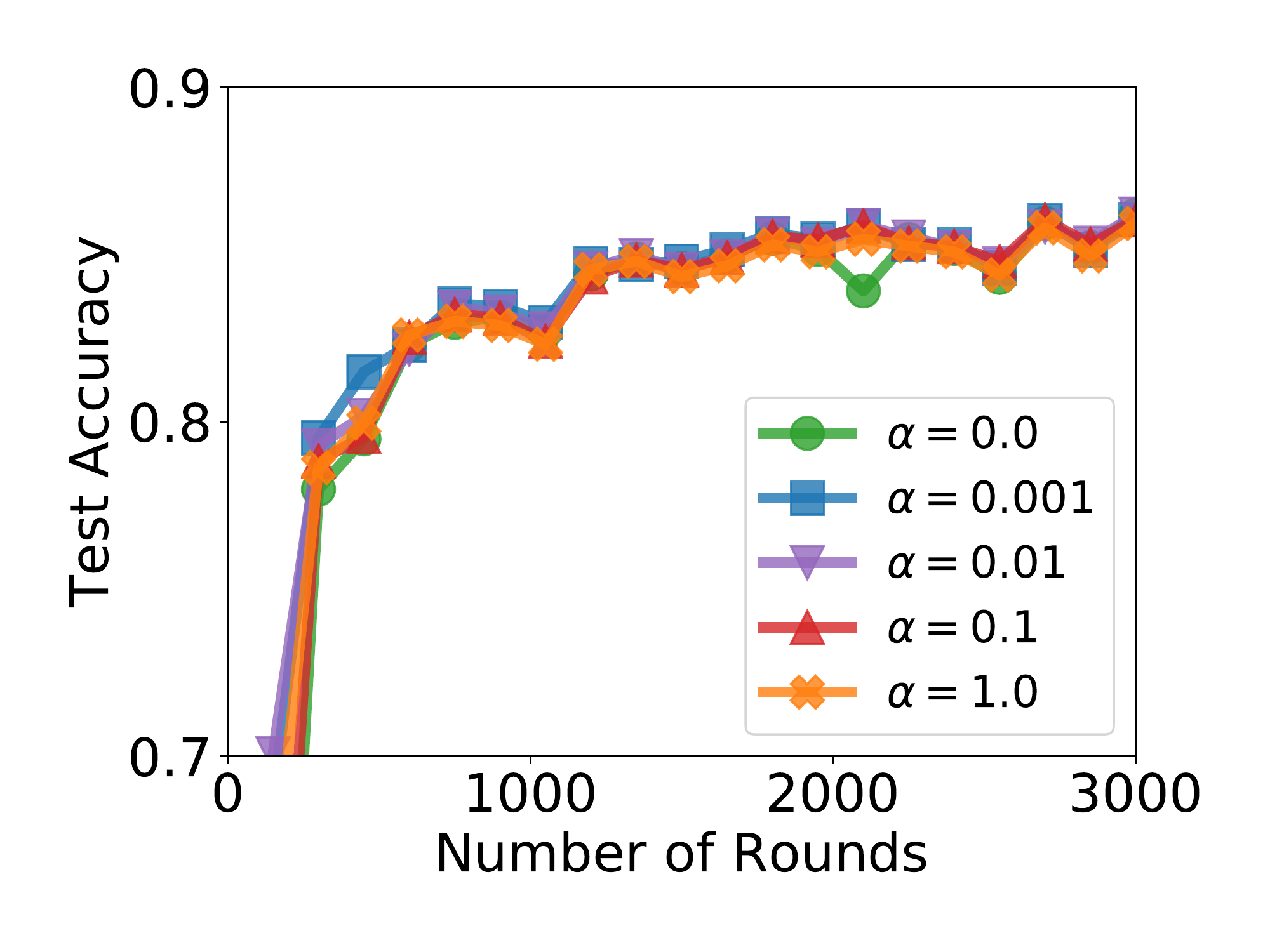}
    \end{subfigure}
\caption{Test accuracy of \localupdate on FEMNIST with tuned learning rates. (Left) Varying types of server momentum, $\alpha = 0$. (Right) No momentum and varying $\alpha$.}
\label{fig:emnist_results_2}
\end{figure}

\section{Additional Experiments}\label{appendix:additional_experiments}

We wish to showcase the convergence-accuracy trade-off discussed in Section \ref{sec:convergence_and_accuracy} in non-convex settings. We train \localupdate with $\alpha = 0$, $\Theta = \Theta_{1:10}$, and let \ServerOpt be gradient descent with learning rate $\eta$. First, we fix $\eta = 0.01$ and vary $\gamma$ over
\[
\gamma \in \{0, 10^{-3}, 10^{-2}, 10^{-1}, 1, 10\}.
\]
We plot the training loss over time on all three datasets in Figure \ref{fig:constant_server_lr}, omitting results that diverge due to $\gamma$ being too large. 

\begin{figure}[ht]
\centering
    \begin{subfigure}{.31\linewidth}
    \centering
    \includegraphics[width=\linewidth]{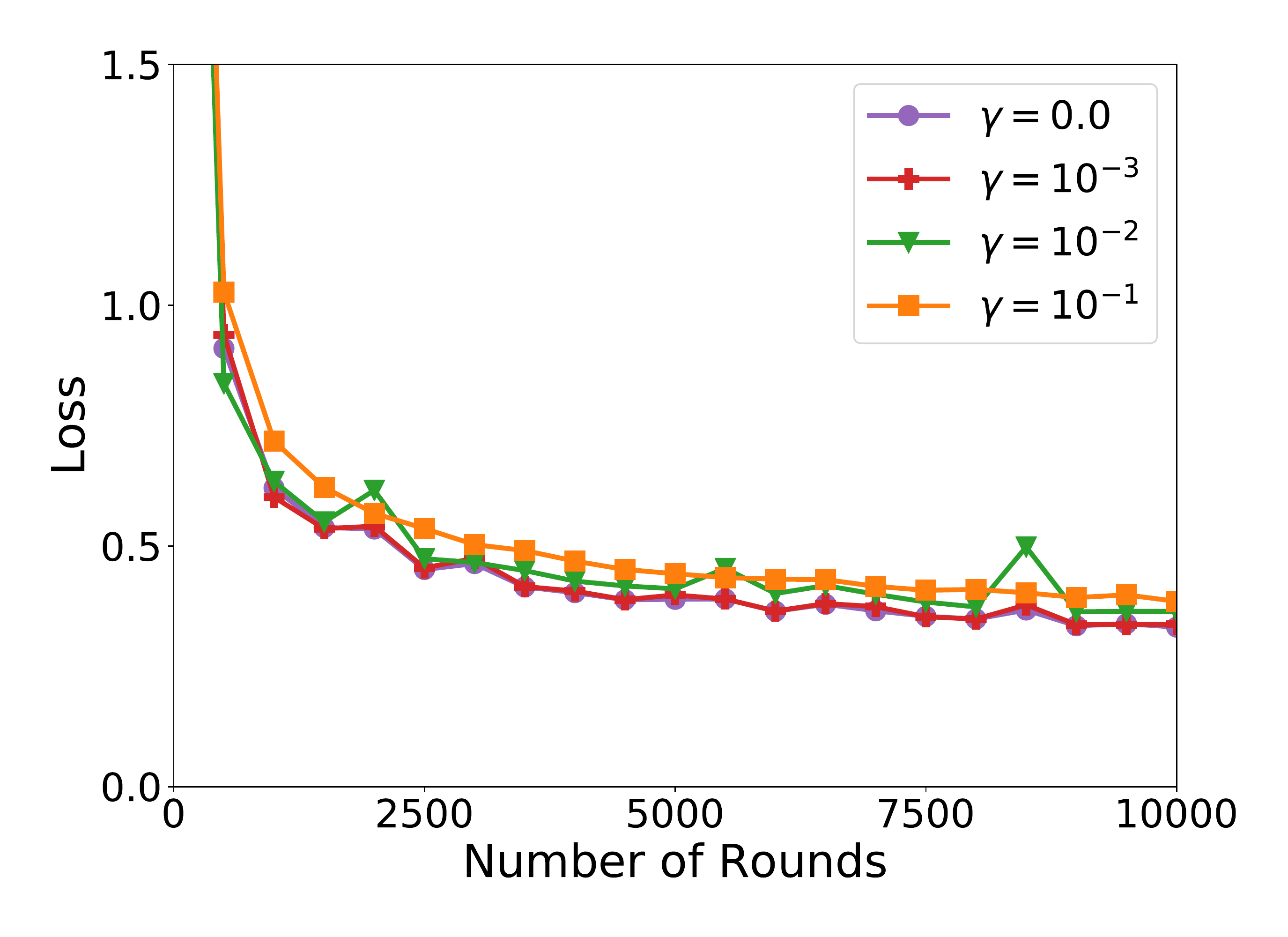}
    \caption{FEMNIST}
    \end{subfigure}
    \begin{subfigure}{.31\linewidth}
    \centering
    \includegraphics[width=\linewidth]{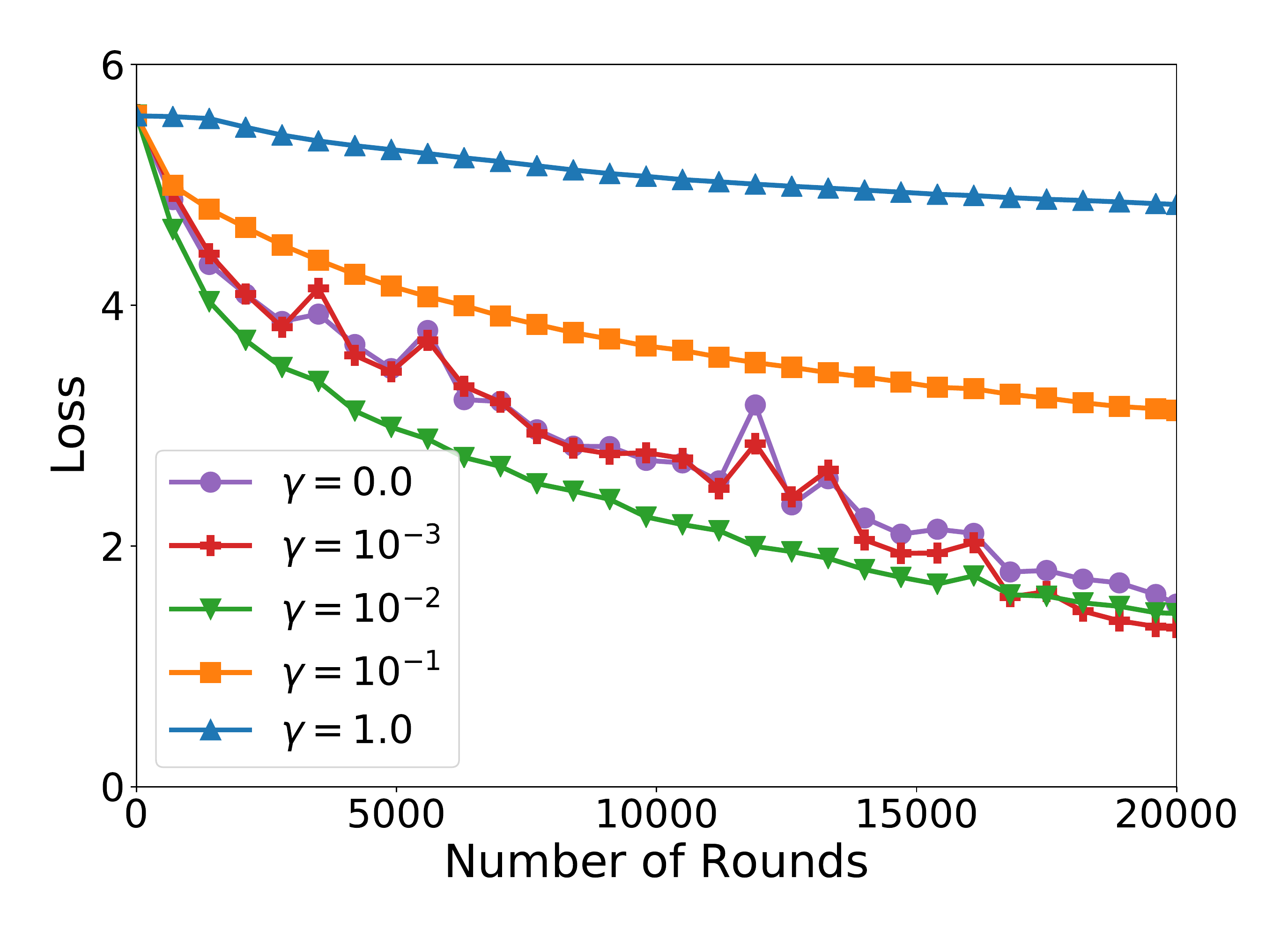}
    \caption{CIFAR-100}
    \end{subfigure}
    \begin{subfigure}{.31\linewidth}
    \centering
    \includegraphics[width=\linewidth]{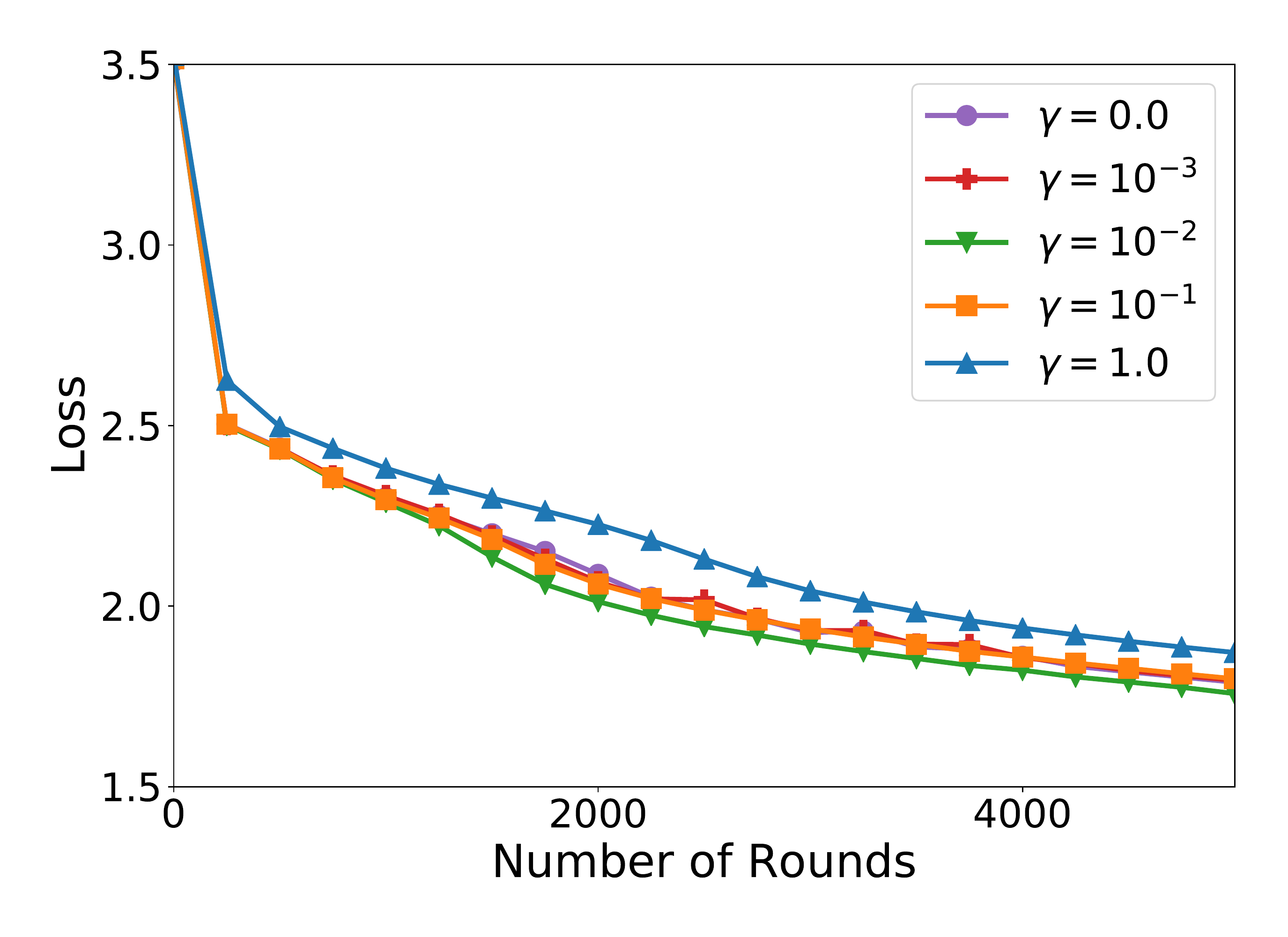}
    \caption{Shakespeare}
    \end{subfigure}
\caption{Training loss of \localupdate with $\alpha = 0$, $\Theta = \Theta_{1:10}$, varying client learning rate $\gamma$, and where \ServerOpt is gradient descent with learning rate $\eta = 0.01$ and no momentum.}
\label{fig:constant_server_lr}
\end{figure}

We see that on all three tasks, especially CIFAR-100, the choice of client learning rate can impact not just the speed of convergence, but what point the algorithm converges to. In general, we see very similar behavior to that described in Sections \ref{sec:convergence_and_accuracy} and \ref{sec:compare}, despite the non-convex loss functions involved in all three tasks. For both FEMNIST and CIFAR-100, smaller client learning rates eventually reach lower training losses than higher learning rates. This is particularly evident in the results for CIFAR-100. While $\gamma = 10^{-2}$ initially performs better than all other methods, it is eventually surpassed by $\gamma = 10^{-3}$, and $\gamma = 0$ ends up obtaining a comparable accuracy. This reflects the idea presented in Section \ref{sec:compare} that hyperparameters should be chosen according to the desired convergence-accuracy trade-off. In communication-limited settings, we should use larger $\gamma$ (or $K$), while in cases where we can run many communication rounds, we should use smaller $\gamma$ (or $K$).

\begin{figure}[ht]
\centering
    \begin{subfigure}{.31\linewidth}
    \centering
    \includegraphics[width=\linewidth]{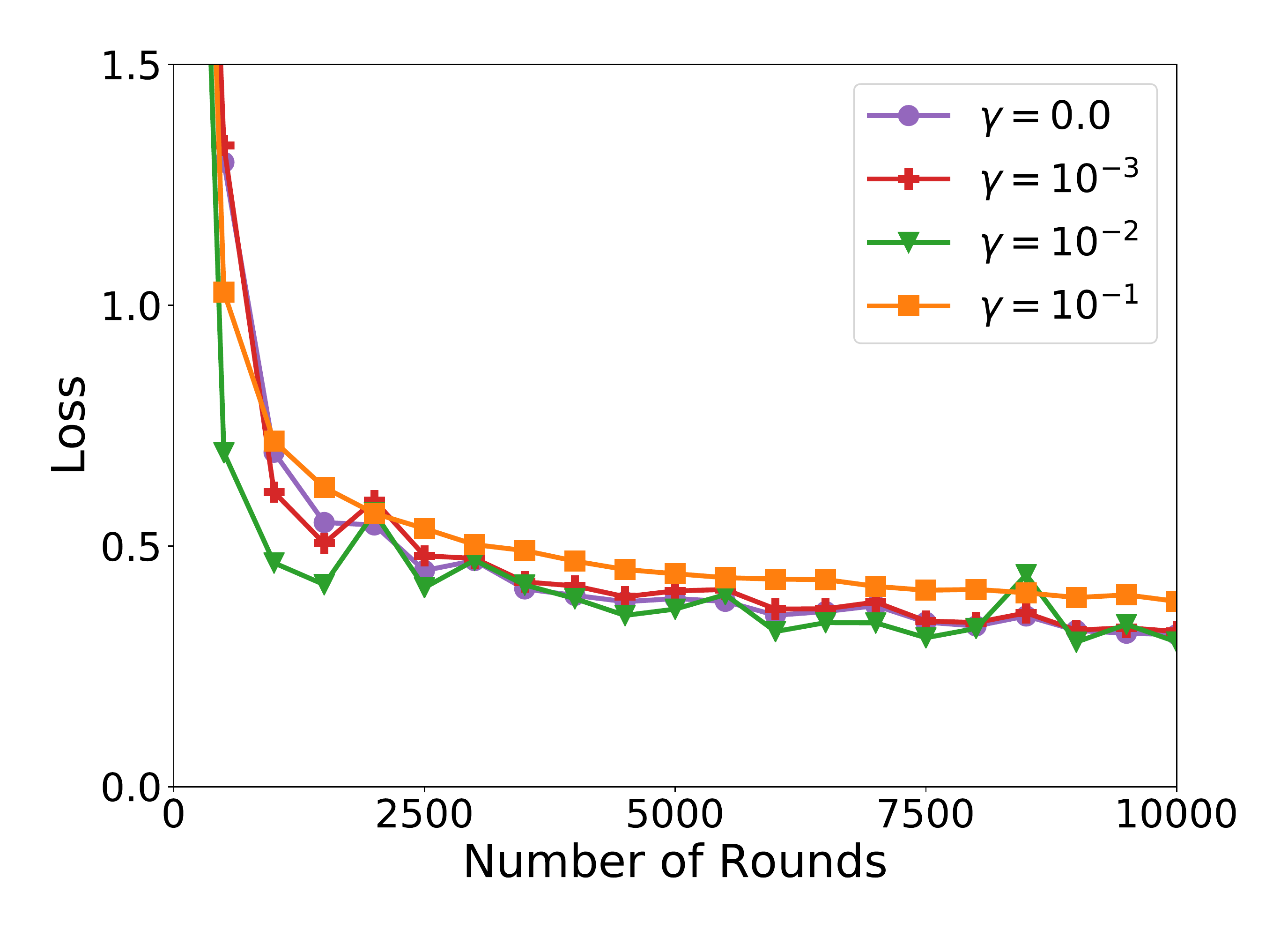}
    \caption{FEMNIST}
    \end{subfigure}
    \begin{subfigure}{.31\linewidth}
    \centering
    \includegraphics[width=\linewidth]{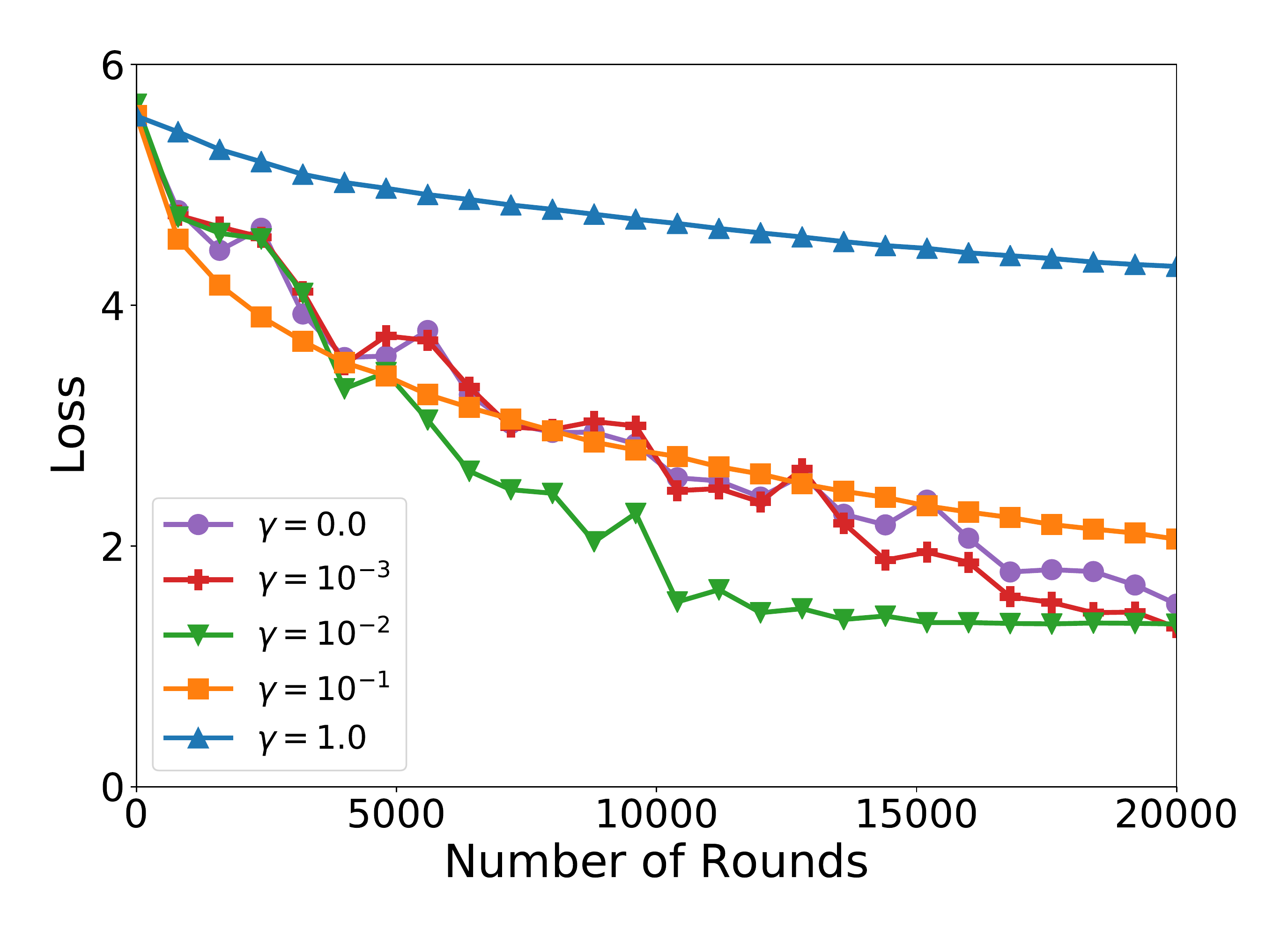}
    \caption{CIFAR-100}
    \end{subfigure}
    \begin{subfigure}{.31\linewidth}
    \centering
    \includegraphics[width=\linewidth]{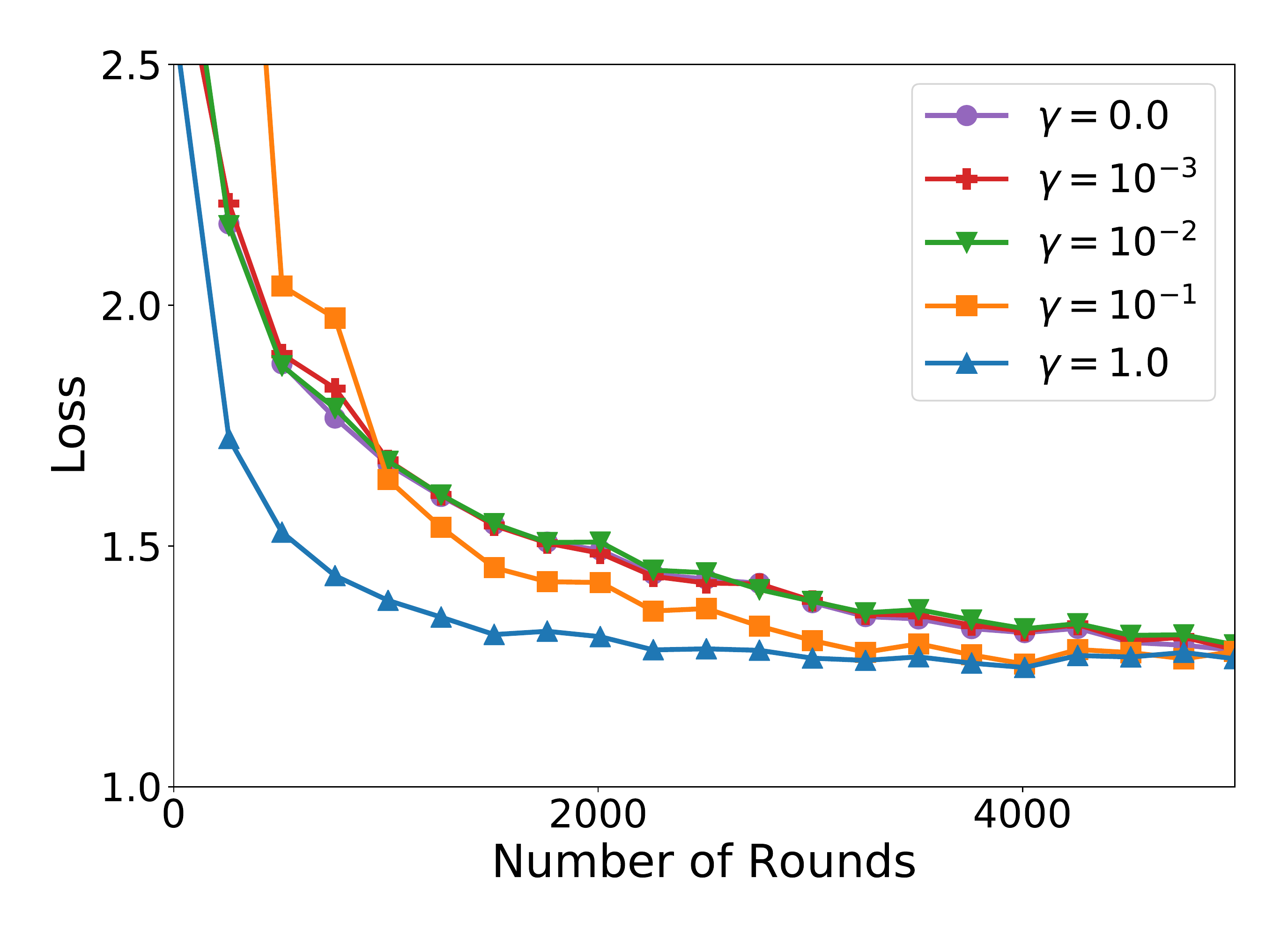}
    \caption{Shakespeare}
    \end{subfigure}
\caption{Training loss of \localupdate with $\alpha = 0, \Theta = \Theta_{1:10}$, varying client learning rate $\gamma$, and where \ServerOpt is gradient descent with tuned server learning rate $\eta$.}
\label{fig:tuned_server_lr}
\end{figure}

In short, we see clear evidence that the choice of client learning rate $\gamma$ leads to a trade-off between convergence and accuracy. However, as shown in Lemmas \ref{lem:cond_fedavg} and \ref{lem:cond_maml}, the condition number of the surrogate loss changes depending on parameters such as $\gamma$. To derive asymptotically optimal rates for strongly convex functions (such as the ones in Table \ref{table:conv_local_update}), one must generally set the learning rate $\eta$ according to the condition number. Thus, we repeat the experiments in Figure \ref{fig:constant_server_lr}, but where we tune the server learning rate $\eta$ instead of fixing it. This helps account for how the optimization dynamics can change as a function of the client learning rate $\gamma$. We vary the server learning rate $\eta$ over
\[
\eta \in \{ 10^{-3}, 10^{-2.5}, \dots, 10\}
\]
and select $\eta$ that leads to the smallest average training loss over the last 100 rounds. The result is given in Figure \ref{fig:tuned_server_lr}. Again, we see similar behavior, but see that when the server learning rate is tuned, larger client learning rates may do much better initially. This reflects the fact that in Table \ref{table:conv_local_update}, the best convergence rates can only be obtained by setting parameters of \ServerOpt correctly.

This points to another benefit of the Pareto frontiers proposed in Section \ref{sec:compare}. Many comparisons of different algorithms, especially empirical ones, can miss good hyperparameter settings. This is heightened by the fact that many FL and ML methods have hyperparameters for both client and server optimizers, comprehensive tuning extremely difficult. This may lead to unfair comparisons between methods. By contrast, the Pareto frontiers showcase convergence-accuracy trade-offs when the hyperparameters of \ServerOpt are selected in an ``optimal'' way, helping derive fair comparisons between methods.

%% file: main.bbl
\begin{thebibliography}{39}
\providecommand{\natexlab}[1]{#1}
\providecommand{\url}[1]{\texttt{#1}}
\expandafter\ifx\csname urlstyle\endcsname\relax
  \providecommand{\doi}[1]{doi: #1}\else
  \providecommand{\doi}{doi: \begingroup \urlstyle{rm}\Url}\fi

\bibitem[Alon et~al.(2002)Alon, Krivelevich, and Vu]{alon2002concentration}
Noga Alon, Michael Krivelevich, and Van~H Vu.
\newblock On the concentration of eigenvalues of random symmetric matrices.
\newblock \emph{Israel Journal of Mathematics}, 131\penalty0 (1):\penalty0
  259--267, 2002.

\bibitem[Balcan et~al.(2019)Balcan, Khodak, and Talwalkar]{balcan2019provable}
Maria-Florina Balcan, Mikhail Khodak, and Ameet Talwalkar.
\newblock Provable guarantees for gradient-based meta-learning.
\newblock In \emph{International Conference on Machine Learning}, pages
  424--433. PMLR, 2019.

\bibitem[Basu et~al.(2019)Basu, Data, Karakus, and Diggavi]{basu2019qsparse}
Debraj Basu, Deepesh Data, Can Karakus, and Suhas Diggavi.
\newblock Qsparse-local-{SGD}: Distributed {SGD} with quantization,
  sparsification and local computations.
\newblock In \emph{Advances in Neural Information Processing Systems}, pages
  14668--14679, 2019.

\bibitem[Bhatia and Davis(2000)]{bhatia2000better}
Rajendra Bhatia and Chandler Davis.
\newblock A better bound on the variance.
\newblock \emph{The American Mathematical Monthly}, 107\penalty0 (4):\penalty0
  353--357, 2000.

\bibitem[Caldas et~al.(2018)Caldas, Wu, Li, Kone{\v{c}}n{\'y}, McMahan, Smith,
  and Talwalkar]{caldas2018leaf}
Sebastian Caldas, Peter Wu, Tian Li, Jakub Kone{\v{c}}n{\'y}, H~Brendan
  McMahan, Virginia Smith, and Ameet Talwalkar.
\newblock {LEAF}: A benchmark for federated settings.
\newblock \emph{arXiv preprint arXiv:1812.01097}, 2018.

\bibitem[Fallah et~al.(2020{\natexlab{a}})Fallah, Mokhtari, and
  Ozdaglar]{fallah2020convergence}
Alireza Fallah, Aryan Mokhtari, and Asuman Ozdaglar.
\newblock On the convergence theory of gradient-based model-agnostic
  meta-learning algorithms.
\newblock In Silvia Chiappa and Roberto Calandra, editors, \emph{Proceedings of
  the Twenty Third International Conference on Artificial Intelligence and
  Statistics}, volume 108 of \emph{Proceedings of Machine Learning Research},
  pages 1082--1092. PMLR, 26--28 Aug 2020{\natexlab{a}}.
\newblock URL \url{http://proceedings.mlr.press/v108/fallah20a.html}.

\bibitem[Fallah et~al.(2020{\natexlab{b}})Fallah, Mokhtari, and
  Ozdaglar]{fallah2020personalized}
Alireza Fallah, Aryan Mokhtari, and Asuman Ozdaglar.
\newblock Personalized federated learning: A meta-learning approach.
\newblock \emph{arXiv preprint arXiv:2002.07948}, 2020{\natexlab{b}}.

\bibitem[Finn et~al.(2017)Finn, Abbeel, and Levine]{finn2017model}
Chelsea Finn, Pieter Abbeel, and Sergey Levine.
\newblock Model-agnostic meta-learning for fast adaptation of deep networks.
\newblock In \emph{Proceedings of the 34th International Conference on Machine
  Learning-Volume 70}, pages 1126--1135. JMLR, 2017.

\bibitem[Hard et~al.(2018)Hard, Rao, Mathews, Ramaswamy, Beaufays, Augenstein,
  Eichner, Kiddon, and Ramage]{hard2018federated}
Andrew Hard, Kanishka Rao, Rajiv Mathews, Swaroop Ramaswamy, Fran{\c{c}}oise
  Beaufays, Sean Augenstein, Hubert Eichner, Chlo{\'e} Kiddon, and Daniel
  Ramage.
\newblock Federated learning for mobile keyboard prediction.
\newblock \emph{arXiv preprint arXiv:1811.03604}, 2018.

\bibitem[Hard et~al.(2020)Hard, Partridge, Nguyen, Subrahmanya, Shah, Zhu,
  Moreno, and Mathews]{hard2020training}
Andrew Hard, Kurt Partridge, Cameron Nguyen, Niranjan Subrahmanya, Aishanee
  Shah, Pai Zhu, Ignacio~Lopez Moreno, and Rajiv Mathews.
\newblock Training keyword spotting models on non-{IID} data with federated
  learning.
\newblock \emph{arXiv preprint arXiv:2005.10406}, 2020.

\bibitem[Hsieh et~al.(2019)Hsieh, Phanishayee, Mutlu, and
  Gibbons]{hsieh2019non}
Kevin Hsieh, Amar Phanishayee, Onur Mutlu, and Phillip~B Gibbons.
\newblock The non-{IID} data quagmire of decentralized machine learning.
\newblock \emph{arXiv preprint arXiv:1910.00189}, 2019.

\bibitem[Hsu et~al.(2019)Hsu, Qi, and Brown]{hsu2019measuring}
Tzu-Ming~Harry Hsu, Hang Qi, and Matthew Brown.
\newblock Measuring the effects of non-identical data distribution for
  federated visual classification.
\newblock \emph{arXiv preprint arXiv:1909.06335}, 2019.

\bibitem[Ingerman and Ostrowski(2019)]{ingerman2019tff}
Alex Ingerman and Krzys Ostrowski.
\newblock Introducing {T}ensor{F}low {F}ederated, 2019.
\newblock URL
  \url{https://medium.com/tensorflow/introducing-tensorflow-federated-a4147aa20041}.

\bibitem[Jiang et~al.(2019)Jiang, Kone{\v{c}}n{\`y}, Rush, and
  Kannan]{jiang2019improving}
Yihan Jiang, Jakub Kone{\v{c}}n{\`y}, Keith Rush, and Sreeram Kannan.
\newblock Improving federated learning personalization via model agnostic meta
  learning.
\newblock \emph{arXiv preprint arXiv:1909.12488}, 2019.

\bibitem[Kairouz et~al.(2019)Kairouz, McMahan, Avent, Bellet, Bennis, Bhagoji,
  Bonawitz, Charles, Cormode, Cummings, et~al.]{kairouz2019advances}
Peter Kairouz, H~Brendan McMahan, Brendan Avent, Aur{\'e}lien Bellet, Mehdi
  Bennis, Arjun~Nitin Bhagoji, Keith Bonawitz, Zachary Charles, Graham Cormode,
  Rachel Cummings, et~al.
\newblock Advances and open problems in federated learning.
\newblock \emph{arXiv preprint arXiv:1912.04977}, 2019.

\bibitem[Karimireddy et~al.(2019)Karimireddy, Kale, Mohri, Reddi, Stich, and
  Suresh]{karimireddy2019scaffold}
Sai~Praneeth Karimireddy, Satyen Kale, Mehryar Mohri, Sashank~J Reddi,
  Sebastian~U Stich, and Ananda~Theertha Suresh.
\newblock {SCAFFOLD}: Stochastic controlled averaging for on-device federated
  learning.
\newblock \emph{arXiv preprint arXiv:1910.06378}, 2019.

\bibitem[Khodak et~al.(2019)Khodak, Balcan, and Talwalkar]{khodak2019adaptive}
Mikhail Khodak, Maria-Florina~F Balcan, and Ameet~S Talwalkar.
\newblock Adaptive gradient-based meta-learning methods.
\newblock In \emph{Advances in Neural Information Processing Systems}, pages
  5915--5926, 2019.

\bibitem[Kingma and Ba(2014)]{kingma2014adam}
Diederik~P Kingma and Jimmy Ba.
\newblock Adam: A method for stochastic optimization.
\newblock \emph{arXiv preprint arXiv:1412.6980}, 2014.

\bibitem[Krizhevsky and Hinton(2009)]{krizhevsky2009learning}
Alex Krizhevsky and Geoffrey Hinton.
\newblock Learning multiple layers of features from tiny images.
\newblock Technical report, Citeseer, 2009.

\bibitem[Lessard et~al.(2016)Lessard, Recht, and Packard]{lessard2016analysis}
Laurent Lessard, Benjamin Recht, and Andrew Packard.
\newblock Analysis and design of optimization algorithms via integral quadratic
  constraints.
\newblock \emph{SIAM Journal on Optimization}, 26\penalty0 (1):\penalty0
  57--95, 2016.

\bibitem[Li et~al.(2019)Li, Sahu, Talwalkar, and Smith]{li2019federated}
Tian Li, Anit~Kumar Sahu, Ameet Talwalkar, and Virginia Smith.
\newblock Federated learning: Challenges, methods, and future directions.
\newblock \emph{arXiv preprint arXiv:1908.07873}, 2019.

\bibitem[Li et~al.(2020{\natexlab{a}})Li, Sahu, Zaheer, Sanjabi, Talwalkar, and
  Smith]{li2018federated}
Tian Li, Anit~Kumar Sahu, Manzil Zaheer, Maziar Sanjabi, Ameet Talwalkar, and
  Virginia Smith.
\newblock Federated optimization in heterogeneous networks.
\newblock In \emph{Proceedings of Machine Learning and Systems 2020}, pages
  429--450, 2020{\natexlab{a}}.

\bibitem[Li et~al.(2020{\natexlab{b}})Li, Sanjabi, Beirami, and
  Smith]{li2019fair}
Tian Li, Maziar Sanjabi, Ahmad Beirami, and Virginia Smith.
\newblock Fair resource allocation in federated learning.
\newblock In \emph{International Conference on Learning Representations},
  2020{\natexlab{b}}.
\newblock URL \url{https://openreview.net/forum?id=ByexElSYDr}.

\bibitem[Malinovsky et~al.(2020)Malinovsky, Kovalev, Gasanov, Condat, and
  Richtarik]{malinovsky2020local}
Grigory Malinovsky, Dmitry Kovalev, Elnur Gasanov, Laurent Condat, and Peter
  Richtarik.
\newblock From local {SGD} to local fixed point methods for federated learning.
\newblock \emph{arXiv preprint arXiv:2004.01442}, 2020.

\bibitem[McMahan et~al.(2017)McMahan, Moore, Ramage, Hampson, and
  y~Arcas]{mcmahan17fedavg}
Brendan McMahan, Eider Moore, Daniel Ramage, Seth Hampson, and
  Blaise~Ag{\"{u}}era y~Arcas.
\newblock Communication-efficient learning of deep networks from decentralized
  data.
\newblock In \emph{Proceedings of the 20th International Conference on
  Artificial Intelligence and Statistics, {AISTATS} 2017}, pages 1273--1282,
  2017.

\bibitem[Nichol et~al.(2018)Nichol, Achiam, and Schulman]{nichol2018first}
Alex Nichol, Joshua Achiam, and John Schulman.
\newblock On first-order meta-learning algorithms.
\newblock \emph{arXiv preprint arXiv:1803.02999}, 2018.

\bibitem[Pathak and Wainwright(2020)]{pathak2020fedsplit}
Reese Pathak and Martin~J Wainwright.
\newblock Fed{S}plit: An algorithmic framework for fast federated optimization.
\newblock \emph{arXiv preprint arXiv:2005.05238}, 2020.

\bibitem[Popoviciu(1935)]{popoviciu1935equations}
Tiberiu Popoviciu.
\newblock Sur les {\'e}quations alg{\'e}briques ayant toutes leurs racines
  r{\'e}elles.
\newblock \emph{Mathematica}, 9:\penalty0 129--145, 1935.

\bibitem[Reddi et~al.(2020)Reddi, Charles, Zaheer, Garrett, Rush,
  Kone{\v{c}}n{\`y}, Kumar, and McMahan]{reddi2020adaptive}
Sashank Reddi, Zachary Charles, Manzil Zaheer, Zachary Garrett, Keith Rush,
  Jakub Kone{\v{c}}n{\`y}, Sanjiv Kumar, and H~Brendan McMahan.
\newblock Adaptive federated optimization.
\newblock \emph{arXiv preprint arXiv:2003.00295}, 2020.

\bibitem[Stich(2019)]{stich2018local}
Sebastian~U. Stich.
\newblock Local {SGD} converges fast and communicates little.
\newblock In \emph{International Conference on Learning Representations}, 2019.
\newblock URL \url{https://openreview.net/forum?id=S1g2JnRcFX}.

\bibitem[Wang et~al.(2020)Wang, Liu, Liang, Joshi, and Poor]{wang2020tackling}
Jianyu Wang, Qinghua Liu, Hao Liang, Gauri Joshi, and H~Vincent Poor.
\newblock Tackling the objective inconsistency problem in heterogeneous
  federated optimization.
\newblock \emph{arXiv preprint arXiv:2007.07481}, 2020.

\bibitem[Woodworth et~al.(2020)Woodworth, Patel, Stich, Dai, Bullins, McMahan,
  Shamir, and Srebro]{woodworth2020local}
Blake Woodworth, Kumar~Kshitij Patel, Sebastian~U Stich, Zhen Dai, Brian
  Bullins, H~Brendan McMahan, Ohad Shamir, and Nathan Srebro.
\newblock Is local {SGD} better than minibatch {SGD}?
\newblock \emph{arXiv preprint arXiv:2002.07839}, 2020.

\bibitem[Wu and He(2018)]{wu2018group}
Yuxin Wu and Kaiming He.
\newblock Group normalization.
\newblock In \emph{Proceedings of the European Conference on Computer Vision
  (ECCV)}, pages 3--19, 2018.

\bibitem[Xie et~al.(2019)Xie, Koyejo, Gupta, and Lin]{xie2019local}
Cong Xie, Oluwasanmi Koyejo, Indranil Gupta, and Haibin Lin.
\newblock Local {A}da{A}lter: Communication-efficient stochastic gradient
  descent with adaptive learning rates.
\newblock \emph{arXiv preprint arXiv:1911.09030}, 2019.

\bibitem[Yang et~al.(2018)Yang, Andrew, Eichner, Sun, Li, Kong, Ramage, and
  Beaufays]{yang2018applied}
Timothy Yang, Galen Andrew, Hubert Eichner, Haicheng Sun, Wei Li, Nicholas
  Kong, Daniel Ramage, and Fran{\c{c}}oise Beaufays.
\newblock Applied federated learning: Improving google keyboard query
  suggestions.
\newblock \emph{arXiv preprint arXiv:1812.02903}, 2018.

\bibitem[Yuan and Ma(2020)]{yuan2020federated}
Honglin Yuan and Tengyu Ma.
\newblock Federated accelerated stochastic gradient descent.
\newblock \emph{Advances in Neural Information Processing Systems}, 33, 2020.

\bibitem[Zhang et~al.(2019)Zhang, Lucas, Ba, and Hinton]{zhang2019lookahead}
Michael Zhang, James Lucas, Jimmy Ba, and Geoffrey~E Hinton.
\newblock Lookahead optimizer: k steps forward, 1 step back.
\newblock In \emph{Advances in Neural Information Processing Systems}, pages
  9593--9604, 2019.

\bibitem[Zhou et~al.(2019)Zhou, Yuan, Xu, Yan, and Feng]{zhou2019efficient}
Pan Zhou, Xiaotong Yuan, Huan Xu, Shuicheng Yan, and Jiashi Feng.
\newblock Efficient meta learning via minibatch proximal update.
\newblock In \emph{Advances in Neural Information Processing Systems}, pages
  1534--1544, 2019.

\bibitem[Zinkevich et~al.(2010)Zinkevich, Weimer, Li, and
  Smola]{zinkevich2010parallelized}
Martin Zinkevich, Markus Weimer, Lihong Li, and Alex~J Smola.
\newblock Parallelized stochastic gradient descent.
\newblock In \emph{Advances in neural information processing systems}, pages
  2595--2603, 2010.

\end{thebibliography}
